\documentclass{article}

\PassOptionsToPackage{numbers, compress}{natbib}

\usepackage[final]{neurips}
\usepackage{easyReview}
\usepackage[utf8]{inputenc} 
\usepackage[T1]{fontenc}    
\usepackage{hyperref}       
\usepackage{url}            
\usepackage{booktabs}       
\usepackage{amsfonts}       
\usepackage{nicefrac}       
\usepackage{microtype}      
\usepackage{amsmath}
\usepackage{amssymb}
\usepackage{mathtools}
\usepackage{amsthm}
\usepackage{pgfplots}
\usepackage[capitalize,noabbrev]{cleveref}
\usepackage{algorithm}
\usepackage{algorithmic}
\usepackage{kbordermatrix}
\usepackage{multirow}
\usepackage{xcolor}

\usepackage{booktabs}
\usepackage{siunitx}
\usepackage{array}
\usepackage{multirow}
\usepackage{amssymb}
\usepackage{tabularx}
\usepackage{subcaption}
\usepackage{todonotes}

\usepackage{tikz}
\usetikzlibrary{fit, shapes, arrows, positioning,fadings,tikzmark,calc,decorations.pathreplacing,patterns, pgfplots.fillbetween}
\usepgfplotslibrary{groupplots}
\pgfplotsset{compat=newest}
\pgfplotsset{every axis legend/.append style={%
cells={anchor=west}}
}
\pgfplotsset{every axis/.append style={
                    label style={font=\tiny},
					tick label style={font=\tiny},
					legend style={font=\tiny}
                    }}
\pgfplotsset{every x tick label/.append style={font=\tiny, yshift=0.5ex}}
\pgfplotsset{every y tick label/.append style={font=\tiny, xshift=0.5ex}}

\theoremstyle{plain}
\newtheorem{theorem}{Theorem}[section]
\newtheorem{proposition}[theorem]{Proposition}
\newtheorem{lemma}[theorem]{Lemma}
\newtheorem{corollary}[theorem]{Corollary}
\theoremstyle{definition}

\theoremstyle{remark}

\theoremstyle{example}

\DeclareMathOperator*{\argmax}{arg\,max}

\DeclareMathOperator*{\arginf}{arg\,inf}

\newcommand{\kl}{{\rm kl}}
\newcommand{\KL}{{\rm KL}}

\newcommand{\alt}{{\rm Alt}}
\newcommand{\deltamin}{\Delta_{\min}}

\newcommand{\deltaminestimatet}[1]{{\Delta}_{\textrm{min},#1}}
\DeclareMathOperator{\var}{Var}
\DeclareMathOperator{\md}{MD}

\newcommand{\includeappendix}{}

\title{Model-Free Active Exploration \\
in Reinforcement Learning}

\author{%
  Alessio Russo \\
  Division of Decision and Control Systems\\
  KTH Royal Institute of Technology\\
  Stockholm, SE \\
   \And
   Alexandre Proutiere \\
   Division of Decision and Control Systems\\
   KTH Royal Institute of Technology\\
   Stockholm, SE \\
}

\begin{document}
\maketitle

\begin{abstract}
We study the problem of exploration in Reinforcement Learning and present a novel model-free solution. We adopt an information-theoretical viewpoint and start from the  instance-specific lower bound of the number of samples that have to be collected to identify a nearly-optimal policy. Deriving this lower bound along with the optimal exploration strategy entails solving an intricate optimization problem and requires a model of the system. In turn, most existing sample optimal exploration algorithms rely on estimating the model. We derive an approximation of the instance-specific lower bound that only involves quantities that can be inferred using model-free approaches. Leveraging this approximation, we devise an ensemble-based model-free exploration strategy  applicable to both tabular and continuous Markov decision processes. Numerical results demonstrate that our strategy is able to identify efficient policies faster than state-of-the-art exploration approaches\let\thefootnote\relax\footnotetext{Code repository: \url{https://github.com/rssalessio/ModelFreeActiveExplorationRL}}.
\end{abstract}
\section{Introduction}\label{sec:introduction}
Efficient exploration remains a major challenge for reinforcement learning (RL) algorithms. Over the last two decades, several exploration strategies have been proposed in the literature, often designed with the aim of minimizing regret. These include model-based approaches such as Posterior Sampling for RL  \cite{osband2013more}({\sc PSRL}) and Upper Confidence Bounds for RL  \cite{auer2008near,lattimore2012pac,auer2002using}({\sc UCRL}), along with model-free  UCB-like methods \cite{jin2018q,dong2019q}. Regret minimization is a relevant objective when one cares about the rewards accumulated during the learning phase. Nevertheless,  an often more important objective  is to devise strategies that explore the environment so as to learn efficient policies using the fewest number of samples \cite{garivier2016optimal}. Such an objective, referred to as Best Policy Identification (BPI), has been investigated in simplistic Multi-Armed Bandit problems \cite{garivier2016optimal,kaufmann2016complexity} and more recently in tabular MDPs \cite{al2021adaptive,marjani2021navigating}. For these problems, tight instance-specific sample complexity lower bounds are known, as well as model-based algorithms approaching these limits.  However, model-based approaches may be computationally expensive or infeasible to obtain. In this paper, we investigate whether we can adapt the design of these algorithms so that they become model-free and hence more practical.

Inspired by \cite{al2021adaptive,marjani2021navigating}, we adopt an information-theoretical approach, and design our algorithms starting from an instance-specific lower bound on the sample complexity of learning a nearly-optimal policy in a Markov decision process (MDP). This lower bound is the value of an optimization problem, referred to as the lower bound problem,  whose solution dictates the optimal exploration strategy in an environment.
Algorithms designed on this instance-specific lower bound, rather than minimax bounds, result in truly adaptive methods, capable of tailoring their exploration strategy according to the specific MDP's learning difficulty.
Our method estimates the solution to the lower bound problem and employs it as our exploration strategy. However, we face two major challenges: (1) the lower bound problem is non-convex and often intractable; (2) this lower bound problem depends on the initially unknown MDP.  In \cite{marjani2021navigating}, the authors propose {\sc MDP-NaS}, a model-based algorithm that explores according to the estimated MDP. They convexify the lower bound problem and  explore according to the solution of the resulting simplified problem.  However, this latter problem  still has a complicated dependency on the MDP. Moreover, extending {\sc MDP-NaS} to large MDPs is challenging since it requires an estimate of the model, and the capability to perform policy iteration. Additionally, {\sc MDP-NaS} employs a {\it forced exploration} technique to ensure that the \emph{parametric} uncertainty (the uncertainty about the true underlying MDP) diminishes over time --- a method,  as we argue later, that we believe not to be efficient in handling this uncertainty.

We propose an alternative way to approximate the lower bound problem, so that its solution can be learnt via a model-free approach. This solution  depends only on the $Q$-function and the variance of the value function. Both quantities can advantageously be inferred using classical stochastic approximation methods. To handle the parametric uncertainty, we propose an ensemble-based method using a bootstrapping technique. This technique is inspired by posterior sampling and allows us to quantify the uncertainty when estimating the $Q$-function and the variance of the value function. 

Our contributions are as follows: (1) we shed light on the role of the instance-specific quantities needed to drive exploration in uncertain MDPs; (2) we derive an alternate upper bound of the lower bound problem that in turn can be approximated using quantities that can be learned in a model-free manner. We then evaluate the quality of this approximation on various environments: ({\it i}) a random MDP, ({\it ii}) the Riverswim environment \cite{strehl2008analysis}, and ({\it iii}) the Forked Riverswim environment (a novel environment with high sample complexity); (3) based on this approximation, we present Model Free Best Policy Identification ({\sc  MF-BPI}), a model-free exploration algorithm for  tabular and continuous MDPs. For the tabular MDPs, we test the performance of {\sc MF-BPI} on the Riverswim and the Forked Riverswim environments, and compare it to that of {\sc Q-UCB} \cite{jin2018q,dong2019q}, {\sc PSRL}\cite{osband2013more}, and {\sc MDP-NaS}\cite{marjani2021navigating}. 
 For continuous state-spaces, we compare our algorithm to {\sc IDS}\cite{nikolov2018information} and {\sc BSP} \cite{osband2018randomized} (Boostrapped DQN with randomized prior value functions) and assess their performance on hard-exploration problems  from the DeepMind BSuite \cite{osband2020bsuite} (the DeepSea and the Cartpole swingup problems).
 
\section{Related Work}\label{sec:related_work}

The body of work related to exploration methods in RL problems is vast, and we mainly focus on online discounted MDPs (for the generative setting, refer to the analysis presented in \cite{gheshlaghi2013minimax,al2021adaptive}).
Exploration strategies in RL often draw inspiration from the approaches used in multi-armed bandit problems \cite{lattimore2020bandit,sutton2018reinforcement}, including $\epsilon$-greedy exploration, Boltzmann exploration \cite{watkins1989learning, sutton2018reinforcement,lattimore2020bandit,atkins2014atkins}, or more advanced procedures, such as Upper-Confidence Bounds (UCB) methods \cite{auer2002using,auer2002finite,lattimore2020bandit} or Bayesian procedures \cite{thompson1933likelihood,wyatt1998exploration,dearden1998bayesian,russo2018tutorial}.
We first discuss  tabular MDPs, and then extend the discussion to the case of RL with function approximation. 

\textbf{Exploration in tabular MDPs.} Numerous algorithms have been proposed with the aim of matching the PAC sample complexity minimax lower bound $\tilde \Omega\left(\frac{|S||A|}{\varepsilon^2(1-\gamma)^3}\right)$ \cite{lattimore2012pac}. In the design of these algorithms, model-free approaches typically rely on  a UCB-like exploration \cite{auer2002using,lattimore2020bandit}, whereas model-based methods leverage estimates of the MDP to drive the exploration. Some well-known  model-free algorithms are \textsc{Median-PAC} \cite{pazis2016improving},  \textsc{Delayed Q-Learning} \cite{strehl2006pac} and \textsc{Q-UCB} \cite{dong2019q,jin2018q}. Some notable model-based algorithms include: {\sc DEL} \cite{ok2018exploration}, an algorithm that achieves asymptotically optimal instance-dependent regret;  \textsc{UCRL} \cite{lattimore2012pac}, an algorithm that uses extended value-iteration to compute an optimistic MDP; {\sc PSRL} \cite{osband2013more}, that  uses posterior sampling to sample an MDP. Other algorithms include  \textsc{MBIE} \cite{strehl2008analysis}, \textsc{E3} \cite{kearns2002near}, \textsc{R-MAX} \cite{brafman2002r,kakade2003sample}, and \textsc{MORMAX} \cite{szita2010model}. Most of existing algorithms are designed towards regret minimization. Recently, however, there has been a growing interest towards exploration strategies with minimal sample complexity, see e.g. \cite{zanette2019,al2021adaptive}. In \cite{al2021adaptive,marjani2021navigating}, the authors showed that computing an exploration strategy with minimal sample complexity requires to solve a non-convex problem. To overcome this challenge, they derived a tractable approximation of the lower bound problem, whose solution provides an efficient exploration policy under the generative model \cite{al2021adaptive} and the forward model \cite{marjani2021navigating}. This policy necessitates an estimate of the model, and includes a forced exploration phase (an $\epsilon$-soft policy to guarantee that all state-action pairs are visited infinitely often). In \cite{taupin2022best}, the above procedure is extended to linear MDPs, but there again,  computing an optimal exploration strategy remains challenging. On a side note, in \cite{wagenmaker2022beyond}, the authors provide an alternative bound in the tabular case for episodic MDPs, and later extend it to linear MDPs \cite{wagenmaker2022instance}. The episodic setting is further explored in \cite{tirinzoni2022near} for deterministic MDPs.

\textbf{Exploration in Deep Reinforcement Learning (DRL).} Exploration methods in DRL environments face several challenges, related to the fact that the state-action spaces are often continuous, and other issues related to training deep neural architectures \cite{sewak2019deep}.
The main issue in these large MDPs is that good exploration becomes extremely hard when  either the reward is sparse/delayed  or the observations contain distracting features \cite{burda2018exploration,yang2021exploration}.  Numerous heuristics have been proposed to tackle these challenges, such as (1) adding an entropy term to the optimization problem to encourage the policy to be more randomized \cite{mnih2016asynchronous,haarnoja2018soft} or (2) injecting noise in the observations/parameters \cite{fortunato2017noisy,plappert2017parameter}.
More generally, exploration techniques generally fall into two categories: \emph{uncertainty-based} and \emph{intrinsic-motivation-based} \cite{yang2021exploration,ladosz2022exploration}. Uncertainty-based methods decouple the uncertainty into  {\it parametric} and {\it aleatoric} uncertainty. Parametric uncertainty \cite{dearden1998bayesian,moerland2017efficient,kirschner2018information,yang2021exploration} quantifies the uncertainty in the parameters of the state-action value. This uncertainty vanishes as the agent explores and learns. The aleatoric uncertainty accounts for the inherent randomness of the environment and of the policy \cite{moerland2017efficient,kirschner2018information,yang2021exploration}.
Various methods have been proposed to address the parametric uncertainty, including  UCB-like mechanisms \cite{chen2017ucb,yang2021exploration}, or TS-like (Thompson Sampling) techniques \cite{osband2016generalization, osband2013more, azizzadenesheli2018efficient, osband2015bootstrapped, osband2016deep, osband2019deep}.  However, computing a posterior of the $Q$-values is a difficult task. For instance,  Bayesian DQN \cite{azizzadenesheli2018efficient}  extends Randomized Least-Squares Value Iteration ({\sc RLSVI}) \cite{osband2016generalization} by considering the features prior to the output layer of the deep-$Q$ network as a fixed feature vector, in order to recast the problem as a linear MDP.
Non-parametric posterior sampling methods include Bootstrapped DQN (and Bootstrapped DQN with prior functions) \cite{osband2016deep,osband2018randomized,osband2019deep}, which maintains several independent $Q$-value functions and randomly samples one of them to explore the environment. Bootstrapped DQN was extended in various ways by integrating other techniques  \cite{bai2021principled,lee2021sunrise}. For the sake of brevity, we refer  the reader to the survey in \cite{yang2021exploration} for an exhaustive list of algorithms.
Most of these algorithms do not directly account for  aleatoric uncertainty in the value function. This uncertainty is usually estimated using methods like Distributional RL \cite{bellemare2017distributional,dabney2018distributional,mavrin2019distributional}.
Well-known exploration methods that account for both aleatoric and epistemic uncertainties include Double Uncertain Value Network (DUVN) \cite{moerland2017efficient} and Information Directed Sampling  ({\sc IDS}) \cite{kirschner2018information,nikolov2018information}. The former uses Bayesian dropout to measure the epistemic uncertainty, and the latter uses distributional RL \cite{bellemare2017distributional} to estimate the variance of the returns. In addition, {\sc IDS} uses bootstrapped DQN to estimate the parametric uncertainty in the form of a bound on the estimate of the suboptimality gaps. These uncertainties are then combined to compute an exploration strategy. Similarly, in \cite{clements2019estimating}, the authors propose \textsc{UA-DQN}, an approach that uses \textsc{QR-DQN} \cite{dabney2018distributional} to learn the parametric and aleatoric uncertainties from the quantile networks. Lastly, we refer the reader to \cite{yang2021exploration,ryan2000intrinsic,barto2013intrinsic} for the class of intrinsic-motivation-based methods.
\section{Preliminaries}\label{sec:preliminaries}

\textbf{Markov Decision Process.}  We consider an infinite-horizon discounted Markov Decision Process (MDP), defined by the tuple $\phi=(S,A, P,q,\gamma,p_0)$. $S$ is the state space, $A$ is the action space, $P: S\times A \to \Delta(S)$ is the distribution over the next state given a state-action pair $(s,a)$, $q:S\times A\to \Delta([0,1])$ is the distribution of the collected reward (with support in $[0,1]$), $\gamma\in[0,1)$ is the discount factor and $p_0$ is the distribution over the initial state.

Let $\pi: \mathcal{S}\to \Delta(A)$ be a stationary Markovian policy that maps a state to  a distribution over actions, and denote by $r(s,a)=\mathbb{E}_{r\sim q(\cdot|s,a)}[r]$ the average  reward collected when an action $a$ is chosen in state $s$. We denote by $V^\pi(s)=\mathbb{E}_\phi^\pi[\sum_{t\geq 0}\gamma^t r(s_t,a_t)|s_0=s]$ the discounted value of policy $\pi$. We denote by $\pi^\star$ an optimal stationary policy: for any $s\in {\cal S}$, $\pi^\star(s) \in \argmax_{\pi} V^\pi(s)$ and define $V^\star(s) =  \max_\pi V^{\pi}(s)$. For the sake of simplicity, we assume that the MDP has a unique optimal policy (we extend our results to more general MDPs in the appendix). We further define $\Pi_\varepsilon^\star(\phi)=\{\pi: \|V^\pi-V^{\pi^\star}\|_\infty \leq \varepsilon\}$, the set of $\varepsilon$-optimal policies in $\phi$ for  $\varepsilon\geq 0$. Finally, to avoid technicalities,  we assume (as in \cite{marjani2021navigating}) that the MDP $\phi$ is communicating (that is, for every pair of states $(s,s')$, there exists a deterministic policy $\pi$ such that state $s'$ is accessible from state $s$ using $\pi$).

We  denote by $Q^\pi(s,a)\coloneqq r(s,a)+\gamma\mathbb{E}_{s'\sim P(\cdot|s,a)}[V^\pi(s')]$ the $Q$-function of $\pi$ in state $(s,a)$. We also define the   sub-optimality gap of action $a$ in state $s$ to be $\Delta(s,a) \coloneqq Q^\star(s,\pi^\star(s))- Q^{\star}(s,a)$, where $Q^\star$ is the $Q$-function of $\pi^\star$, and let $
\deltamin\coloneqq \min_{s,a\neq \pi^\star(s)} \Delta(s,a)$ be the minimum gap in $\phi$.
For some policy $\pi$, we define $\var_{sa}[V^\pi]\coloneqq \var_{s'\sim P(\cdot|s,a)}[V^\pi(s')]$ to be the
variance of the value function  $V^\pi$  in the next state after taking action $a$ in state $s$. More generally, we define $
M_{sa}^{k}[V^\pi]\coloneqq \mathbb{E}_{s'\sim P(\cdot|s,a)}\left[\left(V^\pi(s') -\mathbb{E}_{\bar s\sim P(\cdot|s,a)}[V^\pi(\bar s)]\right)^{2^k}\right]$ to be the $2^k$-th moment of the value function in the next state after taking action $a$ in state $s$.
We also  let $\md_{sa}[V^\pi]\coloneqq \|V^\pi - \mathbb{E}_{s'\sim P(\cdot|s,a)}[V^\pi]\|_\infty$ be the span of $\phi$ under $\pi$, \emph{i.e.}, the maximum deviation from the mean of the next state value after taking action $a$ in state $s$.

\textbf{Best policy identification and sample complexity lower bounds.} The MDP $\phi$ is initially unknown, and we are interested in the scenario where the agent  interacts sequentially with $\phi$. In each round $t\in \mathbb{N}$, the agent selects an action $a_t$ and observes the next state and the reward $(s_{t+1}, r_t)$: $s_{t+1}\sim P(\cdot|s_t,a_t)$ and $r_t\sim q(\cdot|s_t,a_t)$. The objective of the agent is to learn a policy in $\Pi_\varepsilon^\star(\phi)$ (possibly  $\pi^\star$) as fast as possible. This objective is often formalized in a PAC framework where the learner has to stop interacting with the MDP when she can output an $\varepsilon$-optimal policy with probability at least $1-\delta$. In this formalism, the learner strategy consists of (i) a sampling rule or exploration strategy; (ii) a stopping time $\tau$;  (iii) an estimated optimal policy $\hat{\pi}$. The strategy is called $(\varepsilon,\delta)$-PAC if it stops almost surely, and $\mathbb{P}_\phi[\hat{\pi} \in \Pi_\varepsilon^\star(\phi)]\ge 1-\delta$. Interestingly, one may derive instance-specific lower bounds of the sample complexity $\mathbb{E}_\phi[\tau]$ of any $(\varepsilon,\delta)$-PAC algorithm \cite{al2021adaptive,marjani2021navigating}, which involves  computing an optimal allocation vector $\omega_{{\rm opt}}\in \Delta(S\times A)$ (where $\Delta(S\times A)$ is the set of distributions over $S\times A$) that specifies  the proportion of times an agent needs to sample each pair $(s,a)$ to confidently identify the optimal policy:
\begin{equation}\label{eq:lower_bound_sample_complexity}
     \liminf_{\delta\to 0} {\mathbb{E}_\phi[\tau]\over \kl(\delta,1-\delta)} \geq  T_\varepsilon(\omega_{\rm opt}) \text{ where } T_\varepsilon(\omega)^{-1}:= \inf_{\psi \in \alt_\varepsilon(\phi)} \mathbb{E}_{(s,a)\sim\omega}[\KL_{\phi|\psi}(s,a)],
 \end{equation}
 and $\omega_{\rm opt} = \arg\inf_{\omega \in \Omega(\phi)}T_\varepsilon(\omega)^{-1}$. Here, $\alt_\varepsilon(\phi)$ is the set of confusing MDPs $\psi$ such that the $\varepsilon$-optimal policies of $\phi$ are not $\varepsilon$-optimal in $\psi$, {i.e.}, $\alt_\varepsilon(\phi)\coloneqq \{\psi: \phi\ll \psi, \Pi_\varepsilon^\star(\phi)\cap \Pi_\varepsilon^\star(\psi) =\emptyset\}$. In this definition, if the next state and reward distributions under $\psi$ are $P'(s,a)$ and $q'(s,a)$, we write $\phi \ll \psi$ if for all $(s,a)$ the distributions of the next state and of the rewards satisfy $P(s,a)\ll P'(s,a)$ and $q(s,a)\ll q'(s,a)$.We further let $\KL_{\phi|\psi}(s,a)\coloneqq \KL(P(s,a),P'(s,a)) +\KL(q(s,a),q'(s,a))$. $\Omega(\phi)$ is the set of possible allocations; in the generative case it is $\Delta(S\times A)$, while with navigation constraints we have $\Omega(\phi)\coloneqq\{\omega \in\Delta(S\times A): \omega(s) = \sum_{s',a'} P(s|s',a')\omega(s',a')\}, \forall s\in S\}$, with $\omega(s)\coloneqq \sum_a \omega(s,a)$. Finally, $\kl(a,b)$ is the KL-divergence between two Bernoulli distributions of means $a$ and $b$.

\section{Towards Efficient Exploration Allocations}\label{sec:adaptive_sampling_rl}
We aim to extend previous studies on best policy identification to online model-free exploration. In this section, we derive an approximation to the bound proposed in \cite{al2021adaptive}, involving quantities learnable via stochastic approximation, thereby enabling the use of model-free approaches. 

The optimization problem (\ref{eq:lower_bound_sample_complexity}) leading to instance-specific sample complexity lower bounds has an important interpretation \cite{al2021adaptive,marjani2021navigating}. An allocation $\omega_{{\rm opt}}$ corresponds to an exploration strategy with minimal sample complexity. 
To devise an efficient exploration strategy, one could then think of estimating the MDP $\phi$, and  solving (\ref{eq:lower_bound_sample_complexity}) for this estimated MDP to get an approximation of $\omega_{{\rm opt}}$. There are two important challenges towards applying this approach:
\begin{itemize}
    \item[(i)] Estimating the model can be difficult, especially for MDPs with large state and action spaces, and arguably, a model-free method would be preferable.
    \item[(ii)] The lower bound problem (\ref{eq:lower_bound_sample_complexity}) is, in general, non-convex \cite{al2021adaptive, marjani2021navigating}.
\end{itemize}
A simple way to circumvent  issue (ii) involves deriving an upper bound of the value of the sample complexity lower bound problem (\ref{eq:lower_bound_sample_complexity}). Specifically, one may derive an upper bound $U(\omega)$ of $T_\varepsilon(\omega)$ by convexifying the corresponding optimization problem. The exploration strategy can then be the $\omega^\star$ that achieves the infimum of $U(\omega)$. This approach ensures that we identify an approximately optimal policy, at the cost of \emph{over-exploring} at a rate corresponding to the gap $U(\omega^\star) - T_\varepsilon(\omega_{{\rm opt}})$. Note  that using a lower bound of $T_\varepsilon(\omega)$ would not guarantee the identification of an optimal policy, since we would explore "less" than required. The aforementioned approach was already used in \cite{al2021adaptive} where the authors derive an explicit upper bound $U_0(\omega)$ of $T_0(\omega)$. We also apply it, but derive an upper bound such that implementing the corresponding allocation $\omega^\star$ can be done in a model-free manner (hence solving the first issue (i)).

\subsection{Upper bounds on $T_{\varepsilon}(\omega)$}

The next theorem presents the upper bound derived in \cite{al2021adaptive}. 

\begin{theorem}[\cite{al2021adaptive}]\label{thm:upper_bound}
Consider a communicating MDP $\phi$ with a unique optimal policy $\pi^\star$. For all vectors $\omega\in \Delta(S\times A)$, 
\begin{equation}\label{eq:original_upper_bound}
T_0(\omega)\le U_0(\omega)\coloneqq \max_{(s,a):a\neq\pi^\star(s)} \frac{H_0(s,a)}{\omega(s,a)}+ \max_s\frac{H_0^\star}{\omega(s,\pi^\star(s))},
\end{equation}
with 
$$
\left\{
\begin{array}{l}
H_0(s,a) = \frac{2}{\Delta(s,a)^2} + \max\left({ 16\var_{sa}[V^\star]\over \Delta(s,a)^2}, {6 \md_{sa}[V^\star]^{4/3}\over \Delta(s,a)^{\frac{4}{3}} } \right),\\
H_0^\star =\frac{2}{\deltamin^2(1-\gamma)^2} + \min\left(\frac{27}{\deltamin^2(1-\gamma)^3}, \max\left({16 \max_s\var_{s\pi^\star(s)}[V^\star]\over \deltamin^2(1-\gamma)^2}, {6 \max_s\md_{s\pi^\star(s)}[V^\star]^{4/3}\over \deltamin^{4/3} (1-\gamma)^{4/3} } \right)\right).
\end{array}
\right.
$$
\end{theorem}

In the upper bound presented in this theorem, the following  quantities characterize the {\it hardness} of learning the optimal policy: $\Delta(s,a)$ represents  the difficulty of learning that in state $s$ action $a$ is sub-optimal; the variance $\var_{sa}[V^\star]$  measures the aleatoric uncertainty in future state values; and the span $\md_{sa}[V^\star]$ of the optimal value function can be seen as another measure of aleatoric uncertainty, large whenever there is a significant variability in the value for the possible next states.

Estimating the span  $ \md_{sa}[V^\star]$, in an online setting, is a challenging task for large MDPs. Our objective is to derive an alternative upper bound that, in turn, can be approximated using quantities that can be learned in a model-free manner. We observe that the variance of the value function, and more generally its moments $M_{sa}^k[V^\star]^{2^{-k}}$ for $k\geq 1$ (see \ifdefined\includeappendix \cref{appC} \else Appendix C\fi), are smaller than the span. By refining the proof techniques used in \cite{al2021adaptive}, we derive the following alternative  upper bound.

\begin{theorem}\label{thm:upper_bound_T_new}
Let $\varepsilon \geq 0$ and let $k(s,a) \coloneqq \arg\sup_{k \in\mathbb{N}} M_{sa}^k[V^\star]^{2^{-k}}$ (for brevity, we write $k$ instead of $k(s,a)$). Then, $\forall \omega \in \Delta(S\times A)$, we have $T_\varepsilon(\omega)\leq
   U(\omega)$, with \begin{equation}\label{eq:new_upper_bound}
      U(\omega)\coloneqq \max_{s,a\neq \pi^\star(s)} \left(\frac{2+8\varphi^2  M_{sa}^{ k}[V^\star]^{2^{1- k}} }{\omega(s,a)\Delta(s,a)^2} +\max_{s'} \frac{C(s') (1+\gamma)^2}{\omega(s',\pi^\star(s')) \Delta(s,a)^2(1-\gamma)^2}\right),
    \end{equation}
    where   $C(s')=\max\left(4,16\gamma^2\varphi^2 M_{s',\pi^\star(s')}^{ k}[V^\star]^{2^{1- k}}\right)$ and $\varphi$ is the golden ratio. 
\end{theorem}
We can observe that in the worst case, the upper bound $U(\omega^\star)$ of the sample complexity lower bound, with $\omega^\star=\arg\inf_\omega U(\omega)$, scales as $O(\frac{|S||A|\max_s\md_{s,\pi^\star(s)}[V^\star]^2}{\deltamin^2 (1-\gamma)^2})$.  Since $\md_{sa}[V^\star]\leq (1-\gamma)^{-1}$, then $U(\omega^\star)$ scales at most as $O(\frac{|S||A|}{\deltamin^2 (1-\gamma)^4})$. However, the following questions arise: (1) Can we select a single value of $k$ that provides a good approximation across all states and actions?
(2) How much does this bound improve on that of \cref{thm:upper_bound}? 
As we illustrate in the example presented in the next subsection, we believe that  actually selecting $k=1$ for all states and actions leads to sufficiently good results. With this choice, we obtain the following approximation:

\begin{equation}\label{eq:new_bound_var_kbar_1}
 U_{1}(\omega)\coloneqq \max_{s,a\neq \pi^\star(s)}\left( \frac{2+8\varphi^2 \var_{sa}[V^\star]}{\omega(s,a)\Delta(s,a)^2} +\max_{s'} \frac{C'(s') (1+\gamma)^2}{\omega(s',\pi^\star(s'))\Delta(s,a)^2(1-\gamma)^2}\right),
\end{equation}
where $C'(s')=\max\left(4,16\gamma^2\varphi^2 \var_{s',\pi^\star(s')}[V^\star]\right)$. $U_1(\omega)$  resembles the term in \cref{thm:upper_bound} (note  that we do not know whether $ U_1$ is a valid upper bound for $T_\varepsilon$). For the second question, our numerical experiments (presented below) suggest that $U(\omega)$ is a tighter upper bound than $U_0(\omega)$.

 \begin{figure}[t]
 	\centering
 	\includegraphics[width=\columnwidth]{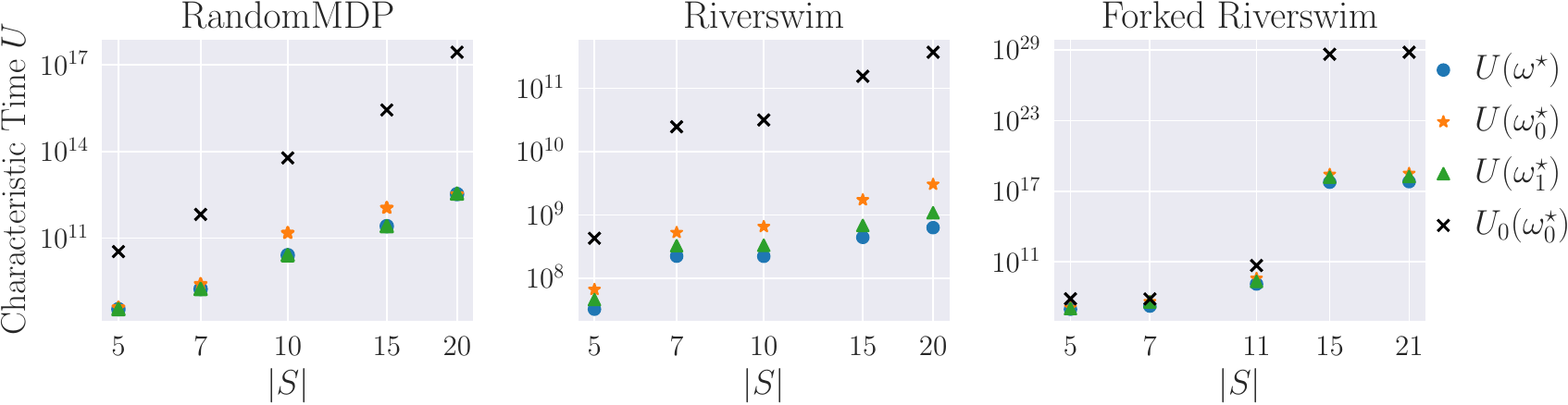}
 	\caption{Comparison of the upper bounds (\ref{eq:original_upper_bound}) and (\ref{eq:new_upper_bound}) for  different sizes of  $S$ and $\gamma=0.95$. We evaluated different allocations using $ U_0(\omega)$ and $U(\omega)$. The allocations are: $\omega_0^\star$ (the optimal allocation in (\ref{eq:original_upper_bound}), $ \omega^\star$ (the optimal allocation in (\ref{eq:new_upper_bound}) and $ \omega_1^\star$ (the optimal allocation in (\ref{eq:new_bound_var_kbar_1}) by setting $ k=1$ uniformly across states and actions). For the random MDP we show the median value across $30$ runs.}
 	\label{fig:example_upper_bound}
 \end{figure}

 \subsection{Example on Tabular MDPs}\label{example:randomly_drawn_mdp_value}
In \cref{fig:example_upper_bound}, we compare the characteristic time upper bounds obtained in the previous subsection. These upper bounds correspond to the allocations $\omega^\star$, $\omega_0^\star$, and $\omega_1^\star$ obtained by minimizing, over $\Delta(S\times A)$\footnote{Results are similar when we account for the navigation constraints. We omit these results for simplicity.}, $U(\omega)$, $U_0(\omega)$, and $U_1(\omega)$, respectively. We  evaluated these characteristic times on various MDPs: (1) a random MDP (see \ifdefined \includeappendix \cref{appA} \else Sec. A in the appendix\fi); (2) the \texttt{RiverSwim} environment \cite{strehl2008analysis}; (3) the \texttt{Forked RiverSwim}, a novel environment  where the agent needs to constantly explore two different states to learn the optimal policy (compared to the \texttt{RiverSwim} environment, the sample complexity is higher; refer to Appendix A for a complete description).

We note that across all plots, the  optimal allocation $\omega_0^\star$ has a quite large characteristic time (black cross). Instead, the optimal allocation $\omega^\star$ (blue circle) computed using our new upper bound (\ref{eq:new_upper_bound}) achieves a lower characteristic time. When we evaluate $\omega_0^\star$ on the new bound (\ref{eq:new_upper_bound}) (orange star), we observe similar characteristic times.

Finally, to verify that we can indeed  choose $ k=1$ uniformly across states and actions, we evaluated  the characteristic time $ \omega_1^\star$ computed using (\ref{eq:new_bound_var_kbar_1}) (green triangle). Our results indicate that the performance is not different from those obtained with $\omega^\star$, suggesting that the quantities of interest (gaps and variances) are enough to learn an efficient exploration allocation. We investigate the choice of $k$ in more detail in \ifdefined \includeappendix \cref{appA} \else Appendix A\fi.

\section{Model-Free Active Exploration Algorithms}\label{sec:obpi}
In this section we present {\sc MF-BPI},  a model-free exploration algorithm that leverages the optimal allocations obtained through the previously derived upper bound of the sample complexity lower bound. We first present an upper bound $\tilde{U}(\omega)$ of $U(\omega)$, so that it is possible to derive a closed form solution of the optimal allocation (an idea previously proposed in \cite{al2021adaptive}). 

\begin{proposition}\label{corollary:upper_bound_new_bound} Assume that $\phi$ has a unique optimal policy $\pi^\star$. For all $\omega\in\Delta(S\times A)$, we have: 
$$
U(\omega)\leq \tilde U(\omega) := \max_{s,a\neq \pi^\star(s)} \frac{ H(s,a)}{\omega(s,a)}+\frac{ H}{\min_{s'}\omega(s',\pi^\star(s'))},
$$
with $ H(s,a) \coloneqq \frac{2+8\varphi^2  M_{sa}^{ k}[V^\star]^{2^{1- k}} }{\Delta(s,a)^2}$ and $ H \coloneqq \frac{\max_{s'} C(s') (1+\gamma)^2}{ \deltamin^2(1-\gamma)^2}$.
The minimizer $\tilde\omega^\star \coloneqq \arg\inf_{\omega}\tilde U(\omega)$  satisfies $\tilde\omega^\star(s,a) \propto  H(s,a) $ for $a\neq\pi^\star(s)$ and $\tilde\omega^\star(s,\pi^\star(s)) \propto \sqrt{ H\sum_{s,a\neq\pi^\star(s)} H(s,a)/|S|}$ otherwise. 
\end{proposition}

In the MF-BPI algorithm, we estimate the gaps $\Delta(s,a)$ and $M_{sa}^{ k}[V^\star]$ for a fixed small value of $k$ (we later explain how to do this in a model-free manner.) and compute the corresponding allocation $\tilde\omega^\star$. This allocation drives the exploration under MF-BPI. Using this design approach, we face two issues:

{\bf (1) Uniform $k$ and regularization.} It is impractical to estimate $M_{sa}^{ k}[V^\star]$ for multiple values of $k$. Instead, we fix a small value of $k$ (\emph{e.g.}, $k=1$ or $k=2$) for all state-action pairs (refer to the previous section for a discussion on this choice). Then, to avoid  excessively small values of the gaps in the denominator, we regularize the allocation $\tilde\omega^\star$ by replacing, in the expression of $H(s,a)$ (resp. $H_{\min}$), $\Delta(s,a)$ (resp. $\Delta_{\min}$) by $(\Delta(s,a)+\lambda)$ (resp. $(\Delta_{\min}+\lambda)$) for some $\lambda>0$.

{\bf (2) Handling parametric uncertainty via bootstrapping.} The quantities $\Delta(s,a)$ and $M_{sa}^k[V^\star]$ required to compute $\tilde\omega^\star $ remain unknown during training, and we adopt the Certainty Equivalence principle, substituting the current estimates of these quantities to compute  the exploration strategy. By doing so, we are inherently introducing parametric uncertainty into these terms that is not taken into account by the allocation $\tilde \omega^\star$. To deal with this uncertainty, the traditional method, as used e.g. in \cite{al2021adaptive,marjani2021navigating}), involves using $\epsilon$-soft exploration policies to guarantee that all state-action pairs are visited infinitely often. This ensures that the estimation errors vanish as time grows large. In practice, we find this type of forced exploration inefficient. In MF-BPI, we opt for a bootstrapping approach to manage parametric uncertainties, which can augment the traditional forced exploration step, leading to more principled exploration.

\subsection{Exploration in tabular MDPs.}
\begin{algorithm}[t]
	\caption{Boostrapped \textsc{MF-BPI} (Boostrapped Model Free Best Policy Identification)}
	\label{algo:boostrapped_mfbpi}
\begin{algorithmic}[1]
	\REQUIRE Parameters $(\lambda,  k,p)$; ensemble	 size $B$;  learning rates $\{(\alpha_{t},\beta_{t})\}_{t}$.
	\STATE Initialize  $Q_{1,b}(s,a)\sim {\cal U}([0,1/(1-\gamma)])$ and $M_{1,b}(s,a)\sim {\cal U}([0,1/(1-\gamma)^{2^{ k}}])$ for all $(s,a)\in S\times A$ and $b\in[B]$. 
	\FOR{$t=0,1,2,\dots,$}
		\STATE Bootstrap a sample $(\hat Q_t, \hat M_t)$ from the ensemble, and compute the allocation  $\omega^{(t)}$ using  \cref{corollary:upper_bound_new_bound}. Sample $a_t\sim \omega^{(t)}(s_t,\cdot)$; observe $(r_t,s_{t+1})\sim q(\cdot|s_t,a_t)\otimes P(\cdot|s_t,a_t)$.
		\FOR{$b=1,\dots,B$}
			\STATE With probability $p$, using the experience $(s_t,a_t,r_t,s_{t+1})$, update  $Q_{t,b}$ and $M_{t,b}$  using \cref{eq:stochastic_approximation_step_qvalues,eq:stochastic_approximation_step_mvalues}.
		\ENDFOR
	\ENDFOR
\end{algorithmic}
\end{algorithm}

The pseudo-code of MF-BPI for tabular MDPs is presented in Algorithm \ref{algo:boostrapped_mfbpi}. In round $t$, MF-BPI explores the MDP using the allocation $\omega^{(t)}$ estimating $\tilde{\omega}^\star$. To compute this allocation, we use \cref{corollary:upper_bound_new_bound} and need (i) the sub-optimality gaps $\Delta(s,a)$, which can be easily derived from the $Q$-function; (ii) the $2^k$-th moment $M_{sa}^k[V^\star]$, which can always be learnt by means of stochastic approximation. In fact, for any  Markovian policy $\pi$ and pair $(s,a)$ we have
$
   M_{sa}^{k}[V_\phi^\pi]=\frac{1}{\gamma^{2^k}}\mathbb{E}_{s'\sim P(\cdot|s,a)}[\delta^\pi(s,a,s')^{2^k}],
$
where $\delta^\pi(s,a,s')=r(s,a)+\gamma \mathbb{E}_{a'\sim \pi(\cdot|s')}[Q^\pi(s',a')] - Q^\pi(s,a)$ is a variant of the TD-error. 
MF-BPI then uses an asynchronous two-timescale stochastic approximation algorithm to learn $Q^\star$ and $M_{sa}^{ k}[V^\star]$,
\begin{align}
	Q_{t+1}(s_t,a_t) &= Q_t(s_t,a_t) + \alpha_t(s_t,a_t)\left(r_t+\gamma \max_a Q_t(s_{t+1},a)-Q_t(s_t,a_t)\right),\label{eq:stochastic_approximation_step_qvalues}\\
	M_{ t+1}(s_t,a_t) &= M_{t}(s_t,a_t) + \beta_t(s_t,a_t)\left(\left(\delta_t'/\gamma\right)^{2^{ k}} - M_{ t}(s_t,a_t)\right),\label{eq:stochastic_approximation_step_mvalues}
\end{align}
where  $\delta_t'= r_t+\gamma \max_a Q_{t+1}(s_{t+1},a)-Q_{t+1}(s_t,a_t)$, and $\{(\alpha_t,\beta_t)\}_{t\geq 0}$ are learning rates  satisfying  $\sum_{t\geq 0}\alpha_t(s,a)=\sum_{t\geq 0}\beta_t(s,a)=\infty, \sum_{t\geq 0}(\alpha_t(s,a)^2+\beta_t(s,a)^2)\leq \infty$, and $\frac{\alpha_t(s,a)}{\beta_t(s,a)}\to 0$.


MF-BPI uses bootstrapping to handle parametric uncertainty. We maintain an ensemble of $(Q,M)$-values, with $B$ members, from which we sample $(\hat Q_t,\hat M_t)$ at time $t$. This sample is generated by sampling a uniform random variable $\xi\sim {\cal U}([0,1])$ and, for each $(s,a)$ set $\hat Q_t(s,a)={\rm Quantile}_\xi({Q_{t,1}(s,a), \dots, Q_{t,B}(s,a)})$ (assuming a linear interpolation). This method is akin to sampling from the parametric uncertainty distribution (we perform the same operation also to compute $\hat M_t$). This  sample is  used to compute the allocation $\omega^{(t)}$ using \cref{corollary:upper_bound_new_bound} by setting  $\Delta_{t}(s,a) = \max_{a'} \hat Q_{t}(s,a')-\hat Q_{t}(s,a)$, $\pi_t^\star(s)=\argmax_a \hat Q_t(s,a)$ and $\deltaminestimatet{t} = \min_{s,a\neq \pi_t^\star(s)} \Delta_t(s,a)$. Note that, the allocation $\omega^{(t)}$ can be mixed with a uniform policy, to guarantee asymptotic convergence of the estimates. Upon observing an experience, with probability $p$, MF-BPI updates a member of the ensemble using this new experience. $p$ tunes the rate at which the models are updated, similar to sampling with replacement, speeding up the learning process. Selecting a high value for $p$ compromises the estimation of the parametric uncertainty, whereas choosing a low value may slow down the learning process.

{\bf Exploration without bootstrapping?} To illustrate the need for our bootstrapping approach, we tried to  use the allocation $\omega^{(t)}$ mixed with a uniform allocation. In \cref{fig:forced_generative_performance}, we show the results on Riverswim-like environments with $5$ states.  While forced exploration ensures infinite visits to all state-action pairs, this guarantee only holds asymptotically. As a result, the allocation mainly focuses on the current MDP estimate, neglecting other plausible MDPs that could produce the same data. This makes the forced exploration approach too sluggish for effective convergence, suggesting its inadequacy for rapid policy learning. These results highlight the need to account for the uncertainty in $Q,M$ when computing the allocation.

\begin{figure}[t]
	\centering
	\includegraphics[width=\linewidth]{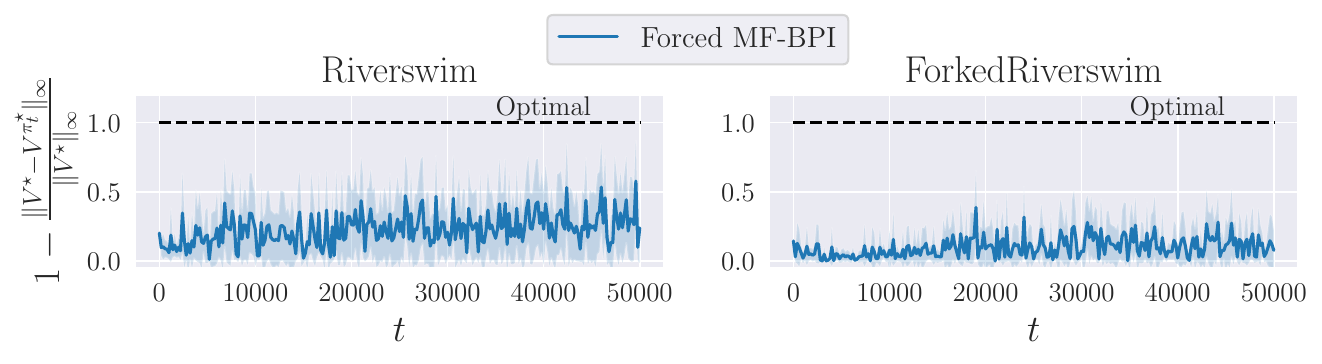}
	\caption{Forced exploration example with $5$ states. We explore according to $\omega^{(t)}(s_t,a) = (1-\epsilon_t) \frac{\tilde{\omega}_t^\star(s_t,a)}{\sum_{a'}\tilde{\omega}_t^\star(s_t,a')} + \epsilon_t \frac{1}{|A|}$, mixing the estimate of the allocation $\tilde{\omega}^\star$ from \cref{corollary:upper_bound_new_bound} with a uniform policy, with $\epsilon_t = \max(10^{-3}, 1/N_t(s_t))$ where $N_t(s)$ indicates the number of times the agent visited state $s$ up to time $t$. Shade indicates $95\%$ confidence interval.}
	\label{fig:forced_generative_performance}
\end{figure}
\begin{algorithm}[b]
\caption{\textsc{DBMF-BPI} (Deep Bootstrapped Model Free BPI)}\label{algo:dbomfbpi}
\small
\begin{algorithmic}[1]
	\REQUIRE Parameters $(\lambda,  k)$; ensemble size $B$;  exploration rate $\{\epsilon_t\}_t$;  estimate $\deltaminestimatet{0}$; mask probability $p$.
	\STATE Initialize replay buffer ${\cal D}$, networks  $Q_{\theta_b}, M_{\tau_b}$  and targets $Q_{\theta'_b}$ for all $b\in[B]$.
	\FOR{$t=0,1,2,\dots,$}
	\STATE {\bf Sampling step.}
	\begin{ALC@g}
		\STATE Compute allocation $\omega^{(t)} \gets {\tt ComputeAllocation}(s_t,\{Q_{\theta_b} ,M_{\tau_b}\}_{b\in [B]},\deltaminestimatet{t},\gamma,\lambda, k, \epsilon_t)$.
		\STATE Sample $a_t\sim \omega^{(t)}(s_t,\cdot)$ and observe $(r_t,s_{t+1})\sim q(\cdot|s_t,a_t)\otimes P(\cdot|s_t,a_t)$.
		\STATE Add transition $z_t=(s_t,a_t,r_t,s_{t+1})$ to the replay buffer ${\cal D}$.
	\end{ALC@g}
	\STATE {\bf Training step.}
	\begin{ALC@g}
		\STATE  Sample a batch ${\cal B}$ from ${\cal D}$, and with probability $p$ add the $i^{th}$ experience in ${\cal B}$ to a sub-batch ${\cal B}_b$, $\forall b\in [B]$.   Update the $(Q,M)$-values  of the $b^{th}$ member in the ensemble using ${\cal B}_b$: $\{Q_{\theta_b},Q_{\theta_b'},M_{\tau_b}\}_{b\in [B]} \gets {\tt Training}(\{{\cal B}_b,Q_{\theta_b},Q_{\theta_b'},M_{\tau_b}\}_{b\in [B]})$.
		\STATE Update estimate $\deltaminestimatet{t+1} \gets {\tt EstimateMinimumGap}(\deltaminestimatet{t}, {\cal B}, \{Q_{\theta_b}\}_{b\in[B]})$.
	\end{ALC@g}
	\ENDFOR
\end{algorithmic}
\end{algorithm}
\subsection{Extension to Deep Reinforcement Learning}
To extend bootstrapped \textsc{MF-BPI} to continuous MDPs,  we propose \textsc{DBMF-BPI} (see \cref{algo:dbomfbpi}, or \ifdefined\includeappendix \cref{appB} \else Appendix B\fi). \textsc{DBMF-BPI} uses the mechanism of prior networks from  \textsc{BSP} \cite{osband2018randomized}(bootstrapping with additive prior) to account for uncertainty that does not originate from the observed data.
As before, we keep an ensemble $\{Q_{\theta_1},\dots,Q_{\theta_B}\}$ of $Q$-values (with their target networks) and an ensemble $\{M_{\tau_1},\dots,M_{\tau_B}\}$ of $M$-values, as well as their prior networks.
We use the same procedure as in the tabular case to compute $(\hat Q_t, \hat M_t)$ at time $t$, except that we sample $\xi \sim {\cal U}([0,1])$ every $T_s\propto (1-\gamma)^{-1}$ training steps (or at the end of an episode) to make the training procedure more stable. The quantity $\hat Q_t$ is used to compute  $\pi_{t}^\star(s_t)$ and  $\Delta_{t}(s_t,a)$. We estimate $\deltaminestimatet{t}$  via stochastic approximation, with the minimum gap from the last batch of transitions sampled from the replay buffer serving as a target. To derive the exploration strategy, we compute $H_{t}(s_t,a) = \frac{2+8\varphi^2 \hat M_{t}(s_t,a)^{2^{1- k}} }{(\Delta_{t}(s_t,a)+\lambda)^2}$ and $H_t=\frac{ 4 (1+\gamma)^2\max(1,4\gamma^2\varphi^2 \hat M_{t}(s_t,\pi_{t}^\star(s_t))^{2^{1- k}})}{ (\deltaminestimatet{t}+\lambda)^2(1-\gamma)^2}$. Next, we set the allocation $\omega_o^{(t)}$ as follows:  $\omega_o^{(t)}(s_t,a) = H_{t}(s_t,a)$ if $a\neq \pi_{t}^\star(s_t)$ and $\omega_o^{(t)}(s_t,a) = \sqrt{H_{t} \sum_{a\neq \pi_{t}^\star(s_t)} H_{t}(s_t,a)}$ otherwise. Finally, we obtain an $\epsilon_t$-soft exploration policy $\omega^{(t)}(s_t,\cdot)$ by mixing $\omega_o^{(t)}(s_t,\cdot)/\sum_{a} \omega_o^{(t)}(s_t,a)$ with a uniform distribution (using an exploration parameter $\epsilon_t$).

\section{Numerical Results}\label{sec:numerical_results}

We evaluate the performance of MF-BPI on benchmark problems and compare it against state-of-the-art methods (details can be found in \ifdefined \includeappendix \cref{appA} \else Appendix A\fi).

\textbf{Tabular MDPs.} In the tabular case, we compared various algorithms on the {\tt Riverswim} and {\tt Forked Riverswim} environments. We evaluate \textsc{MF-BPI} with (1) bootstrapping and with (2) the forced exploration step using an $\epsilon$-soft exploration policy, \textsc{MDP-NaS} \cite{marjani2021navigating}, {\sc PSRL} \cite{osband2013more} and \textsc{Q-UCB} \cite{jin2018q,dong2019q}. For {\sc MDP-NaS},  the model of the MDP was initialized in an optimistic way (with additive smoothing). In both environments, we varied the size of the state space. In \cref{fig:riverswim_results}, we show $1- \frac{\|V^\star-V^{\pi_T^\star}\|_\infty}{\|V^\star\|_\infty}$, a performance measure for the estimated policy $\pi_T^\star$ after $T=|S|\times 10^4$ steps with $\gamma=0.99$. Results (the higher the better) indicate that bootstrapped {\sc MF-BPI} can compete with model-based and model-free algorithms on hard-exploration problems, without resorting to expensive model-based procedures. Details of the experiments, including the initialization of the algorithms, are provided in \ifdefined \includeappendix \cref{appA} \else Appendix A\fi.

\begin{figure}[h]
	\centering
	\includegraphics[width=\linewidth]{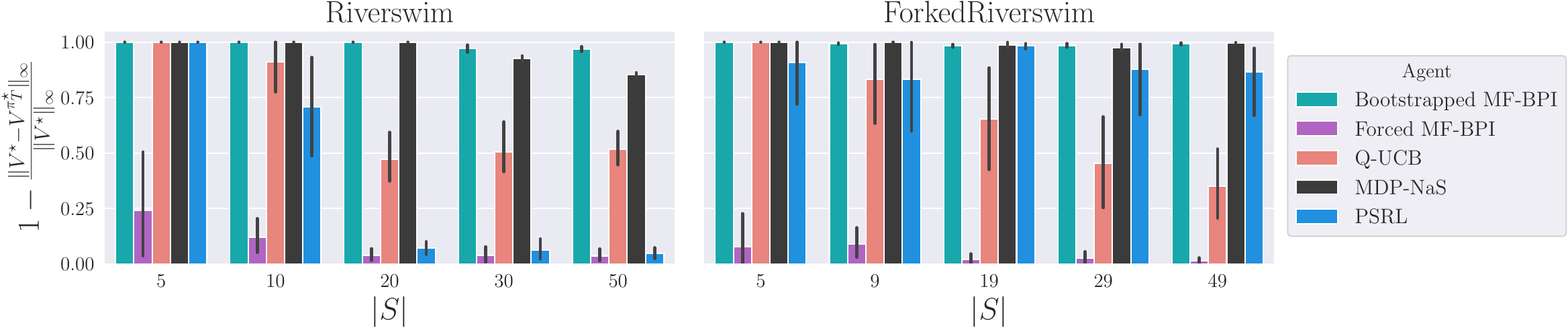}
	\caption{Evaluation of the estimated optimal policy $\pi_T^\star$ after $T$ steps for {\sc MF-BPI, Q-UCB, MDP-NaS} and {\sc PSRL}. Results are averaged across 10 seeds and lines indicate $95\%$ confidence intervals. }
	\label{fig:riverswim_results}
\end{figure}

\begin{figure}[b]
	\centering
	\includegraphics[width=0.45\linewidth]{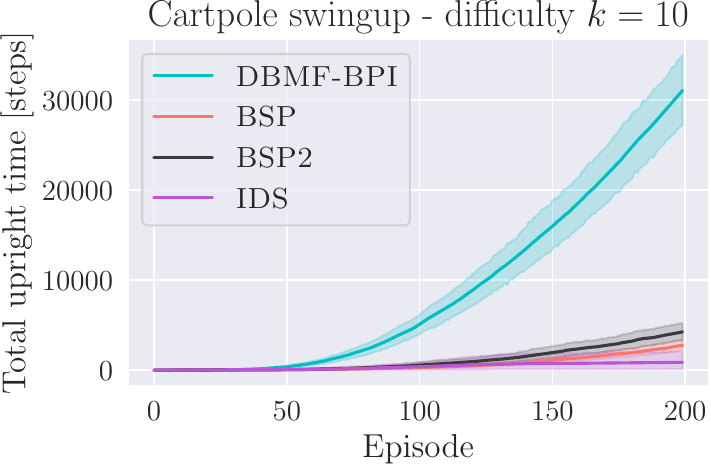}
	\includegraphics[width=0.45\linewidth]{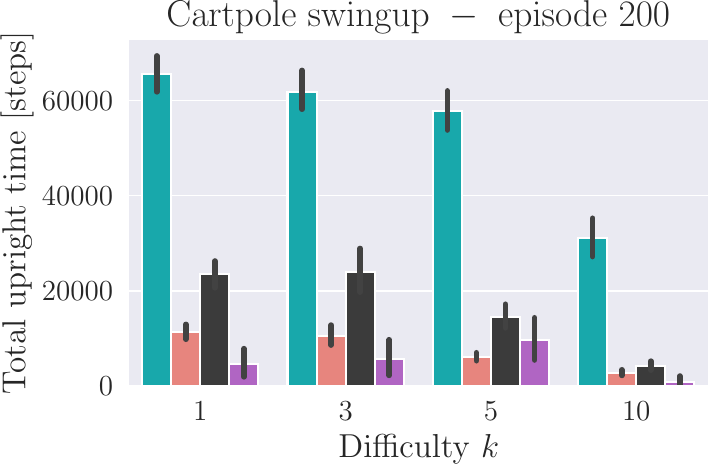}
	\caption{Cartpole swingup problem. On the left: total upright time at a difficulty level of $k=10$. On the right: total upright time after $200$ episodes for different difficulties $k$. To observe a positive reward, the pole's angle must satisfy $\cos(\theta) > k/20$, and the cart's position should satisfy $|x|\leq 1-k/20$. Bars and shaded areas indicate $95\%$ confidence intervals. }
	\label{fig:cartpole_results}
\end{figure}
\textbf{Deep RL.}  In environments with continuous state space, we compared  \textsc{DBMF-BPI} with \textsc{BSP} \cite{osband2019deep,osband2018randomized} (Bootstrapped DQN with
randomized priors) and \textsc{IDS} \cite{nikolov2018information} (Information-Directed Sampling). We also evaluated \textsc{DBMF-BPI}  against {\sc BSP2}, a variant of {\sc BSP} that uses the same masking mechanism as \textsc{DBMF-BPI} for updating the ensemble. These methods were tested on challenging exploration problems from the DeepMind behavior suite \cite{osband2020bsuite} with varying levels of difficulty: (1) a stochastic version of  DeepSea and (2) the Cartpole swingup problem. The DeepSea problem includes a $5\%$ probability of the agent slipping, {i.e.}, that an incorrect action is executed, which increases the aleatoric variance.
\begin{figure}[t]
	\centering
	\includegraphics[width=.9\linewidth]{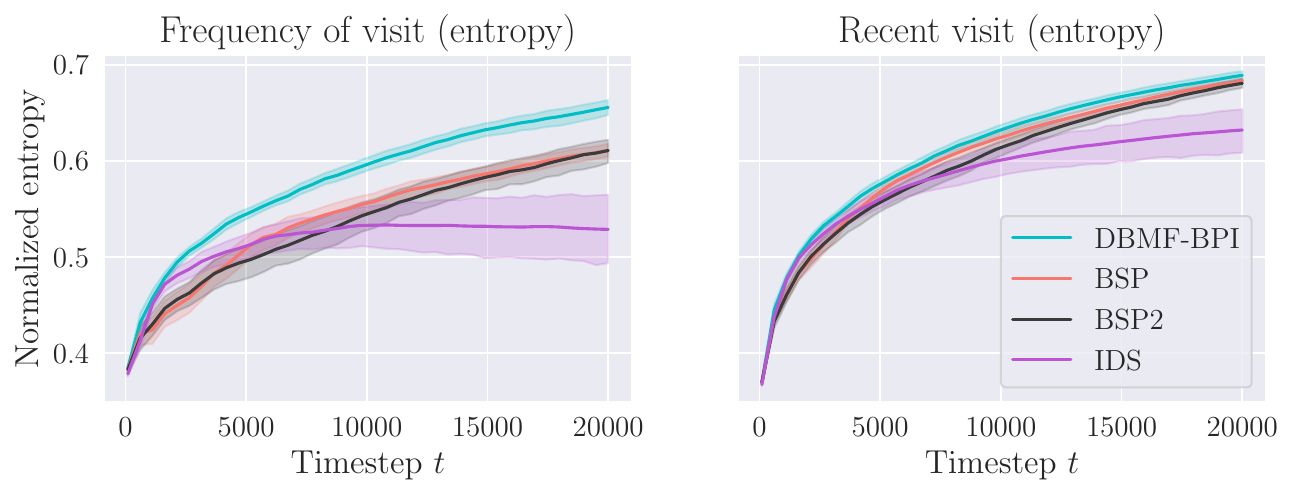}
    \caption{Exploration in Cartpole swingup for $k=5$. On the left, we show the entropy of visitation frequency for the state space \( (x, \dot{x}, \theta, \dot{\theta}) \) during training. On the right,  we show a measure of the dispersion of the most recent visits; smaller values indicate that the agent is less explorative as \( t \) increases.}
	\label{fig:cartpole_entropy}
\end{figure}

The results for the Cartpole swingup problem are depicted in Figure \ref{fig:cartpole_results} for various difficulty levels $k$ (see also \ifdefined\includeappendix \cref{sec:cartpole_appendix} \else Appendix A.5 \fi for more details), demonstrating the ability of {\sc DBMF-BPI} to quickly learn an efficient policy. While {\sc BSP} generally performs well, there is a notable difference in performance when compared to {\sc DBMF-BPI}. For a fair comparison, we used the same network initialization across all methods, except for IDS. Untuned, IDS performed poorly; proper initialization improved its performance, but results remained unsatisfactory. In  \cref{fig:cartpole_entropy}, we present two exploration metrics for difficulty $k=5$. The frequency of visits measures the uniformity and dispersion of visits across the state space, while the second metric evaluates the recency of visits to different regions, capturing how frequently the methods keep visiting previously visited states (a smaller value indicates that the agent tends to concentrate on a specific region of the state space). For detailed analysis, please refer to appendix A.
 
\begin{figure}[h]
	\centering
	\includegraphics[width=0.45\linewidth]{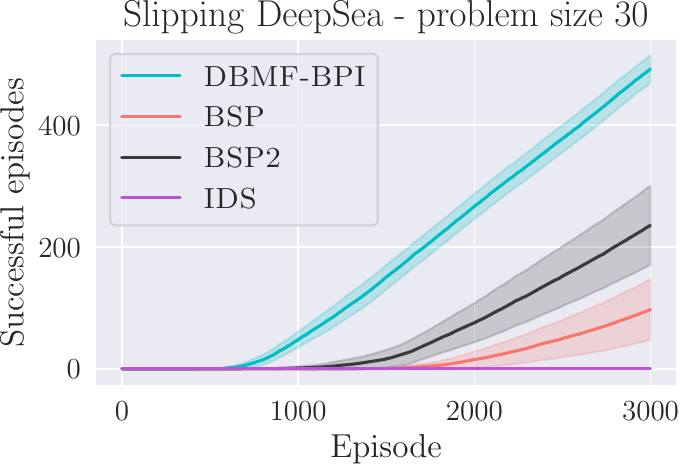}
 \includegraphics[width=0.45\linewidth]{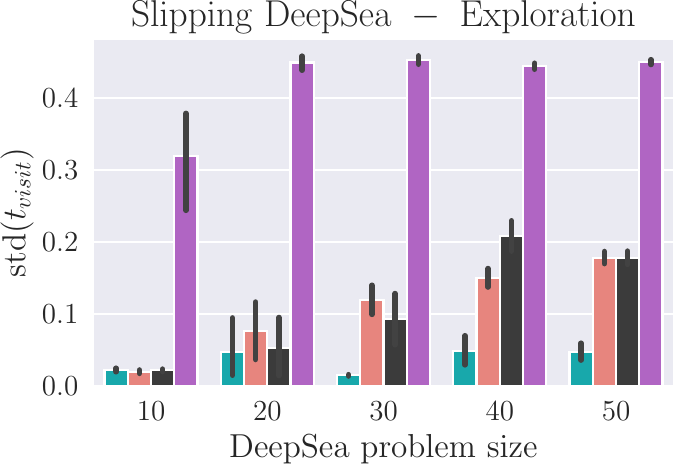}
	\caption{Slipping DeepSea problem. On the left: total number of successful episodes (\emph{i.e.}, that the agent managed to reach the final reward) for a grid with $30^2$ input features. On the right: standard deviation of $t_{{\rm visit}}$ at the last episode, depicting how much each agent explored (the lower the better).}
	\label{fig:deepsea_results}
\end{figure}

For the slipping DeepSea problem, results are depicted in Fig. \ref{fig:deepsea_results} (see also  \ifdefined\includeappendix \cref{sec:deepsea_appendix} \else Appendix A.4 \fi for more details). Besides the number of successful episodes, we also display the standard deviation of $(t_{{\rm visit}})_{ij}$ across all cells $(i,j)$, where $(t_{{\rm visit}})_{ij}$ indicates the last timestep $t$ that a cell $(i,j)$ was visited (normalized by $NT$, the product of the grid size, and the number of episodes). The right plot shows $\mathrm{std}(t_{{\rm visit}})$ for different problem sizes, highlighting the good exploration properties of {\sc DBMF-BPI}. Additional details and exploration metrics can be found in \ifdefined\includeappendix \cref{appA} \else Appendix A\fi.

\section{Conclusions}\label{sec:conclusions}
In this work, we studied the problem of exploration in Reinforcement Learning and presented {\sc MF-BPI},  a  model-free solution for both tabular and continuous state-space MDPs.
To derive this method, we established a novel approximation of the instance-specific lower bound necessary for identifying nearly-optimal policies. Importantly, this approximation depends only on quantities learnable via stochastic approximation, paving the way towards model-free methods. Numerical results on hard-exploration problems highlighted the effectiveness of our approach for learning efficient policies over state-of-the-art methods.
\newpage
\section*{Acknowledgments}
This research was supported by the Swedish Foundation for Strategic Research through the CLAS project (grant RIT17-0046) and partially supported by the Wallenberg AI, Autonomous Systems and Software Program
(WASP) funded by the Knut and Alice Wallenberg Foundation. The authors would also like to thank the anonymous reviewers for their valuable and insightful feedback.  On a personal note, Alessio Russo wishes to personally thank Damianos Tranos, Yassir Jedra, Daniele Foffano, and Letizia Orsini for their invaluable assistance in reviewing the manuscript.
\bibliography{main}
\bibliographystyle{plainnat}
\ifdefined \includeappendix
 \newpage
 \appendix
 \part*{Appendix}
 \tableofcontents
\newpage
\subsection*{Appendix introduction}
We start by examining the wider impact of our work and acknowledging its limitations. This provides a balanced view of our contribution and points out areas for future research.

Next, we turn to the numerical results. Here, we give a more detailed account of our findings and include additional results for further clarity. We also introduce and describe the new Forked RiverSwim environment, an advanced version of the existing RiverSwim model, which has a larger  sample complexity.

In the subsequent section, we break down the algorithms used in our study. This gives a deeper understanding of the methods underpinning our research.

We wrap up the appendix by providing all the proofs that support our conclusions. 

 \subsection*{Broader impact}
 This paper primarily focuses on foundational research in reinforcement learning, specifically the exploration problem, and proposes a novel model-free exploration strategy. While our work does not directly engage with societal impact considerations, we acknowledge the importance of considering the broader implications of AI technologies. As our proposed method improves the efficiency of reinforcement learning algorithms, it could potentially be applied in a wide range of contexts, some of which could have societal impacts. For instance, reinforcement learning is used in decision-making systems, which could include areas like healthcare, finance, and autonomous vehicles, where biases or errors could have significant consequences. Hence, while the direct societal impact of our work may not be immediately apparent, we strongly encourage future researchers and practitioners who apply these techniques to carefully consider the ethical implications and potential negative impacts in their specific use-cases. The responsible use of AI, including the mitigation of bias and the respect for privacy, should always be a priority.
 \subsection*{Limitations}
 While our work presents significant advancements in the area of reinforcement learning, it also has its limitations that need to be acknowledged:

 \begin{itemize}
 	\item  {\bf Assumptions}: Our approach relies on  the assumption that the MDP is communicating. The instance-specific lower bound we propose may not be as effective if this assumption does not hold.
 	\item  {\bf Scalability}: Our method, despite being model-free, still relies on stochastic approximations, which may not scale well with the complexity and size of certain MDPs.
 	\item  {\bf Comparison with Model-Based Approaches}: While we have shown that our approach performs competitively with existing model-based exploration algorithms in hard-exploration environments, a comprehensive comparison across a wider range of environments is needed. It is possible that our method may not perform as well in some MDPs as the model-based approaches.
 	\item  {\bf Bootstrapping}: Although bootstrapping has proven to be an effective technique, its usage is yet to be fully understood in RL applications. To achieve a more profound theoretical comprehension, a comprehensive analysis is necessary.
 \end{itemize}

 These limitations present opportunities for future research and the continued evolution of efficient exploration in reinforcement learning.

 \newpage

\section{Numerical Results}\label{appA}

The appendix begins with the numerical results. We first introduce the Forked RiverSwim environment, a more complex variant of the traditional RiverSwim model.

Our discussion continues with a detailed exposition of \cref{example:randomly_drawn_mdp_value}, providing further experimental details. We conclude this section with additional findings related to both the tabular case and two specific problems: the CartPole Swing-Up and the Slipping DeepSea.

\subsection{The Forked Riverswim Environment}\label{subsec:forked_riverswim}
The \texttt{Forked RiverSwim}$(N)$ is a novel environment (see also \cref{fig:forked_riverswim}) where the agent needs to constantly explore two different states, $(s_g,s_g')$, to learn the optimal policy. The number of states is $2N-1$, and there are $3$ actions.

The environment is similar to  \texttt{RiverSwim}, but the initial state $s_1$ forks into two rivers: the final state in both branches of the rivers ($s_g$ and $s_g'$) have a similar  high reward. Furthermore, the agent can deterministically switch between the two branches at any intermediate state. Intermediate states do not give any reward. Moreover, a little subtlety is that the agent can exploit the deterministic transition between $s_1$ and $s_2'$ to deterministically transition to $s_2$ (although this has a small effect as $N$ grows large).

Lastly, the Bernoulli rewards in $s_g$ and $s_g'$, which are the {\it highly rewarding} states, are quite similar ($1$ vs $0.95$). Therefore, an optimal policy that starts in $s_1$ should achieve a slightly better reward than the optimal policy on the {\tt RiverSwim} environment with $N+1$ states (due to the fact that the transition to $s_2$ from $s_1$ can be made in a deterministic way).

Due to these reasons, this variant introduces additional complexity into the decision-making process. It is reasonable that a learning algorithm  may learn an approximately good greedy policy in a short time-span, but not exactly the optimal one. In fact, we may expect an algorithm to take longer (compared to {\tt Riverswim}) to learn the true optimal policy. Finally, always compared to  \texttt{RiverSwim}, the sample complexity is of orders of magnitude higher, as also depicted in \cref{fig:example_upper_bound}. For a Python implementation, please refer to the GitHub repository of this manuscript.
\begin{figure}[h]
	\centering
	\includegraphics[width=\linewidth]{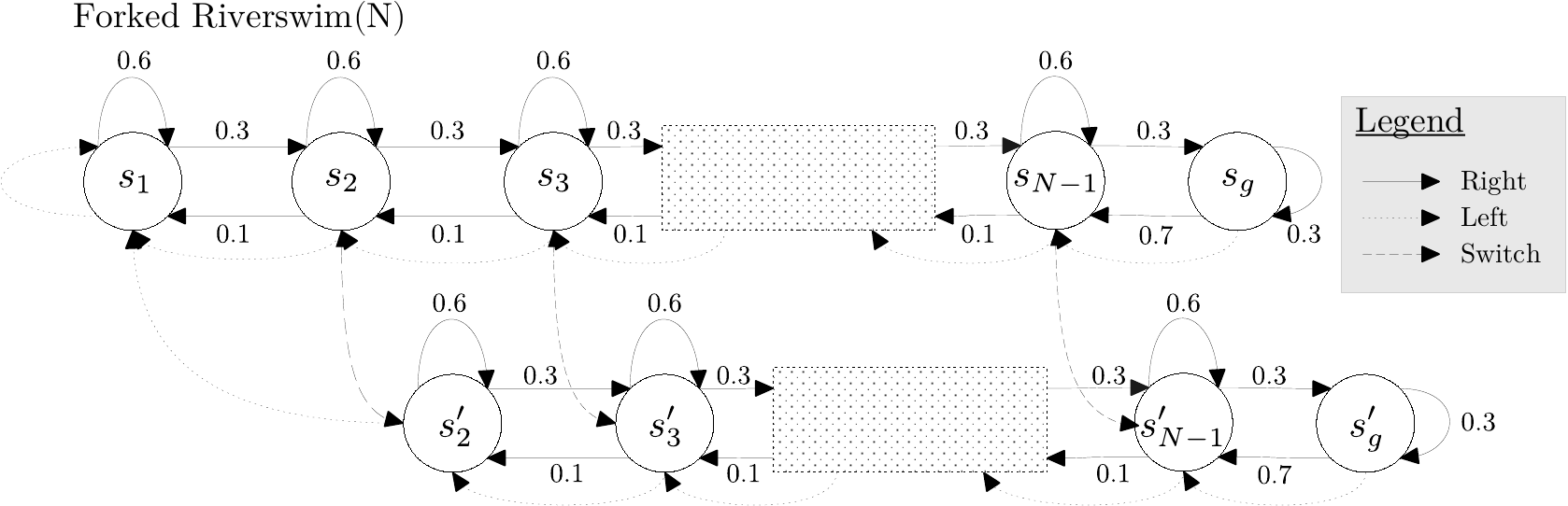}
	\caption{\texttt{Forked Riverswim}$(N)$  with $|S|=2N-1$ states. When taking action \texttt{left} in $s_1$ the agent observes a Bernoulli reward $r$ of parameter $0.05$. When taking action \texttt{right} in $s_g$ (resp. $s_g'$) the agent observes a reward $r$ drawn from a Bernoulli of parameter $1$ (resp. $0.95$). In all other states the reward is $0$. Action \texttt{left} and \texttt{switch} are deterministic, while the probability of action \texttt{Right} is indicated in the figure. The square boxes indicate that the pattern of states is being repeated from $s_3$ (or $s_3'$) until $s_{N-1}$ (or $s_{N-1}'$). This variant introduces additional complexity into the decision-making process, as the Bernoulli rewards in $s_g$ and $s_g'$ are quite similar ($1$ vs $0.95$).}
	\label{fig:forked_riverswim}
\end{figure}
\subsection{Details of Example \ref{example:randomly_drawn_mdp_value}}
In the following we report the details of \cref{example:randomly_drawn_mdp_value}.
In \cref{example:randomly_drawn_mdp_value} we evaluated the characteristic time of three different environments with same discount factor $\gamma=0.95$:
\begin{enumerate}
	\item {\tt RandomMDP}: an MDP with $|S|$ states and $3$ actions. The transition probability for each $(s,a)$ is drawn from a Dirichlet distribution ${\rm Dir}(\alpha_1,\dots,\alpha_{|S|})$, with $\alpha_i=\alpha_{i-1}+(i-1)/10$ and $\alpha_1=1$. The rewards also follow the same Dirichlet distribution, that is, for each $(s,a)$ we sample a $|S|$-dimensional vector $q$ of rewards from ${\rm Dir}(\alpha_1,\dots,\alpha_{|S|})$. This vector defines the rewards in the next state $r(s,a,s')=q_{s'}$, with $q\sim {\rm Dir}(\alpha_1,\dots,\alpha_{|S|})$. 	For this type of environment see also the details of the instance-specific quantities  in \cref{tab:randommdp_details}.

	\item {\tt RiverSwim}: this environment is specified in \cite{strehl2008analysis}, but we refer to the version used in \cite{marjani2021navigating} for a direct comparison. 	The reward is always $0$ except in the initial state $s_1$, and the final state $s_{|S|}$.  In the initial state we have $q(1|s_1, \textrm{left})=0.05$ (probability $0.05$ of observing a reward of $1$, and $0.95$ probability of observing a reward of $0$), while in the final state $q(1|s_{|S|},\textrm{right})=1$. All other rewards are set to $0$. Transition probabilities are the same as in \cite{strehl2008analysis}. 	For this type of environment see also the details of the instance-specific quantities  in \cref{tab:riverswim_details}.
	\item {\tt Forked RiverSwim}: we refer the reader to \cref{subsec:forked_riverswim} for a description of this environment. 	For this type of environment see also the details of the instance-specific quantities  in \cref{tab:forked_riverswim_details}.
\end{enumerate}

Interestingly, these environments have different properties that make them suitable for analysis: (1) the {\tt RandomMDP} environment has very small gaps and variances; (2) the {\tt Riverswim} environment has a relatively larger maximum span; (3) the {\tt Forked Riverswim} environment, in contrast to the {\tt Riverswim} environment, has a very small minimum gap $\deltamin$ and similar values for the span.

\begin{figure}[t]
	\centering
	\includegraphics[width=\columnwidth]{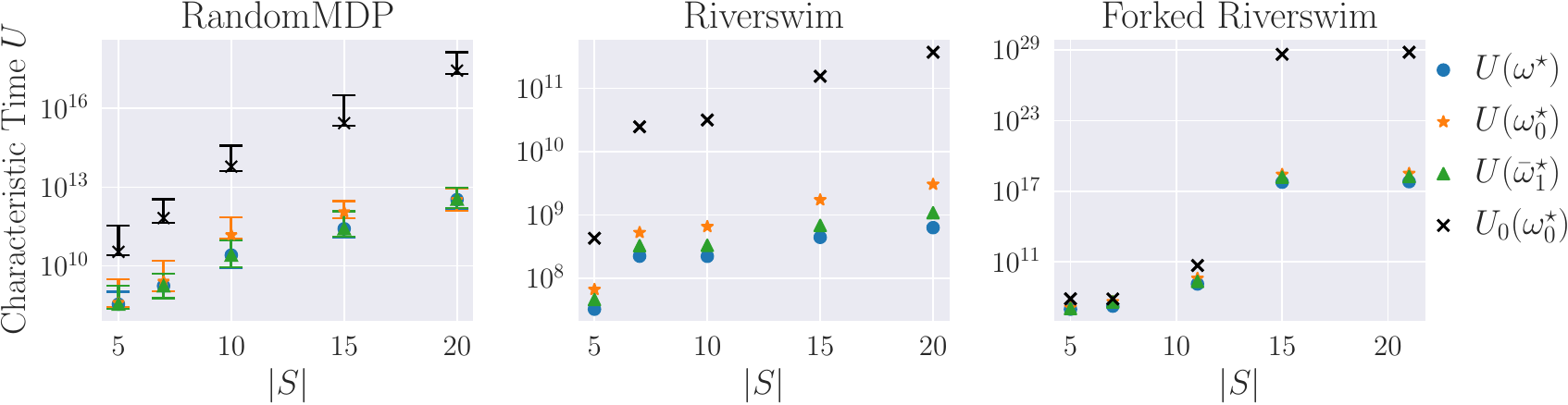}
	\caption{Comparison of (\ref{eq:original_upper_bound}) and (\ref{eq:new_upper_bound}) for discount $\gamma=0.99$ and different sizes of the state space $S$. We evaluated different allocations using $ U_0(\omega)$ and $U(\omega)$. The allocations are: $ \omega^\star$ (the optimal allocation in \cref{eq:new_upper_bound}), $\omega_0^\star$ (the optimal allocation in \cref{eq:original_upper_bound}) and $ \omega_1^\star$ (the optimal allocation that we get from (\ref{eq:new_upper_bound}) by setting $ k=1$ uniformly across states and actions). Results for the RandomMDP  indicate the median and the bars 95\% confidence intervals across 30 runs.}
	\label{fig:example_upper_bound_appendix}
\end{figure}
\begin{table}[b]
	\centering
	\setlength\arrayrulewidth{1pt}
	\resizebox{\textwidth}{!}{
		\begin{tabular}{c|cccccccc}
			\toprule
			$|S|$ & $\deltamin$ & $\max_{sa} \Delta_{sa}$ & $\min_{sa} \md_{sa}[V^\star]$ & $\max_{sa} \md_{sa}[V^\star]$ & $\min_{sa} \var_{sa}[V^\star]$ & $\max_{sa} \var_{sa}[V^\star]$ & $\max_{s,a,k} M_{sa}^k[V^\star]^{2^{-k}}$ \\ \hline
 $5$ & $1.1\cdot 10^{-2}$ & $1.6\cdot 10^{-1}$ & $6.4\cdot 10^{-2}$ & $1.0\cdot 10^{-1}$ & $8.3\cdot 10^{-4}$ & $3.4\cdot 10^{-3}$ & $1.0\cdot 10^{-1}$ \\
$10$ & $2.3\cdot 10^{-3}$ & $6.3\cdot 10^{-2}$ & $2.7\cdot 10^{-2}$ & $3.6\cdot 10^{-2}$ & $1.4\cdot 10^{-4}$ & $3.7\cdot 10^{-4}$ & $3.6\cdot 10^{-2}$ \\
 $25$ & $1.2\cdot 10^{-4}$ & $1.0\cdot 10^{-2}$ & $4.6\cdot 10^{-3}$ & $5.1\cdot 10^{-3}$ & $3.3\cdot 10^{-6}$ & $4.9\cdot 10^{-6}$ & $5.1\cdot 10^{-3}$ \\
$50$ & $9.5\cdot 10^{-6}$ & $1.9\cdot 10^{-3}$ & $9.1\cdot 10^{-4}$ & $9.5\cdot 10^{-4}$ & $1.2\cdot 10^{-7}$ & $1.4\cdot 10^{-7}$ & $9.5\cdot 10^{-4}$ \\
 $100$ & $1.1\cdot 10^{-6}$ & $3.7\cdot 10^{-4}$ & $1.8\cdot 10^{-4}$ & $1.8\cdot 10^{-4}$ & $0$ & $0$ & $1.8\cdot 10^{-4}$ \\
			\bottomrule
		\end{tabular}
	}
	\vspace{1ex}
	\caption{Details of the instance-specific quantities for the \texttt{RandomMDP} environment  (we evaluated up to $k=19$). Results indicate an average over 300 different realizations. Confidence intervals are omitted for brevity, and values are rounded up to the 1st decimal.}
	\label{tab:randommdp_details}
\end{table}
\begin{table}[h]
	\centering
	\setlength\arrayrulewidth{1pt}
	\resizebox{\textwidth}{!}{
		\begin{tabular}{c|cccccccc}
			\toprule
			$|S|$ & $\deltamin$ & $\max_{sa} \Delta_{sa}$ & $\min_{sa} \md_{sa}[V^\star]$ & $\max_{sa} \md_{sa}[V^\star]$ & $\min_{sa} \var_{sa}[V^\star]$ & $\max_{sa} \var_{sa}[V^\star]$ & $\max_{s,a,k} M_{sa}^k[V^\star]^{2^{-k}}$ \\ \hline
		$5$ & $7.6\cdot 10^{-2}$ & $1.3\cdot 10^{0}$ & $1.7\cdot 10^{0}$ & $3.0\cdot 10^{0}$ & $0$ & $3.6\cdot 10^{-1}$ & $1.1\cdot 10^{0}$ \\
 $10$ & $3.4\cdot 10^{-2}$ & $1.3\cdot 10^{0}$ & $2.5\cdot 10^{0}$ & $4.5\cdot 10^{0}$ & $0$ & $3.7\cdot 10^{-1}$ & $1.1\cdot 10^{0}$ \\
$25$ & $1.9\cdot 10^{-2}$ & $1.3\cdot 10^{0}$ & $2.5\cdot 10^{0}$ & $5.0\cdot 10^{0}$ & $0$ & $3.7\cdot 10^{-1}$ & $1.1\cdot 10^{0}$ \\
 $50$ & $8.4\cdot 10^{-3}$ & $1.3\cdot 10^{0}$ & $2.7\cdot 10^{0}$ & $5.4\cdot 10^{0}$ & $0$ & $3.7\cdot 10^{-1}$ & $1.1\cdot 10^{0}$ \\
$100$ & $2.1\cdot 10^{-4}$ & $1.3\cdot 10^{0}$ & $2.9\cdot 10^{0}$ & $5.5\cdot 10^{0}$ & $0$ & $3.7\cdot 10^{-1}$ & $1.1\cdot 10^{0}$ \\
			\bottomrule
		\end{tabular}
	}
	\vspace{1ex}
	\caption{Details of the instance-specific quantities for the \texttt{Riverswim} environment  (we evaluated up to $k=19$). Values are rounded up to the 1st decimal.}
	\label{tab:riverswim_details}
\end{table}

\begin{table}[h]
	\centering
	\setlength\arrayrulewidth{1pt}
	\resizebox{\textwidth}{!}{
		\begin{tabular}{c|cccccccc}
			\toprule
			$|S|$ & $\deltamin$ & $\max_{sa} \Delta_{sa}$ & $\min_{sa} \md_{sa}[V^\star]$ & $\max_{sa} \md_{sa}[V^\star]$ & $\min_{sa} \var_{sa}[V^\star]$ & $\max_{sa} \var_{sa}[V^\star]$ & $\max_{s,a,k} M_{sa}^k[V^\star]^{2^{-k}}$ \\ \hline
	 $5$ & $1.0\cdot 10^{-1}$ & $1.4\cdot 10^{0}$ & $1.0\cdot 10^{0}$ & $2.0\cdot 10^{0}$ & $0$ & $3.2\cdot 10^{-1}$ & $1.0\cdot 10^{0}$ \\
 $11$ & $2.8\cdot 10^{-2}$ & $1.3\cdot 10^{0}$ & $1.6\cdot 10^{0}$ & $2.9\cdot 10^{0}$ & $0$ & $4.9\cdot 10^{-1}$ & $2.0\cdot 10^{0}$ \\
 $25$ & $1.0\cdot 10^{-6}$ & $1.3\cdot 10^{0}$ & $1.7\cdot 10^{0}$ & $3.2\cdot 10^{0}$ & $0$ & $4.8\cdot 10^{-1}$ & $2.0\cdot 10^{0}$ \\
 $51$ & $1.0\cdot 10^{-6}$ & $1.3\cdot 10^{0}$ & $2.1\cdot 10^{0}$ & $4.2\cdot 10^{0}$ & $0$ & $4.8\cdot 10^{-1}$ & $2.0\cdot 10^{0}$ \\
$101$ & $1.0\cdot 10^{-6}$ & $1.3\cdot 10^{0}$ & $2.5\cdot 10^{0}$ & $5.1\cdot 10^{0}$ & $0$ & $4.8\cdot 10^{-1}$ & $2.0\cdot 10^{0}$ \\
			\bottomrule
		\end{tabular}
	}
	\vspace{1ex}
	\caption{Details of the instance-specific quantities for the \texttt{Forked Riverswim} environment (we evaluated up to $k=19$). Values are rounded up to the 1st decimal.}
	\label{tab:forked_riverswim_details}
\end{table}

\begin{figure}[H]
	\centering
	\includegraphics[width=\columnwidth]{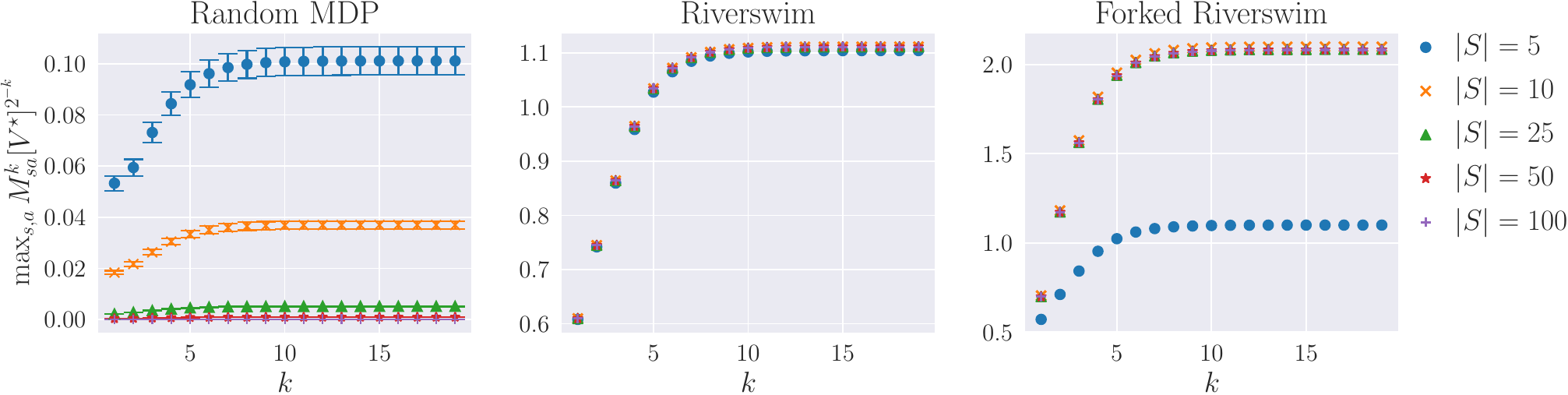}
	\caption{Plot of $\max_{s,a} M_{sa}^k[V^\star]^{2^{-k}}$ for various values of $k$. For the random MDP we depict the median value, as well as the $95\%$ confidence interval.  }
	\label{fig:evaluation_k}
\end{figure}
\begin{figure}[H]
	\centering
	\includegraphics[width=\columnwidth]{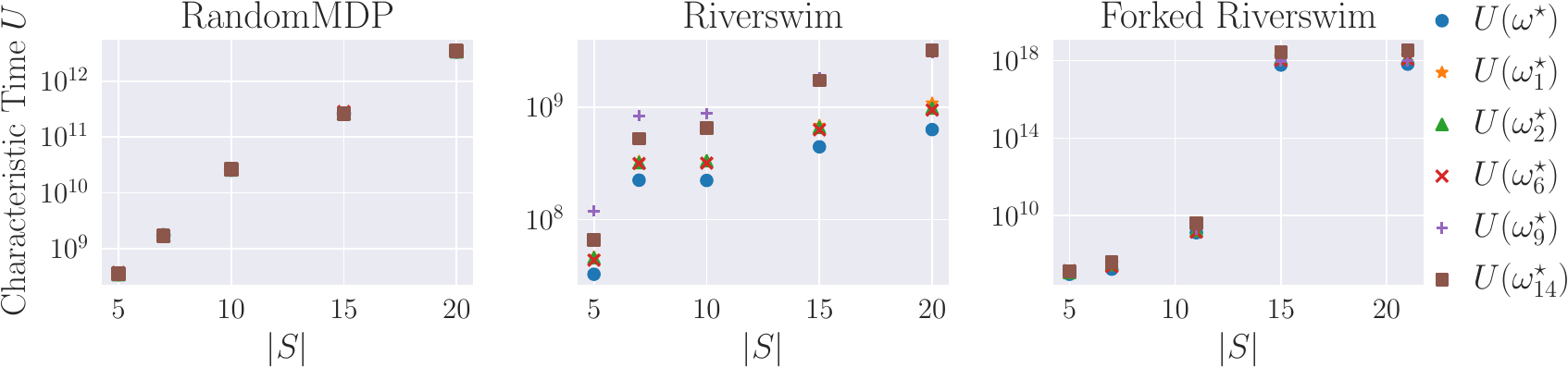}
	\caption{Evaluation of $ \omega_k^\star$ for different values of $k$. For the {\tt RandomMDP} we only show the median value over $300$ runs.}
	\label{fig:evaluation_different_values_k}
\end{figure}

Finally, in \cref{fig:evaluation_k}, we depict $\max_{s,a} M_{sa}^k[V^\star]^{2^{-k}}$ for different values of $k$, up to $k=19$. For the {\tt RandomMDP} environment we observe that $\max_{s,a} M_{sa}^k[V^\star]^{2^{-k}}$ tends to the maximum of the span $\max_{sa} \md_{sa}[V^\star]$, which depends on the size of the state space (as $|S|$ grows larger the span diminishes). For the other two environments, {\tt Riverswim} and {\tt Forked Riverswim}, $\max_{s,a} M_{sa}^k[V^\star]^{2^{-k}}$ does not seem to depend on the size of the state space. Furthermore, we also observe a sudden convergence of this quantity for relatively small values of $k$, followed by a relatively very slow increase.

In \cref{fig:evaluation_different_values_k} are shown the results when we evaluate the allocations $ \omega_k^\star$ for different values of $k$. In general, we do not observe a striking difference between those allocations. 

\newpage
\subsection{Riverswim and Forked Riverswim - Description and Additional Results}

In \cref{fig:cartpole_barplot_2} we present results from the Riverswim and ForkedRiverswim environments. These results include data from two new algorithms: {\sc O-BPI} (Online Best Policy Identification) and {\sc PS-MDP-NaS} (Posterior Sampling for MDP-NaS). 

{\sc O-BPI} is a novel algorithm that draws inspiration from {\sc MDP-NaS}. However, a distinguishing characteristic is its use of stochastic approximation to determine the $Q$-values and $M$-values. These values, as for {\sc MF-BPI}, are used to compute the allocation $\omega$ by solving   the sample-complexity bound $\inf_{\omega \in \Omega(\phi)} U(\omega)$ with  navigation constraints.  On the other hand, {\sc PS-MDP-NaS} is an adaptation of {\sc MDP-NaS} that uses posterior sampling over the MDP's model to address the parametric uncertainty. It's worth noting that both these algorithms,   {\sc O-BPI} and {\sc PS-MDP-NaS}, are model-based, and a detailed description of these algorithms is available in the following section, see \cref{algo:obpi} and \cref{algo:psmdpnas}. 

The results in \cref{fig:cartpole_barplot_2} clearly show the superiority of these allocation computing methods compared to other algorithms such as {\sc PSRL} and {\sc Q-UCB.}
\begin{figure}[h]
	\centering
	\includegraphics[width=\linewidth]{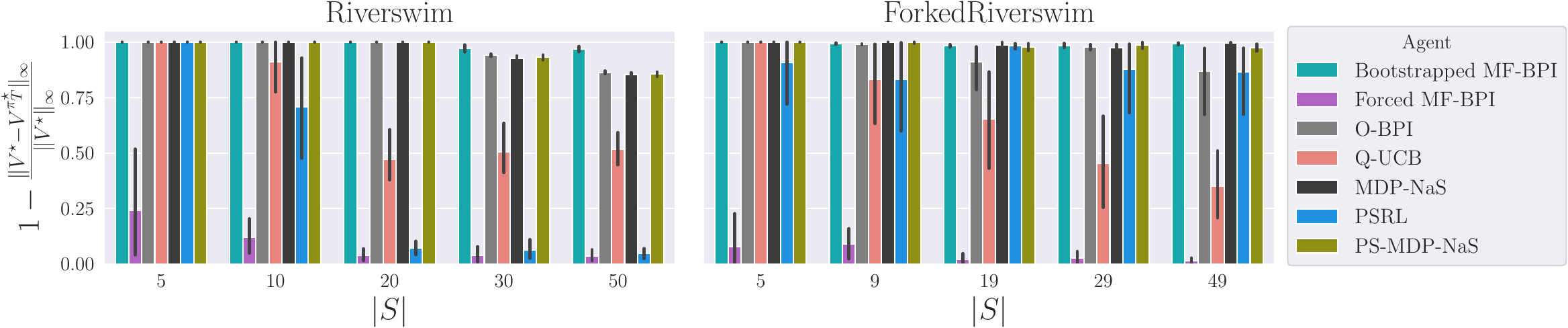}
	\caption{Evaluation of the estimated optimal policy $\pi_T^\star$ after $T$ steps for {\sc MF-BPI, O-BPI, Q-UCB, MDP-NaS, PS-MDP-NaS}, and  {\sc PSRL}. Results are averaged across 10 seeds and lines indicate $95\%$ confidence intervals. Note that for {\tt Forked Riverswim} we have $N=2|S|-1$.
    } 
	\label{fig:cartpole_barplot_2}
\end{figure}

\cref{fig:cartpole_full} on the next page provides a visualization of the performance of each algorithm over the entire horizon  $t=0,\dots, T-1$. We exhibit the performance of the estimated greedy policy $\pi_t^\star$ at each timestep $t$ for each respective method. The results offer a clear demonstration of the efficiency of  those methods based on the instance-specific sample complexity lower bound.

\clearpage

\begin{figure}[h]
	\centering
	\includegraphics[width=.85\linewidth]{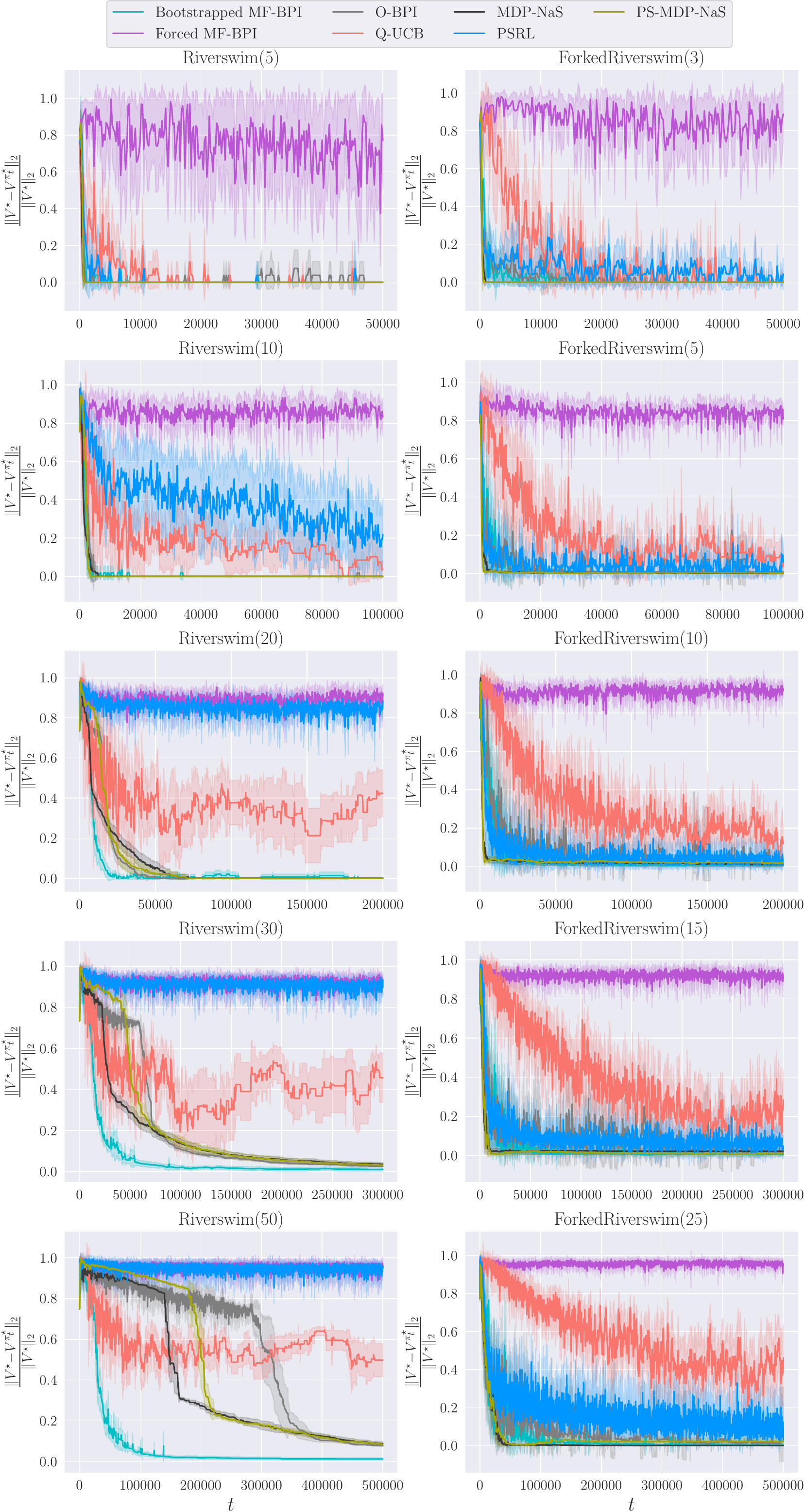}
	\caption{Evaluation of the estimated optimal policy $\pi_t^\star$ for {\sc MF-BPI, O-BPI, Q-UCB, MDP-NaS, PS-MDP-NaS}, and  {\sc PSRL}. Results are averaged across 10 seeds and lines indicate $95\%$ confidence intervals. }
	\label{fig:cartpole_full}
\end{figure}
\clearpage\newpage
\subsection{Slipping DeepSea - Description and Additional Results}\label{sec:deepsea_appendix}
\paragraph{Description.} The Slipping DeepSea problem is an hard-exploration  reinforcement learning problem. In the standard version, there's an $N \times N$ grid, and the agent starts in the top left corner (state 0, 0) and needs to reach the bottom right corner (state $N-1$, $N-1$) for a large reward (the state vector is an $N^2$-dimensional vector, that one-hot encodes the agent's position in the grid).
The agent can move diagonally, left or right (or down when close to the wall).  The agent incurs in a cost when moving of $0.01/N$, while obtaining a positive reward of $1$ when reaching the bottom right corner.
Furthermore, we introduce the modification that there is a small probability of $0.05$ that the incorrect action will be executed.
This is a challenging problem because the optimal policy requires the agent to move  (incurring a negative reward) many times before eventually reaching the high reward in the bottom right corner. However, due to the stochastic nature of the problem (the chance of slipping), the agent might be forced to take suboptimal actions, making it harder to learn the optimal policy.

\paragraph{Additional results.} \Cref{fig:deep_sea_exploration} presents additional metrics encapsulating the exploration conducted by each algorithm, offering a comprehensive summary of the exploration process after $T$ episodes for each size $N$ (note that for a given size $N$ the number of input features in the state is $N^2$).

We focus on two key metrics: (a) $(t_{{\rm visit}})_{ij}$ and (b) $(t_{{\rm avg}})_{ij}$. Here, (a) $(t_{{\rm visit}})_{ij}$ represents the last timestep $t$ at which a cell $(i,j)$ was visited (this value is normalized by $NT$, the multiplication of the grid size and the number of episodes), while (b) $(t{{\rm avg}})_{ij}$ signifies the average frequency with which a cell $(i,j)$ was visited.
\begin{figure}[h]
	\centering
	\includegraphics[width=.65\linewidth]{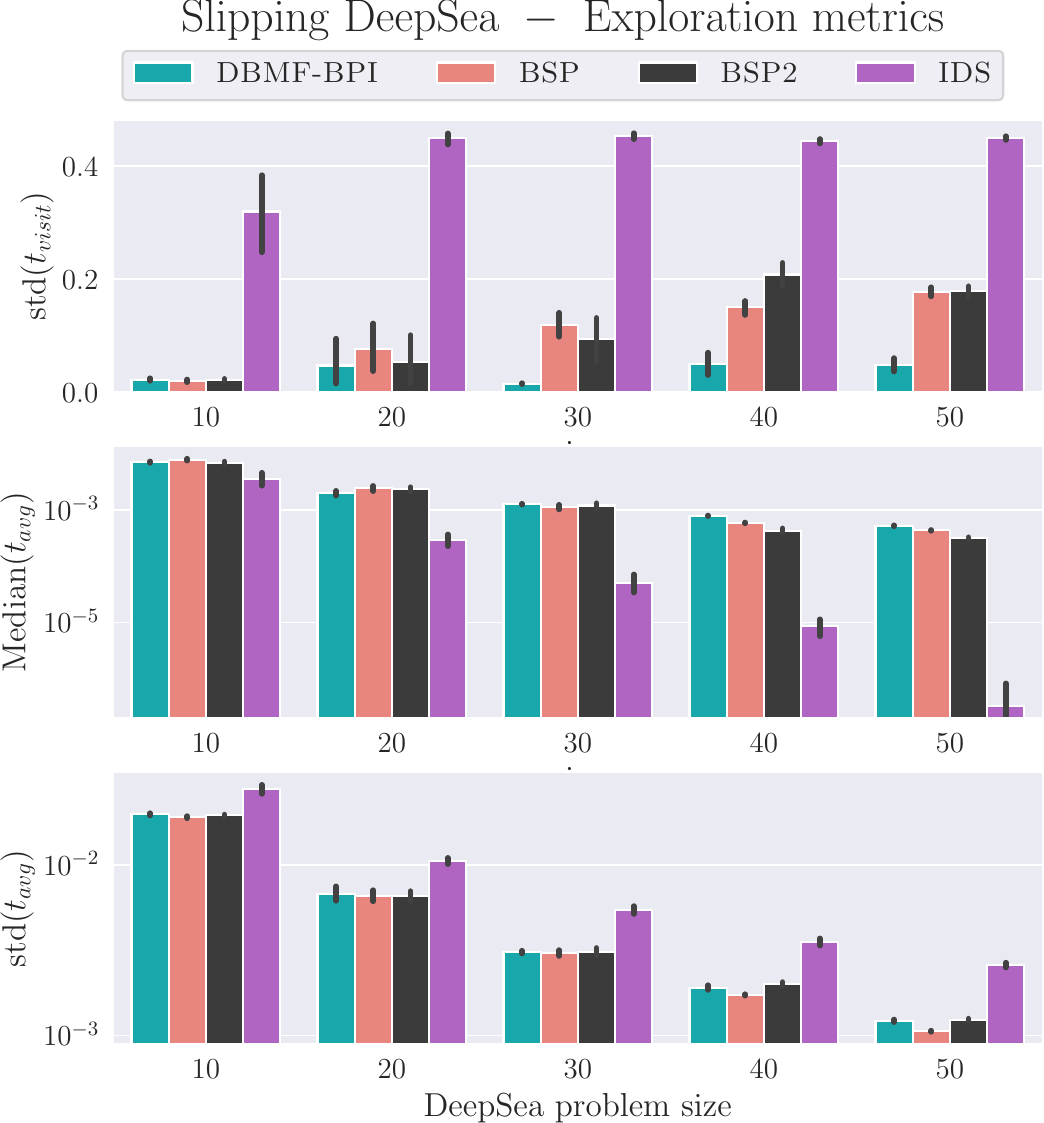}
	\caption{Slipping DeepSea problem - exploration metrics. From top to bottom: (1)standard deviation of $t_{{\rm visit}}$ at the last episode, depicting how much each agent explored (the lower the better); (2) median value of  $(t_{{\rm avg}})_{ij}$, \emph{i.e.}, the median value of a cell's visit frequency; (3) standard deviation of $(t_{{\rm avg}})_{ij}$ across all cells. Results are averaged over $24$ runs and bars indicate $95\%$ confidence intervals.}
	\label{fig:deep_sea_exploration}
\end{figure}
In terms of arrangement, from the top downwards: (1) we present the standard deviation of $(t_{{\rm visit}})_{ij}$ across all cells; (2) we show the median value of a cell's visit frequency; (3) we depict the standard deviation of $(t_{{\rm avg}})_{ij}$ across all cells.
From the central plot, we notice that {\sc DBMF-BPI} tends to visit all cells slightly more frequently. The first plot also highlights that {\sc DBMF-BPI} maintains a consistent visit rate to all cells. This pattern is a strong indication of {\sc DBMF-BPI}'s explorative behavior. Conversely, neither {\sc BSP} nor {\sc BSP2} match this performance in terms of successful episodes (shown in \cref{fig:deep_sea_successful_episodes_appendix}), despite the median value of $t_{{\rm avg}}$ being very similar to that of {\sc DBMF-BPI}.
\begin{figure}[]
	\centering
	\includegraphics[width=.5\linewidth]{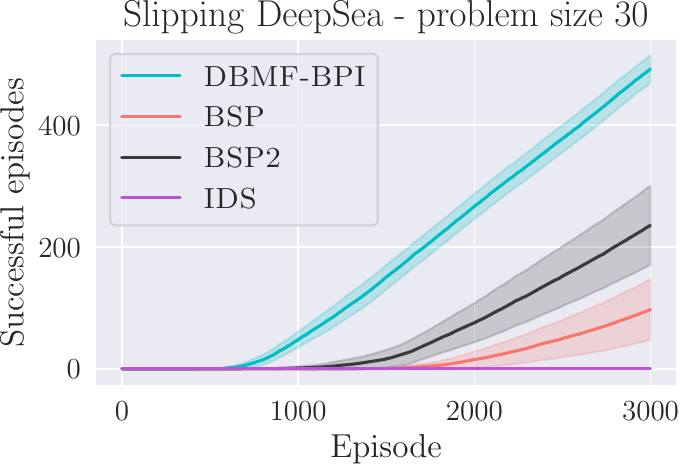}
	\caption{Slipping DeepSea problem. Total number of successful episodes (\emph{i.e.}, that the agent managed to reach the final reward) for a grid with $30^2$ input features. }
	\label{fig:deep_sea_successful_episodes_appendix}
\end{figure}
In order to provide a more comprehensive view, \cref{fig:deepsea_exploration_freq_visit} and \cref{fig:deepsea_exploration_last_visit} present additional exploration metrics. Specifically, we display $(t_{{\rm avg}})_{ij}$ and $(t_{{\rm visit}})_{ij}$, respectively, after $T=3000$ episodes, given a DeepSea problem size of $30$. The initial plot illustrates how {\sc DBMF-BPI} tends to concentrate on the grid's diagonal. However, the bottom plot shows that, in spite of this diagonal focus, {\sc DBMF-BPI} also maintains a consistent exploration of other cells within the grid. We also observe how {\sc BSP} seems to uniformly explore all cells, while {\sc IDS} does not manage to explore the entire grid within the number of episodes.
\begin{figure}[h]
	\centering
	\includegraphics[width=\linewidth]{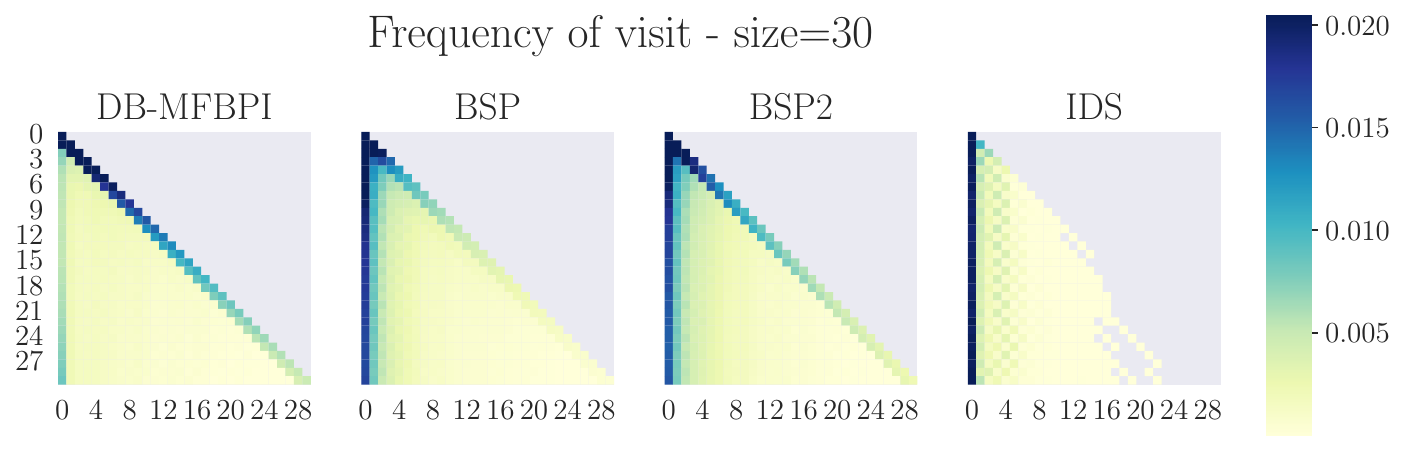}
	\caption{Slipping DeepSea problem. In this figure we depict the average frequency of visits, after $3000$ episodes, when the size of the problem is $k=30$. }
	\label{fig:deepsea_exploration_freq_visit}
	\vspace{.5cm}
	\includegraphics[width=\linewidth]{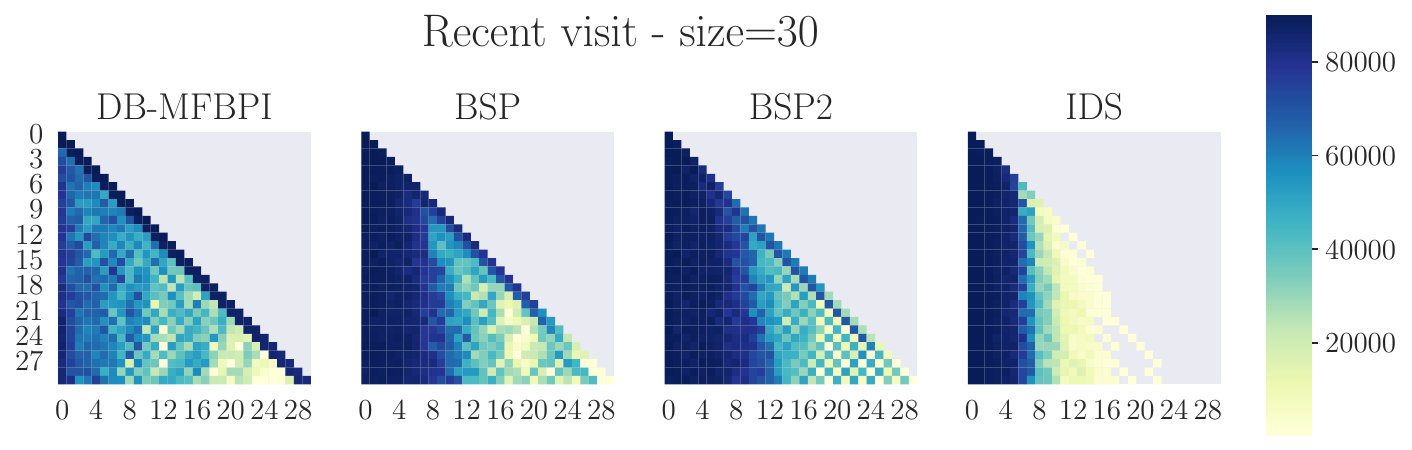}
	\caption{Slipping DeepSea problem. In this figure we depict the last timestep a cell was visited, after $3000$ episodes,   when the size of the problem is $k=30$. }
	\label{fig:deepsea_exploration_last_visit}
\end{figure}
Last, but not least, on the right column in \cref{fig:deepsea_greedy_results}, are shown the results for the learnt greedy policy $\pi_t^\star$ at time $t$. Clearly, {\sc DBMF-BPI} is able to learn an efficient policy more quickly than the other methods for different problem sizes.
\clearpage

\begin{figure}[h]
	\centering
	\includegraphics[width=\linewidth]{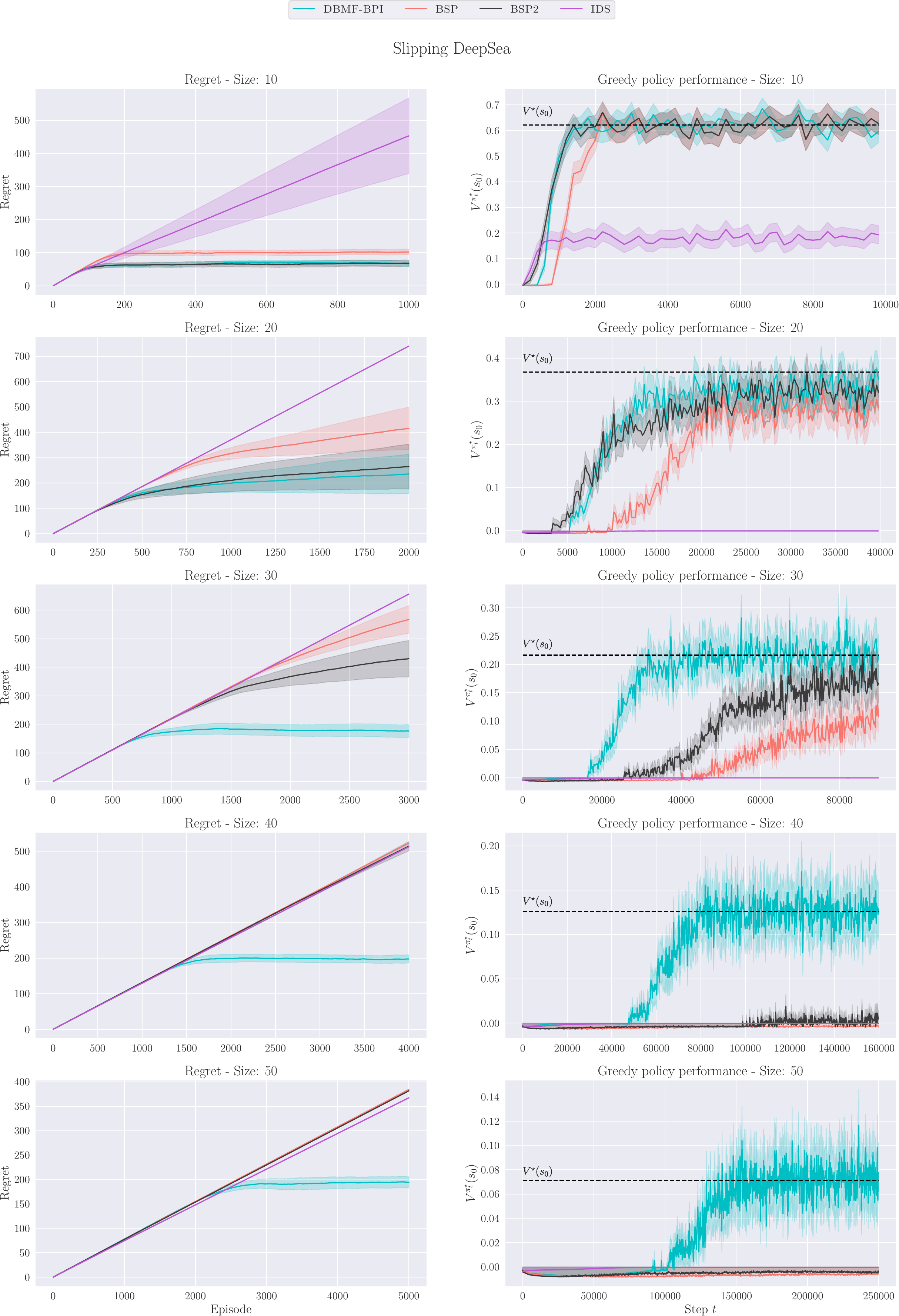}
	\caption{Slipping DeepSea problem - evaluation of the greedy policy. On the left we depict the regret of the learning agent over the number of episodes $T$ for each problem size $k$. On the right, we display the average value of the learnt greedy policy $\pi_t^\star$ at time $t$ (black dashed-line indicates the average optimal value). Results are averaged over $24$ runs, and the shaded area depicts $95\%$ confidence intervals. }
	\label{fig:deepsea_greedy_results}
\end{figure}
\clearpage

\subsection{Cartpole Swingup - Description and Additional Results}\label{sec:cartpole_appendix}
\paragraph{Description.} In this subsection we present additional results for the Cartpole swingup problem. The cartpole swingup problem is a classic problem in control theory and reinforcement learning \cite{barto1983neuronlike}. The task is to balance a pole that is attached by an un-actuated joint to a cart, which moves along a frictionless track. The system is controlled by applying a force to the cart. Initially, the pole is hanging down and the goal is to swing it up so it stays upright. In contrast to the classic cartpole balance problem, the pole needs not only to be balanced when it's upright but also to be swung up to the upright position.

The state of the system at any point in time is described by four variables: the position of the cart $x$, the velocity of the cart $\dot x$, the angle of the pole $\theta$, and the angular velocity of the pole $\dot \theta$. There are $4$ additional variables in the state, and for simplicity we refer the user to \cite{osband2020bsuite}.

To make the problem more difficult, as in \cite{osband2020bsuite} we introduce a parameter $k \in \{1,\dots,19\}$ (to not be confused with the parameter of $M_{sa}^k[V^\star]$) that parameterizes the reward function. Specifically, the agent observes a positive reward of $1$ only if the pole's angle satisfies $\cos(\theta) > k/20$, and the cart's position satisfies $|x|\leq 1-k/20$. There is also a negative reward of $-0.1$ that the agent incurs for moving, which aggravates the explore-exploit tradeoff (algorithms like {\sc DQN} \cite{mnih2015human} simply remain still).

\paragraph{Additional results.}
In \cref{fig:cartpole1,fig:cartpole2,fig:cartpole3}, we provide supplementary results for this problem. \cref{fig:cartpole1} illustrates the total upright time achieved by each learner after $200$ episodes, across various difficulty levels, $k$. Here, the total upright time refers to the total count of steps where the pole maintained an angle satisfying $\cos(\theta) > k/20$, concurrently with the cart maintaining a position that satisfied $|x|\leq 1-k/20$.

Subsequently, \cref{fig:cartpole2} showcases the evolution of this metric throughout all $200$ episodes.

 \cref{fig:cartpole3} demonstrates the performance of the learnt greedy policy $\pi_t^\star$ over the course of the training. Every $10$ episodes, we evaluated the greedy policy over $20$ episodes and computed the cumulative reward.
\begin{figure}[h]
	\centering
	\includegraphics[width=.8\linewidth]{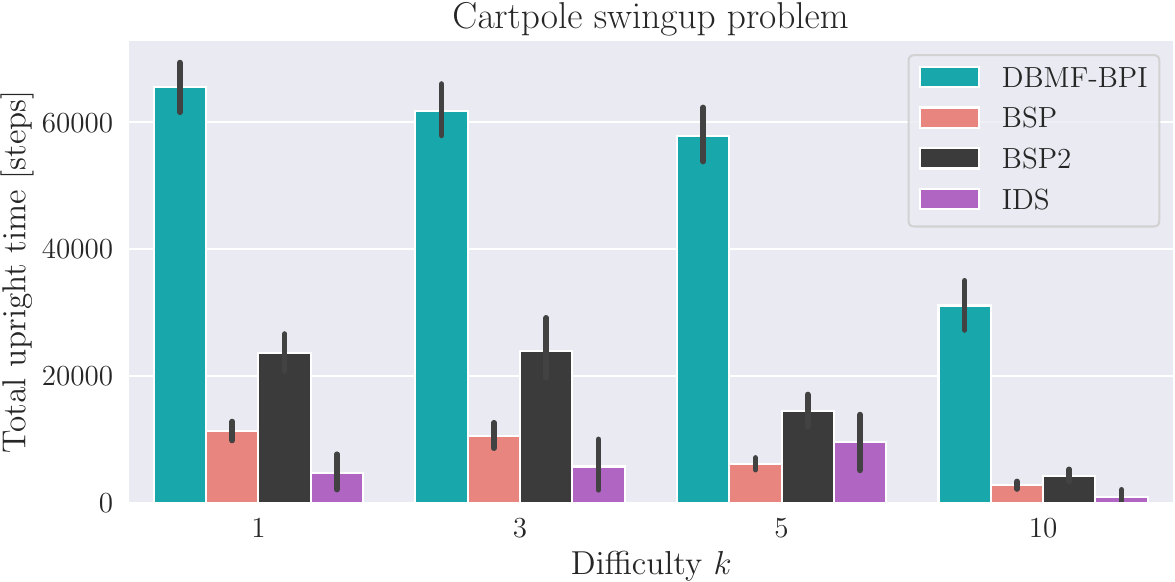}
	\caption{Cartpole swingup problem. Total upright time after $200$ episodes for different difficulties $k$. To observe a positive reward, the pole's angle must satisfy $\cos(\theta) > k/20$, and the cart's position should satisfy $|x|\leq 1-k/20$. Bars indicate $95\%$ confidence intervals. }
	\label{fig:cartpole1}
\end{figure}

\clearpage
\begin{figure}[h]
	\centering
	\includegraphics[width=.65\linewidth]{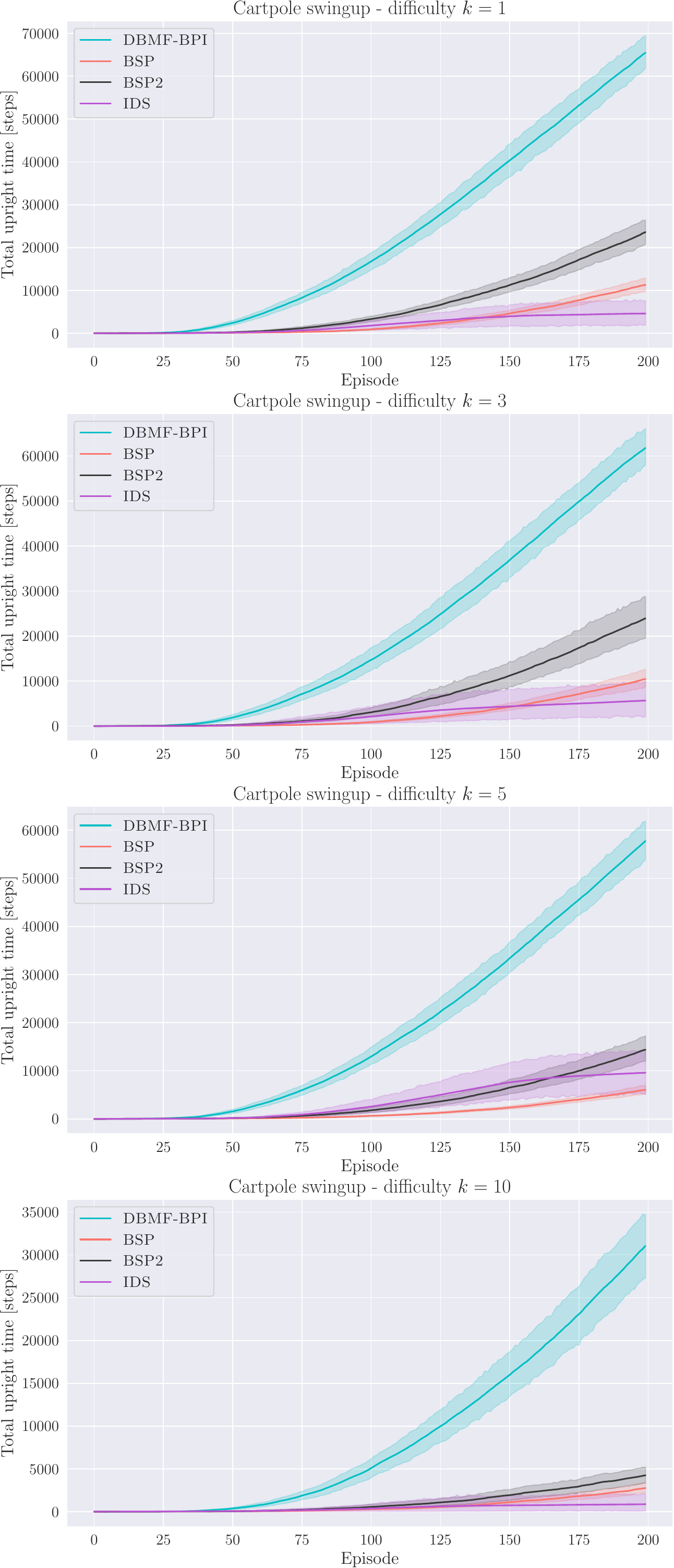}
	\caption{Cartpole swingup problem. Total upright time over $200$ episodes for different difficulties $k$. To observe a positive reward, the pole's angle must satisfy $\cos(\theta) > k/20$, and the cart's position should satisfy $|x|\leq 1-k/20$. Bars indicate $95\%$ confidence intervals. }
	\label{fig:cartpole2}
\end{figure}
\clearpage
\begin{figure}[h]
	\centering
	\includegraphics[width=.65\linewidth]{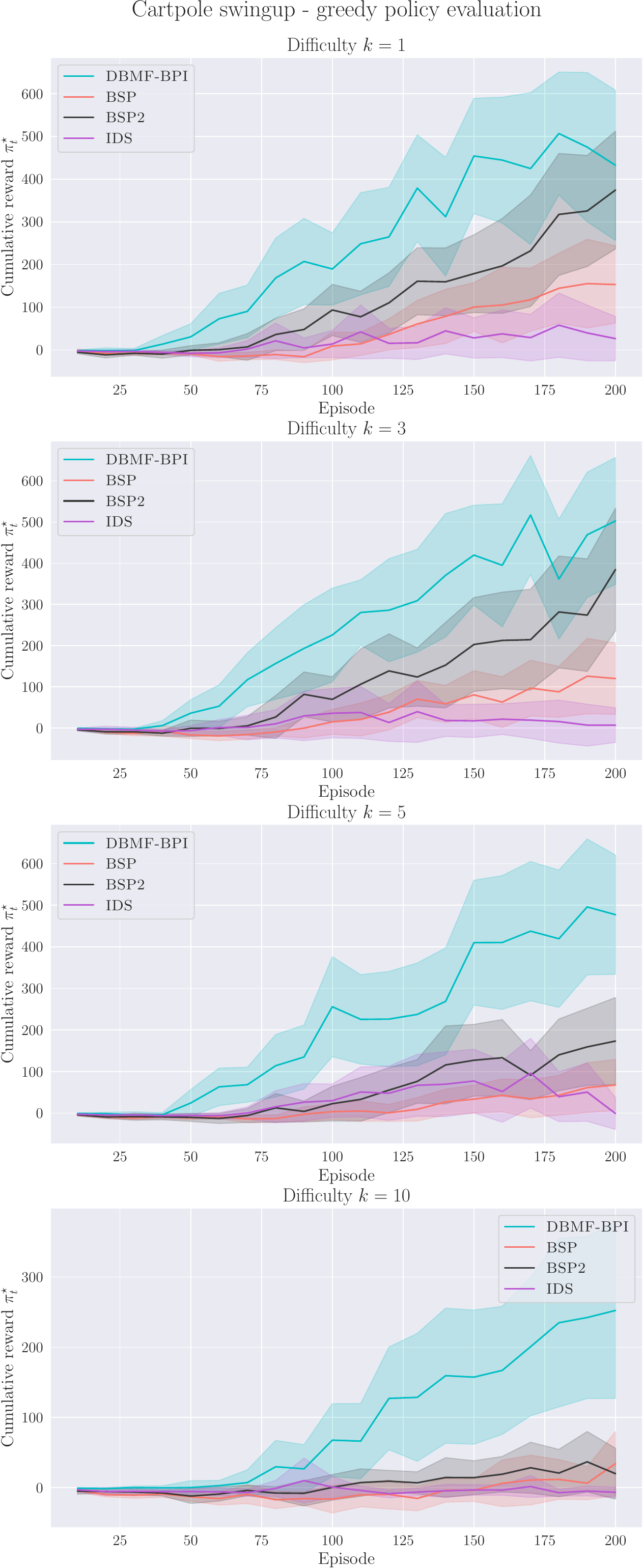}
	\caption{Cartpole swingup problem. Performance of the learnt greedy policy $\pi_t^\star$ over the training episodes (average cumulative reward collected by the greedy policy).}
	\label{fig:cartpole3}
\end{figure}
\clearpage

\paragraph{Exploration results.}
In \cref{fig:cartpole4,fig:cartpole5,fig:cartpole6}, we show additional results that illustrate the exploration of the various algorithms for difficulties $k=3,5$.

In \cref{fig:cartpole4}, we display two metrics at each training step $t$: the entropy of visit frequency and the entropy of the most recent visit.
The first metric quantifies how thoroughly the method has explored the state space $(x,\dot x, \theta, \dot \theta) $ up to time $t$. To do this, we discretize the state space into bins and tally the occurrences in each bin. We then normalize these counts by their sum and calculate the resulting entropy, which is normalized to the range $[0,1]$.

While this measure of visit frequency provides some insight, it is insufficient for understanding whether the algorithm continues to explore new states or revisits old ones. To address this, the second metric measures the dispersion of the timing of the last visits to various regions of the state space. A larger dispersion indicates that the algorithm is concentrating on a specific region, resulting in a smaller entropy (and vice-versa). To calculate this, we again use normalized entropy.


Finally, in \cref{fig:cartpole5} and \cref{fig:cartpole6}, we illustrate the visitation frequency after \(20K\) training steps for \( (x, \dot{x}) \) and \( (\dot{x}, \dot{\theta}) \) at difficulty levels \( k=3 \) and \( k=5 \). Darker regions signify higher visitation frequencies. The pattern in \( (\dot{x}, \dot{\theta}) \) is characteristic of algorithms that have learned to stabilize the policy. Notably, \textsc{DBMF-BPI} is also actively exploring various velocities. A similar trend is observed for \( (x, \dot{x}) \): while most methods focus on an \emph{s}-shaped trajectory, \textsc{DBMF-BPI} also explores other regions of the state space.

\begin{figure}[h]
	\centering
	\includegraphics[width=\linewidth]{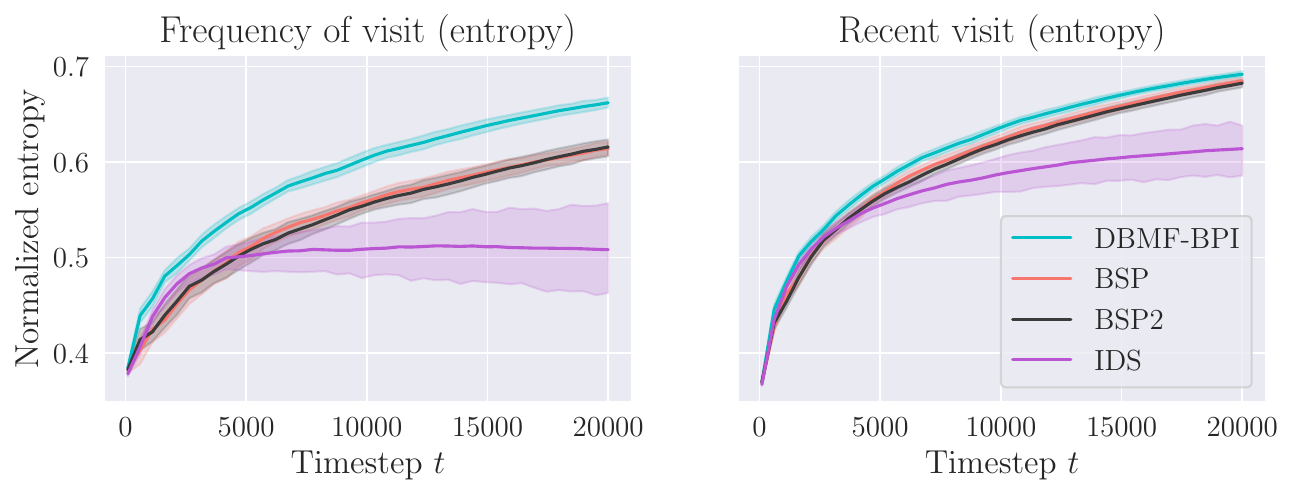}
 \includegraphics[width=\linewidth]{figures/cartpole/exploration_entropy_5.pdf}
	\caption{Exploration in Cartpole swingup: At the top, we present results for difficulty \(k=3\), and at the bottom, for \(k=5\). In the left column, we depict the entropy of visitation frequency for the state space \( (x, \dot{x}, \theta, \dot{\theta}) \) during training. In the right column, we display a measure of the dispersion of the most recent visits; smaller values indicate that the agent is less explorative as \( t \) increases.}
	\label{fig:cartpole4}
\end{figure}

\begin{figure}[h]
	\centering
	\includegraphics[width=.8\linewidth]{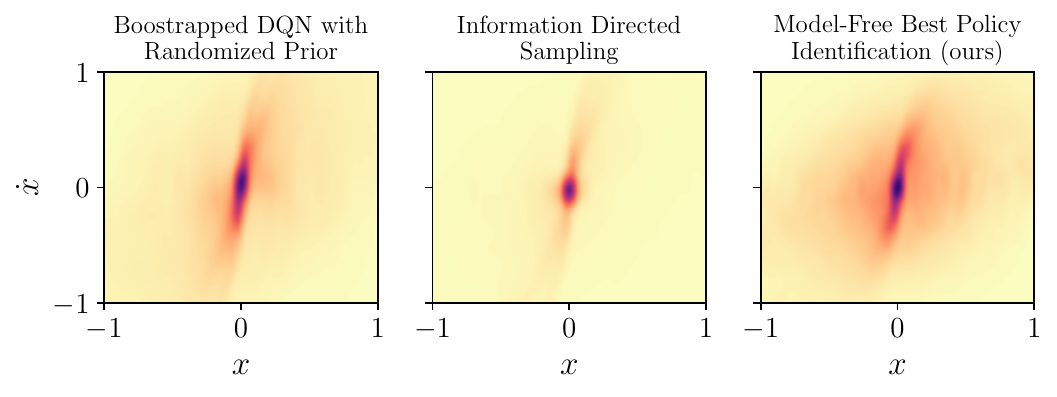}
 \includegraphics[width=.8\linewidth]{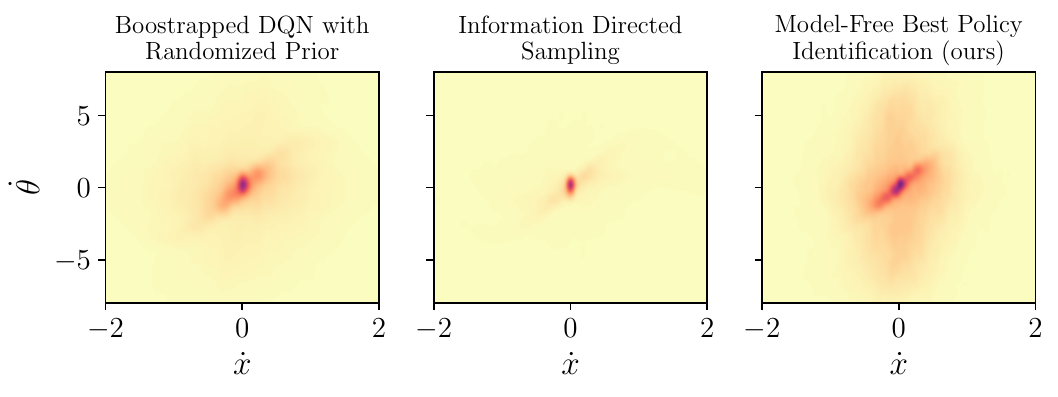}
	\caption{Cartpole Swingup \cite{osband2020bsuite} after $20$K training steps for difficulty $k=3$, comparing  {\sc BSP} (Bootstrapped DQN with randomized priors) \cite{osband2018randomized},  {\sc IDS} (Information-Directed Sampling) \cite{nikolov2018information}, and  {\sc MF-BPI} (Model-Free Best Policy Identification). Darker areas indicate higher visitation frequency. At the top we show this frequency for $(x,\dot x)$, the cart's position and linear's velocity, and at the bottom of $(\dot x,\dot\theta)$, the cart's linear and pole's angular velocities. }
	\label{fig:cartpole5}
\end{figure}

\begin{figure}[h]
	\centering
	\includegraphics[width=.8\linewidth]{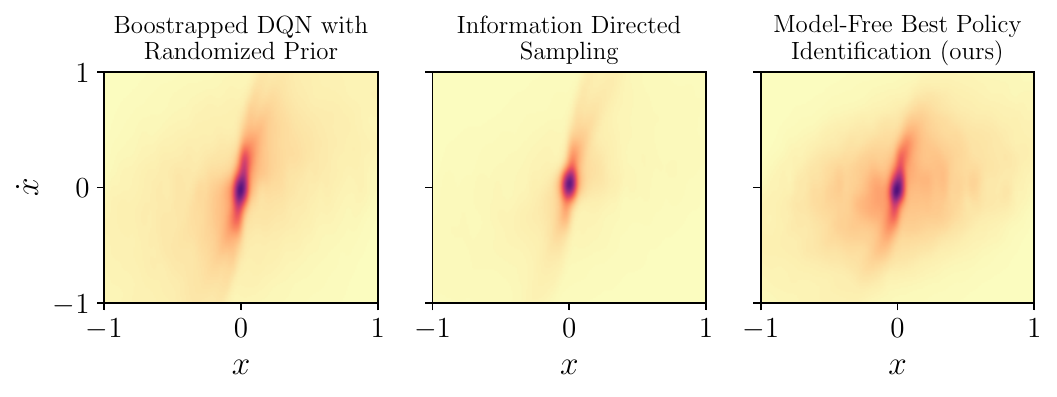}
 \includegraphics[width=.8\linewidth]{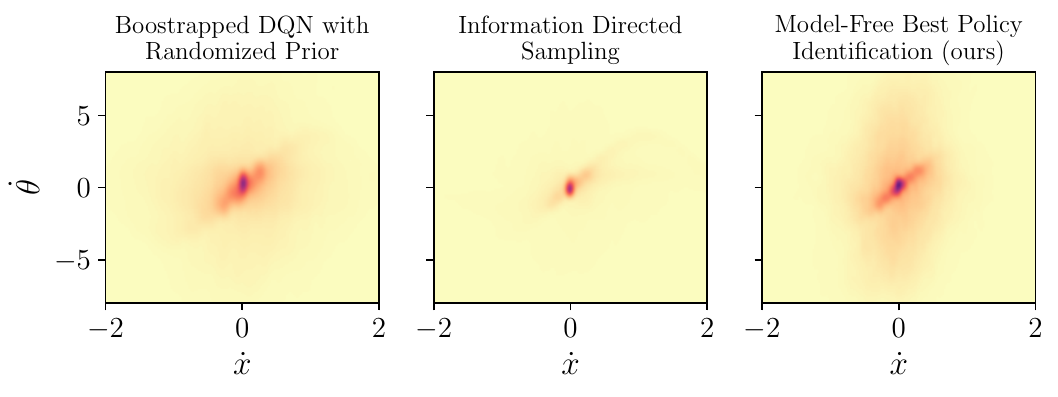}
	\caption{Cartpole Swingup \cite{osband2020bsuite} after $20$K training steps for difficulty $k=5$, comparing  {\sc BSP} (Bootstrapped DQN with randomized priors) \cite{osband2018randomized},  {\sc IDS} (Information-Directed Sampling) \cite{nikolov2018information}, and  {\sc MF-BPI} (Model-Free Best Policy Identification). Darker areas indicate higher visitation frequency. At the top we show this frequency for $(x,\dot x)$, the cart's position and linear's velocity, and at the bottom of $(\dot x,\dot\theta)$, the cart's linear and pole's angular velocities. }
	\label{fig:cartpole6}
\end{figure}

\clearpage\newpage
\subsection{Parameters, Hardware, Code and Libraries}
In this section, we outline the parameters used for the simulations, describe the hardware employed to run the simulations, and list the libraries that we used.

\subsubsection{Simulation parameters - Riverswim and Forked Riverswim}
In both the Riverswim and Forked Riverswim environments we used a discount factor of $\gamma=0.99$. Depending on the size of the state space, the horizon length was different. We used $T=10000 \times |S|$ for the Riverswim environment, and $T=20000\times N$ for the Forked Riverswim environment (where $N$ is the length of the main river; see also the description of the environment in \cref{subsec:forked_riverswim}).

We run simulations for $10$ different seeds, and evaluated the estimated greedy policy $\pi_t^\star$ every $200$ steps. All agents were optimistically initialized (\emph{i.e.}, the $Q$-values were initialized to $1/(1-\gamma)$, etc...), and model-based approaches used additive smoothing (with factor $1$).

For the {\sc MDP-NaS} and {\sc PS-MDP-NaS} (see next section for a description) we computed the allocation every $T_0=\min(T_{max},\max(200, \frac{T_{max}t}{T/2} ))$ steps, where $T_{max}=\frac{2000T}{50000}$. For {\sc PSRL} we computed a new greedy policy every $\lceil 1/(1-\gamma) \rceil$ steps.

We used a learning rate of $\alpha_t = \frac{H+1}{H+k_t}$ to learn the $Q$-values, where $H=(1-\gamma)^{-1}$ and $k_t=N_t(s_t,a_t)$ is the number of visits to $(s_t,a_t)$ at time $t$. Similarly, to learn the $M$-values we used a learning rate of $\beta_t = \alpha_t^{1.1}$ (which was not optimized).

For bootstrapped {\sc MFBPI} we used a parameter $ k=1$, and an ensemble size $B=50$ with training probability $p=0.7$. Similarly, these values were not optimized.

All methods that employed an $\epsilon$-soft policy, obtained a final policy $\omega$ by mixing
mixed the original policy $\pi$ with a uniform policy as follows $\omega(a|s)= (1-\epsilon_t) \pi(a|s) + \epsilon_t/ |A|$. The value of $\epsilon_t$ is $\epsilon_t= \min(1, 1/N_t(s_t))$  where $N_t(s_t)$ is the total number of visits to state $s_t$ at time $t$.

\subsubsection{Simulation parameters - Slipping DeepSea}

For the DeepSea problem we used a discount factor of $\gamma=0.99$, and different problem sizes $N\in\{10,20,30,40,50\}$. The number of training episodes was $T=100 N$. Every $200$ steps we evaluated the performance of the estimated greedy policy $\pi_t^\star$ over $20$ episodes. For all simulations we used a slipping probability of $0.05$. The number of features in the state is $N^2$, and the number of actions is $2$.

Refer to \cref{tab:parameters_slipping_deepsea} for the parameters of the agents.

\begin{table}[h]
	\centering
	\caption{Parameters of the agents for the slipping DeepSea problem.}
	\small
	\begin{tabular}[t]{c|cccc}
		\toprule
		{\bf Property} & {\sc DBMF-BPI} &{\sc BSP} & {\sc BSP 2} & {\sc IDS}  \\ \midrule
		Ensemble size $Q$	& $20$ & $20$ & $20$  &  $20+\frac{(N-10)}{2}$\\ \hline
		Ensemble size $M$   & $20$ & N.A. &N.A. & N.A.      \\ \hline
		Hidden layers sizes   & $[32]$ & $[32]$ &$[32]$ & $[50]$      \\ \hline
		Num. of quantiles & N.A. & N.A. & N.A. & $50$ \\ \hline
		Prior scale $Q$-values\\(depends on $N$)	& $\{3,5,10,15,20\}$		& $\{3,5,10,15,20\}$& $\{3,5,10,15,20\}$&  N.A.   \\ \hline
		Prior scale $M$-values\\(depends on $N$)		& $\{3,5,10,15,20\}$		& $\{3,5,10,15,20\}$& $\{3,5,10,15,20\}$&  N.A.   \\ \hline
		Replay buffer size			&$10^5$ 	&$10^5$ &$10^5$ &       $10^5$            \\  \hline
		Training period			& $1$	& $1$&$1$ &$1$                   \\  \hline
	    Target network update period		& $4$	&$4$ &$4$ &$4$                   \\  \hline
		Batch size			& $128$	&$128$ &$128$ &$128$                   \\  \hline
		Mask probability $p$			& $0.7$	&$0.5$ & $0.7$& N.A.                   \\  \hline
		Learning rate $Q$-values			& $5\times 10^{-4}$	&$10^{-3}$ & $10^{-3}$& $5\times 10^{-4}$                  \\  \hline
		Learning rate $M$-values			& $5\times 10^{-4}$	& N.A.&N.A. &    N.A.               \\  \hline
		Learning rate quantile network & N.A.	&N.A. & N.A.&      $10^{-6}$             \\  \hline
		$ k$			& $2$	&N.A. & N.A.& N.A.                  \\
		\bottomrule
	\end{tabular}

	\label{tab:parameters_slipping_deepsea}
\end{table}

\subsubsection{Simulation parameters - Cartpole Swingup}
For the Cartpole swingup problem we used a discount factor of $\gamma=0.99$, and different difficulties  $k\in\{1,3,5,10\}$.
 The number of training episodes was $T=200$, and we run simulations for $20$ different seeds. Every $10$ steps in the training we evaluated the performance of the estimated greedy policy $\pi_t^\star$ over $20$ episodes. The state is a vector in $\mathbb{R}^8$ and the number of actions is $3$.

For every method, with the exception of {\sc IDS}, we set up the parameters in the $i^{th}$ layer of each network by sampling from a truncated Gaussian distribution with a $0$ mean and a standard deviation of $1/\sqrt{f_{in}}$, where $f_{in}$ represents the number of inputs to the $i^{th}$ layer. Values were cut off at twice the standard deviation. For {\sc IDS}, enhancing the standard deviation improved results. Specifically, we employed a standard deviation of $1.5/\sqrt{f_{in}}$ for the $Q$-networks ensemble, and a standard deviation of $2/\sqrt{f_{in}}$ for the quantile network. Generally, this initialization mirrors an optimistic initialization. However, the results can vary significantly between runs, and our observation was that the {\sc IDS} method often exhibited greater variance compared to the other methods incorporated in our study. To conclude, the bias for all layers was set to $0$.

Refer to \cref{tab:parameters_cartpole} for the parameters of the agents.

\begin{table}[h]
	\centering
	\caption{Parameters of the agents for the Cartpole swingup problem.}
	\small
	\begin{tabular}[t]{c|cccc}
		\toprule
		{\bf Property} & {\sc DBMF-BPI} &{\sc BSP} & {\sc BSP 2} & {\sc IDS}  \\ \midrule
		Ensemble size $Q$	& $20$ & $20$ & $20$  &  $20$\\ \hline
		Ensemble size $M$   & $20$ & N.A. &N.A. & N.A.      \\ \hline
		Hidden layers sizes   & $[50]$ & $[50]$ &$[50]$ & $[50]$      \\ \hline
		Num. of quantiles & N.A. & N.A. & N.A. & $50$ \\ \hline
		Prior scale $Q$-values\\(depends on $N$)	& $3$		& $3$& $3$&  N.A.   \\ \hline
		Prior scale $M$-values\\(depends on $N$)		& $3$		& $3$& $3$&  N.A.   \\ \hline
		Replay buffer size			&$10^5$ 	&$10^5$ &$10^5$ &       $10^5$            \\  \hline
		Training period			& $1$	& $1$&$1$ &$1$                   \\  \hline
		Target network update period		& $4$	&$4$ &$4$ &$4$                   \\  \hline
		Batch size			& $128$	&$128$ &$128$ &$128$                   \\  \hline
		Mask probability $p$			& $0.7$	&$0.5$ & $0.7$& N.A.                   \\  \hline
		Learning rate $Q$-values			& $5\times 10^{-4}$	&$5\times 10^{-4}$ & $5\times 10^{-4}$& $5\times 10^{-4}$                  \\  \hline
		Learning rate $M$-values			& $5\times 10^{-4}$	& N.A.&N.A. &    N.A.               \\  \hline
		Learning rate quantile network & N.A.	&N.A. & N.A.&      $10^{-6}$             \\  \hline
		$ k$			& $2$	&N.A. & N.A.& N.A.                  \\
		\bottomrule
	\end{tabular}

	\label{tab:parameters_cartpole}
\end{table}
\subsubsection{Hardware and simulation time} To run the simulations, we used a local stationary computer with Ubuntu 20.10, an Intel® Xeon® Silver 4110 Processor (8 cores) and 48GB of ram. On average, it takes approximately $14$ days to complete all the simulations contained in this manuscript. Ubuntu is an open-source Operating System using the Linux kernel and based on Debian. For more information, please check \url{https://ubuntu.com/}.

\subsubsection{Code and libraries} We set up our experiments using Python 3.10 \cite{van1995python} (For more information, please refer to the following link \url{http://www.python.org}), and made use of the following libraries: Cython \cite{behnel2011cython}, NumPy \cite{harris2020array}, SciPy \cite{2020SciPy-NMeth}, PyTorch  \cite{NEURIPS2019_9015}, CVXPY \cite{diamond2016cvxpy}, MOSEK \cite{mosek}, Seaborn \cite{michael_waskom_2017_883859}, Pandas \cite{mckinney2010data}, Matplotlib \cite{hunter2007matplotlib}.

In the code, we make use of some code from the Behavior suite \cite{osband2020bsuite}, which is licensed with the APACHE 2.0 license. Changes, and new code, are published under the MIT license. To run the code, please, read the attached README file for instructions.\newpage

 \newpage
 
\section{Algorithms}\label{appB}
In the following section we describe the algorithms that we discuss in this manuscript. For simplicity, we provide a brief summary of them in form of table I.

\begin{table}[h]
	\centering
	\caption{Description of the various algorithms}
	\setlength{\leftmargini}{0.3cm}
	\small
	\begin{tabular}[t]{@{}m{.18\textwidth}m{.45\textwidth}m{.3\textwidth}@{}}
		\toprule
		Name & Description & Key points  \\ \midrule
		\textsc{PS-MDP-NaS}	& An adaptation of \textsc{MDP-NaS} that uses posterior sampling to sample an MDP $\phi_t$, which is then used to compute the optimal allocation (in \cref{eq:new_upper_bound}).            &  Requires the user to:
		\begin{itemize}
			\item Keep an estimate of the MDP.
			\item Perform value/policy iteration.
			\item Compute the allocation (a convex problem).
			\item Uses posterior sampling at each time step to sample an MDP and compute the allocation.
		\end{itemize}                \\\midrule
		\textsc{O-BPI}		&  An adaptation of \textsc{MDP-NaS} that learns the $Q$-values and $M$-values. These values are used to compute the optimal allocation in \cref{eq:new_upper_bound}.           			&\begin{itemize}
			\item Does not perform value iteration.
			\item Requires to keep an estimate of the transition function.
			\item Compute the allocation (a convex problem).
			\item Uses forced exploration to sample all state-action pairs i.o.
		\end{itemize}                           \\ \midrule
		Bootstrapped \textsc{MF-BPI}		& This algorithm is an extension of \textsc{O-BPI} that computes the allocation using the closed form solution in \cref{corollary:upper_bound_new_bound}. The $Q,M$-values used to compute the allocation are bootstrap samples.           			&  \begin{itemize}
			\item Does not perform value iteration and does not require to keep an estimate of the model.
			\item Closed form solution for the allocation.
			\item Uses bootstrapping (forced exploration not necessary).
		\end{itemize}                            \\ \midrule
		\textsc{DBMF-BPI}		& An extension of \textsc{BO-MFPI} to the Deep-RL setting. The baseline architecture is inspired from BootstrappedDQN with prior networks. This architecture is then adapted to compute a generative allocation.            			&
		\begin{itemize}
			\item  Like boostrapped \textsc{MF-BPI}.
			\item Requires to keep an ensemble of $Q,M$-networks.
			\item Can be applied to continuous state spaces.
		\end{itemize}                           \\ \bottomrule
	\end{tabular}

	\label{tab:algorithms_summary}
\end{table}

\subsection{{\sc PS-MDP-NaS} - Posterior Sampling for navigating MDPs}
In this sub-section we present {\sc PS-MDP-NaS}, an adaptation of {\sc MDP-NaS} that uses posterior sampling. An outline of the algorithm is given in \cref{algo:psmdpnas}. For simplicity, we omit the use of any stopping rule, since we focus more on the practical implementation of the algorithm.

At each timestep we sample an MDP $\phi_t$ from a posterior distribution, and use it to solve the optimal allocation in \cref{thm:upper_bound_T_new} with navigation constraints. When computing the optimal allocation $\arginf_{\omega \in \Omega(\phi)} U(\omega)$, we limit the maximum number of possible values of $ k$ for simplicity.

The algorithm considers a Dirichlet prior for the transition function, and a Gamma prior for the reward distribution. Specifically, for each $(s,a)$ we have a prior hyper-parameter $\rho_{sa} \in\mathbb{R}^{|S|}$ that characterizes the transition function, and two other hyper-parameters $\alpha_{sa},\beta_{sa}\in\mathbb{R}$ that characterize the reward distribution for each $(s,a)$.

After observing an experience at time $t$, the posterior parameters $\rho_{sa}(t),\alpha_{sa}(t),\beta_{sa}(t)$ at time $t$ are computed as follows
\begin{align*}
	\rho_{sa}(t) &\gets \rho_{sa} + N_t(s,a,s'),\\
	\alpha_{sa}(t) &\gets \alpha_{sa} + R_t(s,a),\\
	\beta_{sa}(t) &\gets \beta_{sa}+ N_t(s,a) - R_t(s,a).
\end{align*}
where $N_t(s,a,s')$ is the number of times the agent experienced state $s'$ after choosing action $a$ in state $s$ up to time $t$, $R_t(s,a)=\sum_{n=0}^t r_n \mathbf{1}\{s_n=s \land a_n=a\}$ is the total cumulative reward observed up to time $t$ after choosing action $a$ in state $s$, and, lastly, $N_t(s,a)=  \sum_{s'}N_t(s,a,s')$ is the total number of times the agent chose action $a$ in state $s$.

\begin{algorithm}[t]
	\caption{\textsc{PS-MDP-NaS} - Posterior Sampling for navigating MDPs}
	\label{algo:psmdpnas}
	\begin{algorithmic}[1]
		\REQUIRE Parameters $(\rho,\alpha,\beta)$.
		\STATE Initialize  counter $N_0(s,a,s')\gets 0$ for all $(s,a,s')\in S\times A\times S$.
		\STATE Observe $s_0\sim p_0$.
		\FOR{$t=0,1,2,\dots,$}
		\STATE {\bf Computing the allocation.}
		\begin{ALC@g}
			\STATE Sample a transition function $P_t(\cdot|s,a) \sim {\rm Dir}(\rho_{sa}(t))$ and reward distribution $q_t(\cdot|s,a) \sim {\rm Ber}(\alpha_{sa}(t)/(\alpha_{sa}(t)+\beta_{sa}(t)))$.
			\STATE Perform policy iteration using $\phi_t=(P_t,q_t)$ and compute $\pi_t^\star$, the greedy policy at time $t$. Use $\pi_t^\star$ to derive the various quantities needed to compute the allocation in Thm. \ref{thm:upper_bound_T_new}
			\STATE Compute allocation $\omega^{(t)}$ by solving the optimization problem in Thm. \ref{thm:upper_bound_T_new} using $(P_t,q_t)$.
		\end{ALC@g}
		\STATE {\bf Sampling step.}
		\begin{ALC@g}
			\STATE Sample $a_t\sim \omega^{(t)}(s_t,\cdot)$ and observe $(r_t,s_{t+1})\sim q(\cdot|s_t,a_t)\otimes P(\cdot|s_t,a_t)$.
		\end{ALC@g}

		\STATE {\bf Posterior update.}
		\begin{ALC@g}
			\STATE  Update number of visits $N_{t+1}(s_t,a_t,s_{t+1}) \gets N_{t}(s_t,a_t,s_{t+1}) + 1$ and total cumulative reward $R_{t+1}(s_t,a_t) \gets R_t(s_t,a_t) + r_t$.
			\STATE Update posterior parameters
			\begin{align*}
				\rho_{sa}(t+1) &\gets \rho_{sa} + N_{t+1}(s,a,s'),\\
				\alpha_{sa}(t+1) &\gets \alpha_{sa} + R_{t+1}(s,a),\\
				\beta_{sa}(t+1) &\gets \beta_{sa}+ N_{t+1}(s,a) - R_{t+1}(s,a).
			\end{align*}
		\end{ALC@g}
		\ENDFOR
	\end{algorithmic}
\end{algorithm}

\subsection{{\sc O-BPI} - Online Best Policy Identification}
In this part, we introduce {\sc O-BPI}, or Online Best Policy Identification. This procedure bears resemblance to  {\sc MDP-NaS}, but sidesteps the need for policy iteration at every timestep. Instead, we employ stochastic approximation to learn the $(Q,M)$-values and use these calculated values to compute the allocation. We describe a variation where the user exclusively learns the $M$-function for a given $k$. It's important to note, however, that the agent has the capability to learn multiple $M$-functions, for varying $k$ values, to better approximate the true solution.

We present a version of the algorithm that uses forced exploration, where we mix the allocation that we obtain from \cref{thm:upper_bound_T_new} with a uniform distribution.  It's straightforward to derive an extension using bootstrapping, as we show in subsequent subsections.

\paragraph{O-BPI.} To compute the allocation $\omega$ we require to estimate the transition function (\emph{e.g.}, using maximum likelihood), and  we denote its estimate at time $t$ by $\hat P_t$. To derive $\omega$, as for {\sc MF-BPI}, we compute $\pi_t^\star, \Delta_t$ and $\deltaminestimatet{t}$ using the estimate $Q$-function $Q_t$. Then, we solve  the following convex problem
\begin{equation}\label{eq:upper_bound_time_t}
	\arginf_{\omega \in {\cal C}_t}\max_{s,a\neq \pi_{t}^\star(s)} \frac{2+8\varphi^2  M_{ t}(s,a)^{2^{1- k}} }{\omega(s,a)(\Delta_{t}(s,a)+\lambda)^2} +\max_{s'} \frac{C_t(s')(1+\gamma)^2}{\omega(s',\pi^\star(s')) (\Delta_{t}(s,a)+\lambda)^2(1-\gamma)^2},
\end{equation}
where   $C_t(s')=\max\left(4,16\gamma^2\varphi^2 M_{ t}(s',\pi_{t}^\star(s'))^{2^{1- k}}\right)$. The constraint set ${\cal C}_t$ is simply $\Delta(S\times A)$ in the generative case, and ${\cal C}_t=\{\omega:  \omega_{s'} = \sum_{s,a} \omega(s,a) \hat P_t(s'|s,a), \forall s'\}$ in the case with navigation constraints.
In particular, for finite state-action MDPs we use  $N_t(s,a,s')$ (the number of visits up to time $t$ of $(s,a,s')$) to estimate $P$.
As in \textsc{MDP-NaS} \cite{marjani2021navigating}, to ensure that the various estimates asymptotically converge to the true quantities, we force  exploration using a D-tracking like procedure, that is, with probability $\epsilon_t \propto 1/N_t(s_t)^\lambda, \lambda\in(0,1]$,  we choose an action uniformly at random in state $s_t$ at time $t$ (since this type of forced exploration is slightly different from the one proposed in \cite{marjani2021navigating}. We have the following guarantee.

\begin{lemma}[Forced exploration]\label{lemma:forced_exploration}
	Let $\epsilon_t(s) \coloneqq 1/N_t(s)^\alpha$ with  $\alpha\in (0,1]$. Then, \textsc{O-BPI} satisfies $\mathbb{P}_\phi(\forall (s,a)\in S\times A, \lim_{t\to\infty} N_t(s,a)=\infty)=1$.
\end{lemma}
\begin{proof}
	The lemma follows from Observation 1 in \cite{singh2000convergence}. We use the fact that in communicating MDPs every state gets visited infinitely often as long as each action is chosen infinitely often in each state. Denote by $\mathbb{P}(a_t=a|s_t=s, N_t(s)=i)$ the probability that action $a$ is executed at the $i^{th}$ visit to state $s$. The forced exploration step in \textsc{O-BPI} ensures that
	$\mathbb{P}(a_t=a|s_t=s, N_t(s)=i) \geq \epsilon_t(s)/|A| = 1 / (i^\alpha |A|).$ Consequently, for all $(s,a)$ and $0< \alpha \leq 1$ we have that
	\[
	\sum_{i=1}^\infty \mathbb{P}(a_t=a|s_t=s, N_t(s)=i) \geq \frac{1}{|A|}\sum_{i=1}^\infty \frac{1}{i^\alpha} = \infty.
	\]
	By the Borel-Cantelli lemma it follows that asymptotically each action is chosen infinitely often in each state, which yields the desired result.
\end{proof}

\subsection{Boostrapped {\sc MF-BPI} - Model Free Best Policy Identification}
In this section, we describe in more detail bootstrapped {\sc MF-BPI}. {\sc MF-BPI} is a model-free algorithm that adapts exploration based on a sample-complexity bound, while using bootstrapping to characterize the epistemic uncertainty. The algorithm also relies on   a closed-form solution for computing the allocation $\omega$, which eliminates the need for solving an optimization problem.

Recall the closed form solution for the allocation $\omega$ from  Corollary \ref{corollary:upper_bound_new_bound}:
\begin{equation}
	\omega(s,a) \propto \begin{cases}
		H(s,a) & \text{if } a\neq \pi^\star(s)\\
		\sqrt{ H\sum_{s,a\neq\pi^\star(s)} H(s,a)/|S|} & \text{otherwise},
	\end{cases}
\end{equation}
where  $H(s,a)$ and $H$ are defined as follows:
\begin{align}
H(s,a) &= \frac{2+8\varphi^2 M_{sa}^{ k}[V^\star]^{2^{1- k}} }{(\Delta(s,a)+\lambda)^2},\\
	H &= \frac{\max_{s'} C(s') (1+\gamma)^2}{ (\deltamin+\lambda)^2(1-\gamma)^2},
\end{align}
for some fixed value $ k$ that should be treated as a hyper-parameter, parameter $\lambda\geq 0$ and $C(s')=\max\left(4,16\gamma^2\varphi^2 M_{s',\pi^\star(s')}^{ k}[V^\star]^{2^{1- k}}\right)$.
\begin{algorithm}[t]
	\caption{Bootstrapped \textsc{MF-BPI}}
	\label{algo:mfbpi_full}
	\begin{algorithmic}[1]
		\REQUIRE Parameters $(\lambda,  k,p)$; ensemble size $B$;  learning rates $\{(\alpha_{t,b},\beta_{t,b})\}_{t,b}$.
		\STATE Initialize  $Q_{1,b}(s,a)\sim {\cal U}([0,1/(1-\gamma)])$ and $M_{1,b}(s,a)\sim {\cal U}([0,1/(1-\gamma)^{2^ k}])$ for all $(s,a)\in S\times A$ and $b\in[B]$.
		\STATE Observe $s_0\sim p_0$.
		\FOR{$t=0,1,2,\dots,$}
		\STATE {\bf Compute allocation.}
		\begin{ALC@g}
			\STATE Sample $\xi \sim {\cal U}([0,1])$ and set,  $\hat Q_t(s,a)={\rm Quantile}_\xi(\{Q_{t,1}(s,a), \dots, Q_{t,B}(s,a)\})$ (sim. $\hat M_t)$ for all $(s,a)$.
			\STATE Compute generative solution $\omega_o^{(t)}$ using  \cref{corollary:upper_bound_new_bound}. Let $\pi_{t}^\star(s) = \argmax_a \hat Q_t(s,a)$, $\Delta_{t}(s,a) = \hat Q_{t}(s,\pi_{t}^\star(s)) - \hat Q_{t}(s,a)$, $\deltaminestimatet{t} = \min_{s,a\neq\pi_{t}^\star(s)} \Delta_{t}(s,a)$ and
			\begin{align*}
				&H_{t}(s,a) \coloneqq \frac{2+8\varphi^2 \hat  M_{t}(s,a)^{2^{1- k}} }{(\Delta_{t}(s,a)+\lambda)^2},\\
				&H_{t} \coloneqq \frac{\max_{s'} 4 (1+\gamma)^2\max(1,4\gamma^2\varphi^2 \hat  M_{t}(s',\pi_{t}^\star(s'))^{2^{1- k}})}{ (\deltaminestimatet{t}+\lambda)^2(1-\gamma)^2}
			\end{align*}
			\STATE Set
			\[\omega_o^{(t)}(s_t,a) =
			\begin{cases}
				H_{t}(s_t,a) & \hbox{ if }a\neq \pi_t^\star(s_t),\\
				\sqrt{H_{t} \sum_{s,a\neq \pi_t^\star(s)} H_{t}(s,a)/|S|} & \hbox{ otherwise }.
			\end{cases}
			\]
			\STATE Let $\omega^{(t)}(s_t,a) =  \frac{\omega_o^{(t)}(s_t,a)}{\sum_{a'} \omega_o^{(t)}(s_t,a')} $ be the policy at time $t$ in state $s_t$.
		\end{ALC@g}
		\STATE Sample $a_t\sim \omega^{(t)}(s_t,\cdot)$; observe $(r_t,s_{t+1})\sim q(\cdot|s_t,a_t)\otimes P(\cdot|s_t,a_t)$.
		\STATE {\bf Training step.}
		\begin{ALC@g}
			\FOR{$b=1,\dots,B$}
			\STATE With probability $p$, using the experience $(s_t,a_t,r_t,s_{t+1})$, update  the values $Q_{t,b},M_{t,b}$  using \cref{eq:stochastic_approximation_step_qvalues,eq:stochastic_approximation_step_mvalues}
			\begin{align*}
				Q_{t+1,b}(s_t,a_t) &= Q_{t,b}(s_t,a_t) + \alpha_{t,b}(s_t,a_t)\left(r_t+\gamma \max_a Q_{t,b}(s_{t+1},a)-Q_{t,b}(s_t,a_t)\right),\\
				M_{ t+1,b}(s_t,a_t) &= M_{t,b}(s_t,a_t) + \beta_{t,b}(s_t,a_t)\left(\left(\delta_{t,b}'/\gamma\right)^{2 k} - M_{t,b}(s_t,a_t)\right),
			\end{align*}
			where  $\delta_{t,b}'= r_t+\gamma \max_a Q_{t+1,b}(s_{t+1},a)-Q_{t+1,b}(s_t,a_t)$.
			\ENDFOR
			\STATE Compute greedy policy as \[\bar \pi_t^\star(s) \gets  {\rm Median}(\{ \argmax_a Q_{t+1,1}(s,a),\dots,\argmax_a Q_{t+1,B}(s,a)   \}).\]
		\end{ALC@g}
		\ENDFOR
	\end{algorithmic}
\end{algorithm}
Rather than resorting to policy iteration, our approach involves learning the $Q$-values and $M$-values through stochastic approximation. The algorithm keeps track of estimates $Q_t(s,a)$ and $M_t(s,a)$ for all states and actions up to a given time $t$. The updates for the stochastic approximation are carried out at every time step $t$ and can be represented as:
\begin{align}
	Q_{t+1}(s_t,a_t) &= Q_t(s_t,a_t) + \alpha_t(s_t,a_t)\left(r_t+\gamma \max_a Q_t(s_{t+1},a)-Q_t(s_t,a_t)\right),\\
	M_{t+1}(s_t,a_t) &= M_t(s_t,a_t) + \beta_t(s_t,a_t)\left(\left(\delta_t'/\gamma\right)^{2^ k} - M_t(s_t,a_t)\right).
\end{align}
In this equation, $\delta_t'= r_t+\gamma \max_a Q_{t+1}(s_{t+1},a)-Q_{t+1}(s_t,a_t)$, and ${(\alpha_t,\beta_t)}_{t\geq 0}$ are the learning rates that meet the Robbins-Monroe conditions \cite{borkar2009stochastic}.

\paragraph{Bootstrap sample.} Our method employs a bootstrap sampling strategy to estimate uncertainties in a non-parametric way. This approach can augment a forced exploration step, ensuring the convergence of the estimates asymptotically. 
We maintain a collection of $(Q,M)$-values and produce a new bootstrap sample $(\hat Q_t,\hat M_t)$ at every time step $t$.
In particular, we start with an ensemble of $Q$-functions ${Q_1,\dots,Q_B}$ (similarly for ${M_1,\dots,M_B}$) initialized uniformly at random in $[0,1/(1-\gamma)]$ (similarly $[0,1/(1-\gamma)^{2k}]$).   It's important to highlight that initializing the ensemble members across the full range of potential values is essential to account for uncertainties not arising from the collected data.

At each timestep $t$, a bootstrap sample $\hat Q_t(s,a)$ (and similarly $\hat M_t$) is generated by sampling a uniform random variable $\xi\sim {\cal U}([0,1])$, and, for each $(s,a)$, $\hat Q_t(s,a)={\rm Quantile}_\xi({Q_{t,1}(s,a), \dots, Q_{t,B}(s,a)})$; in other words, $\hat Q_t(s,a)$ is the $\xi$-quantile of ${Q_{t,1}(s,a), \dots, Q_{t,B}(s,a)}$ assuming a linear interpolation between the $Q$-values. This method is akin to sampling from an empirical cumulative distribution function (CDF), where the CDF in this case embodies the uncertainty over the $Q$-values. We also apply the same bootstrap sampling procedure to $\hat M_t(s,a)$. Empirically, we found the method to work for small values of the ensemble size $B$ for most problems, but we did not conduct extensive research on this topic. A promising venue of research is to study how the parameter $p$ and the ensemble size $B$ can be tuned for the problem at hand.

\paragraph{Computing the allocation.} This bootstrap sample $(\hat Q_t, \hat M_t)$ is subsequently used to calculate the allocation $\omega^{(t)}$, where $\Delta_{t}(s,a) = \max_{a'} \hat Q_{t}(s,a')-\hat Q_{t}(s,a)$, $\pi_t^\star(s)=\argmax_a \hat Q_t(s,a)$, and $\deltaminestimatet{t} = \min_{s,a\neq \pi_t^\star(s)} \Delta_t(s,a)$.
 Then, we set
\begin{align*}
	&H_{t}(s,a) \coloneqq \frac{2+8\varphi^2 \hat  M_{t}(s,a)^{2^{1- k}} }{(\Delta_{t}(s,a)+\lambda)^2},\\
	&H_{t} \coloneqq \frac{\max_{s'} 4 (1+\gamma)^2\max(1,4\gamma^2\varphi^2 \hat  M_{t}(s',\pi_{t}^\star(s'))^{2^{1- k}})}{ (\deltaminestimatet{t}+\lambda)^2(1-\gamma)^2}
\end{align*}
as well as
\[\omega_o^{(t)}(s_t,a) =
\begin{cases}
	H_{t}(s_t,a) & \hbox{ if }a\neq \pi_t^\star(s_t),\\
	\sqrt{H_{t} \sum_{s,a\neq \pi_t^\star(s)} H_{t}(s,a)/|S|} & \hbox{ otherwise }.
\end{cases}
\]
The final policy is then obtain by normalizing $\omega_o^{(t)}$: $\omega^{(t)}(s_t,a) =  \frac{\omega_o^{(t)}(s_t,a)}{\sum_{a'} \omega_o^{(t)}(s_t,a')} $.

\paragraph{Greedy policy.}
Lastly, an overall greedy policy $\bar \pi_t^\star$ can be estimated by using the ensemble of $Q$-functions. For example, by majority voting as
\[\bar \pi_t^\star(s) \gets  {\rm Mode}(\{ \argmax_a Q_{t+1,1}(s,a),\dots,\argmax_a Q_{t+1,B}(s,a)   \}).\]

\subsection{{\sc DBMF-BPI} - Deep Boostrapped Model Free Best Policy Identification}
To generalize bootstrapped MF-BPI to continuous Markov Decision Processes (MDPs), we propose DBMF-BPI. DBMF-BPI leverages the concept of prior networks from BSP (Bootstrapping with Additive Prior) \cite{osband2018randomized}, to account for uncertainty not arising from the observed data.

\paragraph{Ensemble.} As in the previous method, we maintain an ensemble of $Q$-values ${Q_{\theta_1},\dots,Q_{\theta_B}}$ (along with their target networks) and an ensemble of $M$-values ${M_{\tau_1},\dots,M_{\tau_B}}$. Specifically, $Q$-values are computed as follows for a generic $b$-th member of the ensemble
\[Q_{\theta_b}(s,a) = Q_{\theta_{b,0}}(s,a) + \beta_Q Q_{\theta_{b,p}}(s,a),\]
where $\beta_Q \geq 0$ is a hyper-parameter defining the scale of the prior, $\theta_{b,0}$ is a learnable parameter, and $Q_{\theta_{b,p}}$ is a fixed, randomly-initialize, $Q$-network that serves as a randomized prior value function. Similarly, we compute the $M$-values as
\[M_{\tau_b}(s,a) = M_{\tau_{b,0}}(s,a) + \beta_M M_{\tau_{b,p}}(s,a),\]
where $\beta_M\geq 0$ is a hyper-parameter, $\tau_{b,0}$ is a learnable parameter and $M_{\tau_{b,p}}$ is a fixed random prior network for the $M$- function.

Note that the function of the prior network is to guarantee that the $(Q,M)$-values are capable of covering the full spectrum of potential values. This is similar to the random initialization procedure in MF-BPI. An alternate strategy might involve initializing the network in an optimistic way (\emph{i.e.}, by sampling parameters from a Gaussian distribution with larger variance), but our observations indicate that this may lead to worse performance.

\paragraph{Bootstrap sample.} As before, at each timestep $t$ a bootstrap sample $\hat Q_t(s,a)$ (and similarly $\hat M_t$) is generated by sampling a uniform random variable $\xi\sim {\cal U}([0,1])$, and, for each $(s,a)$, set $\hat Q_t(s,a)={\rm Quantile}_\xi({Q_{t,\theta_1}(s,a), \dots, Q_{t,\theta_B}(s,a)})$. However, for numerical stability, we found it was most effective to sample $\xi$ at the end of an episode, or every $n \propto (1-\gamma)^{-1}$ steps.

\paragraph{Computing the allocation.} Using the bootstrap sample $(\hat Q_t, \hat M_t)$ we compute the allocation as follows. We set
\begin{align}
	H_{t}(s_t,a) &= \frac{2+8\varphi^2 \hat M_{t}(s_t,a)^{2^{1- k}} }{(\Delta_{t}(s_t,a)+\lambda)^2},\\
	H_t&=\frac{ 4 (1+\gamma)^2\max(1,4\gamma^2\varphi^2 \hat M_{t}(s_t,\pi_{t}^\star(s_t))^{2^{1- k}})}{ (\deltaminestimatet{t}+\lambda)^2(1-\gamma)^2},
\end{align}
where $\pi_t^\star(s_t) = \argmax_a \hat Q_t(s_t,a)$. Note that $H_t$ is an approximation of the true value (we are not taking the maximum over all possible states). Subsequently, we establish the allocation $\omega_o^{(t)}$: $\omega_o^{(t)}(s_t,a) = H_{t}(s_t,a)$ if $a\neq \pi_{t}^\star(s_t)$, and $\omega_o^{(t)}(s_t,a) = \sqrt{H_{t} \sum_{a\neq \pi_{t}^\star(s_t)} H_{t}(s_t,a)}$ otherwise. In the final step, we construct an $\epsilon_t$-soft exploration policy $\omega^{(t)}(s_t,\cdot)$ by blending $\omega_o^{(t)}(s_t,\cdot)/\sum_{a} \omega_o^{(t)}(s_t,a)$ with a uniform distribution, utilizing an exploration parameter $\epsilon_t$.

\paragraph{Training and minimum gap estimation.} The training procedure follows that of the classical DQN algorithm \cite{mnih2015human}. Each $Q$-network is trained by minimizing an MSE loss criterion. We use also the MSE loss to train  the $M$-networks  over a batch sampled from the replay buffer (note that the $M$-networks do not require a target network).

Next, $\deltaminestimatet{t}$ is estimated through stochastic approximation, using the smallest gap from the most recent batch of transitions retrieved from the replay buffer as a reference. In particular,  the target is given by the following expression
\[
\delta_t = \min_{b\in[B]}\min_{j\in {\cal B}}  \max_{a\neq \pi_{\theta_b}^\star(s_j)}  Q_{\theta_b}(s_j,\pi_{\theta_b}^\star(s_j)) - Q_{\theta_b}(s_j,a)
\]
with $\pi_{\theta_b}(s) = \argmax_a Q_{\theta_b}(s,a)$. The estimate is then updated as $\deltaminestimatet{t+1} \gets (1-\alpha_t)\deltaminestimatet{t} + \alpha_t \delta_t$ for some learning rate $\alpha_t=O(1/t)$.
\paragraph{Greedy policy}
Lastly, an overall greedy policy $\bar \pi_t^\star$ can be estimated by using the ensemble of $Q$-functions. For example, by majority voting as
\[\bar \pi_t^\star(s) \gets  {\rm Mode}(\{ \argmax_a Q_{t+1,\theta_1}(s,a),\dots,\argmax_a Q_{t+1,\theta_B}(s,a)   \}).\]

The full pseudo-code of the algorithm can be found in the next page.
\clearpage
\begin{algorithm}[t]
\caption{\textsc{DBMF-BPI} (Deep Bootstrapped Model Free BPI) - Full Algorithm}
\label{algo:dbmfbpi_full}
\small
\begin{algorithmic}[1]
    \REQUIRE Parameters $(\lambda,  k)$; ensemble size $B$;  exploration rate $\{\epsilon_t\}_t$;  estimate $\deltaminestimatet{0}$; mask probability $p$.
    \FUNCTION{{\tt MainLoop}}
        \STATE Initialize replay buffer ${\cal D}$, networks  $Q_{\theta_b}, M_{\tau_b}$  and targets $Q_{\theta'_b}$ for all $b\in[B]$.
        \FOR{$t=0,1,2,\dots,$}
	        \STATE {\bf Sampling step.}
	        \begin{ALC@g}
	        	\STATE Compute allocation $\omega^{(t)} \gets {\tt ComputeAllocation}(s_t,\{Q_{\theta_b} ,M_{\tau_b}\}_{b\in [B]},\deltaminestimatet{t},\gamma,\lambda, k, \epsilon_t)$.
	        	\STATE Sample $a_t\sim \omega^{(t)}(s_t,\cdot)$ and observe $(r_t,s_{t+1})\sim q(\cdot|s_t,a_t)\otimes P(\cdot|s_t,a_t)$.
	        	\STATE Add transition $z_t=(s_t,a_t,r_t,s_{t+1})$ to the replay buffer ${\cal D}$.
	        \end{ALC@g}
	        \STATE {\bf Training step.}
	        \begin{ALC@g}
	        	\STATE  Sample a batch ${\cal B}$ from ${\cal D}$, and with probability $p$ add the $i^{th}$ experience in ${\cal B}$ to a sub-batch ${\cal B}_b$, $\forall b\in [B]$.   Update the $(Q,M)$-values  of the $b^{th}$ member in the ensemble using ${\cal B}_b$: $\{Q_{\theta_b},Q_{\theta_b'},M_{\tau_b}\}_{b\in [B]} \gets {\tt Training}(\{{\cal B}_b,Q_{\theta_b},Q_{\theta_b'},M_{\tau_b}\}_{b\in [B]})$.
	        	\STATE Update estimate $\deltaminestimatet{t+1} \gets {\tt EstimateMinimumGap}(\deltaminestimatet{t}, {\cal B}, \{Q_{\theta_b}\}_{b\in[B]})$.
	        \end{ALC@g}
	        	\STATE Compute greedy policy as \[\bar \pi_t^\star(s) \gets  {\rm Median}(\{ \argmax_a Q_{t+1,\theta_1}(s,a),\dots,\argmax_a Q_{t+1,\theta_B}(s,a)   \}).\]
        \ENDFOR
    \ENDFUNCTION
\end{algorithmic}

\vspace{.1em}
\begin{algorithmic}[1]
    \FUNCTION{{\tt EstimateAllocation}$(s_t,\{Q_{\theta_b},M_{\tau_b}\}_{b\in [B]}, \deltaminestimatet{t} \gamma,\lambda, k, \epsilon_t)$}
    	\STATE Sample $\xi \sim {\cal U}([0,1])$ and set,  $\hat Q_t(s_t,a)={\rm Quantile}_\xi(\{Q_{t,\theta_1}(s_t,a), \dots, Q_{t,\theta_B}(s_t,a)\})$ (sim. $\hat M_t)$.
        \STATE Let $\pi_{t}^\star(s_t) = \argmax_a \hat Q_t(s_t,a)$, and set $\Delta_{t}(s_t,a) =\hat Q_t(s_t,\pi_{t}^\star(s_t,a)) - \hat Q_t(s_t,a)$.
        \STATE Compute MDP-related quantities
        \begin{align*}
        	&H_{t}(s_t,a) \coloneqq \frac{2+8\varphi^2 \hat M_t(s_t,a)^{2^{1- k}} }{(\Delta_{t}(s_t,a)+\lambda)^2},\\
        	&H_{t} \coloneqq \frac{ 4 (1+\gamma)^2\max(1,4\gamma^2\varphi^2 \hat M_t(s_t,\pi_{t}^\star(s_t))^{2^{1- k}})}{ (\deltaminestimatet{t}+\lambda)^2(1-\gamma)^2}
        \end{align*}
        \STATE Set
        \[\omega_o^{(t)}(s_t,a) =
        \begin{cases}
        	H_{t}(s_t,a) & \hbox{ if }a\neq \pi_{t}^\star(s_t),\\
        	\sqrt{H_{t} \sum_{a\neq \pi_{t}^\star(s_t)} H_{t}(s_t,a)} & \hbox{ otherwise }.
        \end{cases}
        \]
        \STATE {\bf Return } $\omega^{(t)}(s_t,a) =  \dfrac{\epsilon_t}{|A|} +(1-\epsilon_t)\frac{\omega_o^{(t)}(s_t,a)}{\sum_{a'} \omega_o^{(t)}(s_t,a')} $, the policy at time $t$ in state $s_t$.
    \ENDFUNCTION
\end{algorithmic}

\vspace{.1em}
\begin{algorithmic}[1]
    \FUNCTION{{\tt Training}$(\{{\cal B}_b,Q_{\theta_b},Q_{\theta_b'},M_{\tau_b}\}_{b\in [B]})$}
        \FOR{ each model in the ensemble $b=1,\dots,B$}
             \STATE Compute targets $y_j = r_j+\gamma \max_a Q_{\theta_b'}(s_{j+1},a)$ and perform a gradient descent step on $Q_{\theta_b}$ using $\nabla_{\theta_b} (y_j-Q_{\theta_b}(s_j,a_j))^2$ for all $j\in {\cal B}_b$.
             \STATE Compute targets $\bar y_j=(r_j + \max_a Q_{\theta_b}(s_{j+1},a) - Q_{\theta_b}(s_j,a_j))/\gamma$ and perform a gradient descent step on $M_{\tau_b}$using $\nabla_{\tau_b} (\bar y_j^{2^{ k}} - M_{\tau_b}(s_j,a_j))^2$.
        \ENDFOR
        \STATE Every $K$ steps update target models: $\theta_{b'} \gets \theta_b$ for all $b\in [B]$.
        \STATE {\bf Return} updated models $\{Q_{\theta_b},Q_{\theta_b'},M_{\tau_b}\}_{b\in [B]}$.
    \ENDFUNCTION
\end{algorithmic}

\vspace{.1em}
\begin{algorithmic}[1]
	\FUNCTION{{\tt EstimateMinimumGap}$(\deltaminestimatet{t}, {\cal B}, \{Q_{\theta_b}\}_{b\in [B]})$}
	\STATE Set learning rate $\alpha_t = O(1/t)$.
	\STATE Update estimate of $\deltaminestimatet{t}$: let $\pi_{\theta_b}^\star(s_j) = \argmax_a Q_{\theta_b}(s_j,a)$ and compute target
	\[\delta_t = \min_{b\in[B]}\min_{j\in {\cal B}}  \max_{a\neq \pi_{\theta_b}^\star(s_j)}  Q_{\theta_b}(s_j,\pi_{\theta_b}^\star(s_j)) - Q_{\theta_b}(s_j,a)\] and update estimate $\deltaminestimatet{t+1} \gets (1-\alpha_t)\deltaminestimatet{t} + \alpha_t \delta_t$.
	\STATE {\bf Return} updated estimate $\deltaminestimatet{t+1}$.
	\ENDFUNCTION
\end{algorithmic}

\end{algorithm}

\clearpage
 \newpage
 \newpage
 
\section{Proofs}\label{appC}
In this appendix, we provide the proofs of our main results. We start with some preliminary results. We then introduce new notation to accommodate extensions beyond the assumptions made in the main body of the paper, and prove our main theorem. Specifically, we broaden our sample-complexity bounds to encompass communicating MDPs without a unique optimal policy. 

\subsection{Preliminaries}

Let $V:{\cal S}\to \mathbb{R}$ be a bounded function. We show that $\var_{sa}[V] \leq \md_{sa}[V]^2$. This inequality follows directly from the Bhatia-Davis inequality \cite{bhatia2000better}. Applied to the value function of our MDP, this result implies that in the bound derived in Theorem \ref{thm:upper_bound}, the term corresponding to the span of $V^\star$ might be sometimes dominant, and we might indeed wish to remove it from the upper bound. 

\begin{lemma}\label{lemma:bathia-davis-ineq}
    Consider an MDP $\phi$ with $|S|$ states and a bounded vector  $V \in \mathbb{R}^{|S|}$. For any $(s,a),$ we have $\var_{sa}[V] \leq \md_{sa}[V]^2$. If $\md_{sa}[V] \leq 1$ then $\var_{sa}[V] \leq \md_{sa}[V]$.
\end{lemma}
\begin{proof}[Proof of \cref{lemma:bathia-davis-ineq}]
The result is obtained leveraging the Bhatia-Davis inequality \cite{bhatia2000better}. Fix $(s,a)$, and consider a bounded vector $V$. Let $\mu(s,a)=\mathbb{E}_{s'\sim P(\cdot|s,a)}[V(s')]$, $M = \max_s V(s)$ and $m=\min_s V(s)$. Then, define
    \[
    G(s,a)\coloneqq\mathbb{E}_{s'\sim P(\cdot|s,a)}[(M-V(s'))(V(s')-m)].
    \]
    We have $G(s,a)= -mM -\mathbb{E}_{s'\sim P(\cdot|s,a)}[V(s')^2]+(M+m)\mu(s,a)$. Since $0\leq G(s,a)$, 
    \begin{align*}
    -\mu(s,a)^2 &\leq  -mM -\mathbb{E}_{s'\sim P(\cdot|s,a)}[V(s')^2]+(M+m)\mu(s,a) - \mu(s,a)^2,\\
    \var_{P(s,a)}[V]&\leq -mM +(M+m)\mu(s,a) -\mu(s,a)^2,\\
    \var_{P(s,a)}[V]&\leq  (M-\mu(s,a))(\mu(s,a)-m).
    \end{align*}
    Since $\md_{sa}[V]= \|V - \mu(s,a)\|_\infty= \max(M-\mu(s,a), \mu(s,a)-m)$, we conclude that
    \[
    \var_{P(s,a)}[V] \leq \max(M-\mu(s,a), \mu(s,a)-m)^2 =\md_{P(s,a)}[V]^2.
    \]
    This also implies that, if $\md_{sa}[V] \leq 1$, then $ \var_{sa}[V]\leq \md_{sa}[V]$.
\end{proof}
More generally, we also note that
\begin{align*}
M_{sa}^{k}[V^\pi]^{2^{-k}}&\leq\mathbb{E}_{s'\sim P(\cdot|s,a)}\left[\left(\max_{s'} V^\pi(s') -\mathbb{E}_{\bar s\sim P(\cdot|s,a)}[V^\pi(\bar s)]\right)^{2^k}\right]^{2^{-k}},\\
&=\mathbb{E}_{s'\sim P(\cdot|s,a)}\left[\left \| V^\pi-\mathbb{E}_{\bar s\sim P(\cdot|s,a)}[V^\pi(\bar s)]\right\|_\infty^{2^k}\right]^{2^{-k}},\\
&= \md_{sa}[V^\pi].
\end{align*}

\subsection{Alternative upper bounds}

In this subsection, we establish the alternative upper bounds $\bar U_\varepsilon$ of the sample complexity lower bound proposed in Theorem \ref{thm:upper_bound_T_new}. Our results extend those of \cite{al2021adaptive} to MDPs where the optimal policy might not be unique. 

\subsubsection{Sample complexity lower bound}

Assume for now that the way the learner interacts with the MDP corresponds to the generative model: in each round, she can pick any (state, action) pair and observe the corresponding next state and reward. Under this model, the following theorem provides a sample complexity lower bound satisfied by any $(\varepsilon,\delta)$-PAC algorithm.   


\begin{theorem}[$(\delta,\varepsilon)$-PAC lower bound]\label{thm:lower_bound_delta_epsilon_pac}
Consider $\varepsilon\geq 0$, and a communicating MDP $\phi$, not necessarily with a unique optimal policy. Then, the sample complexity $\tau$ of any $(\delta,\varepsilon)$-PAC algorithm under the generative model satisfies the following lower bound:
 \begin{equation}
\mathbb{E}_\phi[\tau]\geq T_\varepsilon \kl(\delta,1-\delta),
 \end{equation}
 where   $T_\varepsilon=\sup_{\omega \in \Delta(S\times A)} T_\varepsilon(\omega)$ is the optimal characteristic time, and
 \begin{equation}\label{eq:teps}
    T_\varepsilon(\omega)^{-1}=\inf_{\psi \in \alt_\varepsilon(\phi)} \mathbb{E}_{(s,a)\sim\omega}[\KL_{\phi|\psi}(s,a)].
 \end{equation}
\end{theorem}

The proof follows the same lines as in \cite{al2021adaptive}. A similar lower bound can be derived in the forward model where the learner has to follow the system trajectory \cite{marjani2021navigating}: it is obtained by replacing the supremum over $\omega \in \Delta(S\times A)$ by a supremum over $\omega \in \Omega(\phi)$, to account for the navigation constraints.


\subsubsection{Upper bound on $T_\varepsilon(\omega)$}

As explained in \cite{al2021adaptive}, even for $\varepsilon=0$, (\ref{eq:teps}) is in general a non-convex problem. Therefore it may not always be possible to even approximately solve it. An alternative way, introduced in\cite{al2021adaptive}, consists in convexifying the problem. The solution of the new problem then gives an upper bound of $T_0$.

To this aim, we will start from the following result, providing a decomposition of the confusing set. 
\begin{proposition}\label{corollary:upper_bound}
    We have $T_\varepsilon(\omega)\leq  T(\omega)$ for all $\omega$, where $ T(\omega)$ is defined as
    \begin{equation}
         T(\omega)^{-1} =\min_{\pi \in \Pi_0^\star(\phi)}\min_{s,a\neq\pi(s)}\min_{\psi\in   \alt_{\pi,sa}(\phi)} \mathbb{E}_{(s,a)\sim\omega}[\KL_{\phi|\psi}(s,a)].
    \end{equation}
    where $  \alt_{\pi,sa}(\phi)=\{\psi: \phi \ll \psi, Q_\psi^{\pi}(s,a) > V_\psi^{\pi}(s)\}$.
\end{proposition}
\begin{proof}
    A similar result was derived in \cite{al2021adaptive}. Its proof follows directly from  \cref{lemma:baralt_contains_alt} and \cref{lemma:relaxation_confusing_set}. From \cref{lemma:baralt_contains_alt} we have that the set $\alt(\phi)=\{\psi: \psi\ll \phi, \Pi_0^\star(\phi) \cap \Pi_0^\star(\psi)=\emptyset\}$ contains $\alt_\varepsilon(\phi)$. From \cref{lemma:relaxation_confusing_set} we  have that
    $ \alt(\phi) \subseteq \cup_{\pi\in \Pi_0^\star(\phi)}\cup_{s}\cup_{a\neq\pi(s)}  \alt_{\pi,sa}(\phi)$, where  \[  \alt_{\pi,sa}(\phi)=\{\psi: Q_\psi^{\pi}(s,a) > V_\psi^{\pi}(s) \}.\]
    Therefore \[T_\varepsilon(\omega)^{-1} \geq \min_{\pi \in \Pi_0^\star(\phi)}\min_{s,a\neq\pi(s)}\min_{\psi\in   \alt_{\pi,sa}(\phi)} \mathbb{E}_{(s,a)\sim\omega}[\KL_{\phi|\psi}(s,a)] =  T(\omega)^{-1}.\]
\end{proof}



From the previous proposition, we are able to derive the upper bound of $T_\varepsilon$.
\begin{theorem}\label{theorem:new_alternative_bound}
    Consider a communicating MDP $\phi$, not necessarily with a unique optimal policy. Then, for every $(s,a)$ there exists $\bar k(s,a)\in \mathbb{N}$ s.t. for all $\omega\in \Delta(S\times A)$ we have
     \begin{equation}
         T_\varepsilon(\omega)\leq  U(\omega) ,
     \end{equation}
 with
     \begin{equation}
     \resizebox{\hsize}{!}{%
     	$
     U(\omega)=\max_{\pi\in\Pi_0^\star(\phi)} \max_{s,a\neq \pi(s)}\Bigg( \frac{2+8\varphi^2M_{sa}^{(\bar k(s,a))}[V_\phi^\star]^{2^{1-\bar k(s,a)}} }{\Delta_{\pi}(s,a)^2\omega(s,a)}
     +\max_{s'} \frac{4C^\pi(s')(1+\gamma)^2}{\omega(s',\pi(s')) \Delta_{\pi}(s,a)^2(1-\gamma)^2}\Bigg),
     $%
 }
    \end{equation}
    where $\Delta_{\pi}(s,a) \coloneqq V_\phi^{\pi}(s)-Q_\phi^{\pi}(s,a)$ and   $C^\pi(s')=\max\left(1,4\gamma^2\varphi^2 M_{s'\pi(s')}^{(\bar k(s',\pi(s')))}[V_\phi^\star]^{2^{1-\bar k(s',\pi(s'))}}\right)$.
\end{theorem}
\begin{proof}
    The proof is similar as that of Theorem 1 in \cite{al2021adaptive}.
    We start from the result of \cref{corollary:upper_bound}: 
    \[
    T_\varepsilon(\omega)^{-1} \geq \min_{\pi\in\Pi_0^\star(\phi)}\min_{s,a\neq\pi(s)}\inf_{\psi \in \alt_{\pi,sa}(\phi)} \mathbb{E}_{(s,a)\sim\omega}[\KL_{\phi|\psi}(s,a)].
    \]
    For a fixed $(\pi,s,a)$, the constraint $\inf_{\psi \in \alt_{\pi,sa}(\phi)}$ does not involve the pairs $(\tilde s, \tilde a)\in S\times A \setminus \{(s,a),(s',\pi(s'))_{s'\in S} \}$. As argued in \cite{al2021adaptive}, by convexity, the solution must satisfy $\KL_{\phi|\psi}(\tilde s,\tilde a)=0$ for those pairs. Hence
    \begin{align*}
    &\inf_{\psi \in \alt_{\pi,sa}(\phi)} \mathbb{E}_{(s,a)\sim\omega}[\KL_{\phi|\psi}(s,a)] = \\
    &\qquad\qquad \inf_{\psi \in \alt_{\pi,sa}(\phi)}\omega(s,a)\KL_{\phi|\psi}(s,a) + \sum_{s'}\omega(s',\pi(s')) \KL_{\phi|\psi}(s',\pi(s')).   
    \end{align*}
    
    Let $\Delta_{\pi}(s,a) \coloneqq V_\phi^{\pi}(s)-Q_\phi^{\pi}(s,a)$.
    Then, using the fact that  $Q_\psi^{\pi}(s,a) > V_\psi^{\pi}(s)$, we obtain
    \begin{align*}
    \Delta_{\pi}(s,a) &< V_\phi^{\pi}(s) - Q_\phi^{\pi}(s,a) + Q_\psi^{\pi}(s,a) -V_\psi^{\pi}(s).
    \end{align*}
    This is similar to condition (5) in \cite{al2021adaptive}. Next,  let $\Delta r(s,a)=r_\psi(s,a)-r_\phi(s,a)$, $\Delta P(s,a) = P_\psi(s,a)-P_\phi(s,a)$, where the distribution $P(s,a)$  of the next state given $(s,a)$ is represented as a column vector of dimension $|S|$. Further define the vector difference between the value in $\psi$ and $\phi$ of $\pi$: $\Delta V^\pi = \begin{bmatrix}V_\psi^\pi(s_1) -V_\phi^\pi(s_1) & \dots & V_\psi^\pi(s_{|S|}) -V_\phi^\pi(s_{|S|}) \end{bmatrix}^\top$. Then, letting $\mathbf{1}(s)=e_s$ be the unit vector with $1$ in position $s$,  we find
    \begin{align*}
    \Delta_{\pi}(s,a)&< Q_\psi^{\pi}(s,a)-Q_\phi^{\pi}(s,a) -\mathbf{1}(s)^\top \Delta V^\pi,\\
    &< \Delta r(s,a) +\gamma(P_\psi(s,a)^\top V_\psi^\pi-P_\phi(s,a)^\top V_\phi^\pi) -\mathbf{1}(s)^\top \Delta V^\pi,\\
    &< \Delta r(s,a) +\gamma \Delta P(s,a)^\top V_\phi^\pi +(\gamma P_\psi(s,a)-\mathbf{1}(s))^\top \Delta V^\pi.
    \end{align*}

    Now, observe that:
    \begin{align*}
        V_\psi^\pi(s) - V_\phi^\pi(s) &= \Delta r(s,\pi(s)) + \gamma (P_\psi(s,\pi(s))^\top V_\psi^\pi -P_\phi(s,\pi(s))^\top V_\phi^\pi),\\
        &=  \Delta r(s,\pi(s))  +\gamma(P_\psi(s,\pi(s))^\top \Delta V^\pi +\Delta P(s,\pi(s))^\top V_\phi^\pi),\\
        &\leq \left|  \Delta r(s,\pi(s))  +\gamma(P_\psi(s,\pi(s))^\top \Delta V^\pi +\Delta P(s,\pi(s))^\top V_\phi^\pi)\right|,\\
        &\leq \max_{s'}|\Delta r(s',\pi(s')) + \gamma \Delta P(s',\pi(s'))^\top V_\phi^\pi | +\gamma \max_{\tilde s}| V_\psi^\pi(\tilde s) - V_\phi^\pi(\tilde s)|.
    \end{align*}
    We deduce that:
    \[
    \|\Delta V^\pi\|_\infty \leq \frac{1}{1-\gamma} \left[\max_{s'}|\Delta r(s',\pi(s'))| + \gamma  |\Delta P(s',\pi(s'))^\top V_\phi^\pi|  \right].
    \]
    Using the fact that $\|\gamma P_\psi(s,a)-\mathbf{1}(s)\|_1 = |\gamma P(s|s,a)- 1| + \gamma(1-P(s|s,a)) \leq 1+\gamma$, we can bound $|(\gamma P_\psi(s,a)-\mathbf{1}(s))^\top \Delta V^\pi|$ as follows:
    \begin{align*}
        |(\gamma P_\psi(s,a)-\mathbf{1}(s))^\top \Delta V^\pi| &\leq \|\gamma P_\psi(s,a)-\mathbf{1}(s)\|_1 \|\Delta V^\pi\|_\infty\\
        &\leq \frac{1+\gamma}{1-\gamma} \left[\max_{s'}|\Delta r(s',\pi(s'))| + \gamma|\Delta P(s',\pi(s'))^\top V_\phi^\pi|  \right].
    \end{align*}
    Therefore, 
    \[
    \resizebox{\hsize}{!}{%
    	$
    \Delta_{\pi}(s,a)< |\Delta r(s,a)| +\gamma |\Delta P(s,a)^\top V_\phi^\pi| +\frac{1+\gamma}{1-\gamma} \left[\max_{s'}|\Delta r(s',\pi(s'))| + \gamma |\Delta P(s',\pi(s'))^\top V_\phi^\pi|  \right].$%
	}
    \]
    Write each of the terms as a fraction of $\Delta_\pi(s,a)$ using $\{\alpha_i\}_{i=1}^3$, which are non-negative terms satisfying $\sum_{i=1}^3 \alpha_i>1$:
    \[
    \begin{cases}
        \alpha_1 \Delta_\pi(s,a) =  |\Delta r(s,a)|,\\
        \alpha_2 \Delta_\pi(s,a) = \gamma|\Delta P(s,a)^\top V_\phi^\pi|,\\
        \alpha_3 \Delta_\pi(s,a) =   \dfrac{1+\gamma}{1-\gamma}\max_{s'}\left[|\Delta r(s',\pi(s'))|+\gamma |\Delta P(s',\pi(s'))^\top V_\phi^\pi|\right].
    \end{cases}
    \]
    For the first term, using the Pinsker inequality, we immediately get: $(\alpha_1 \Delta_{\pi}(s,a))^2 \leq 2\KL_{q_\phi,q_\psi}(s,a)$.

    For the second term, using  \cref{lemma:alternative_lemma4}, we obtain:
    \begin{align*}
        (\alpha_2 \Delta_{\pi}(s,a))^2&\leq 8\gamma^2\varphi^2 M_{sa}^{(\bar k(s,a))}[V_\phi^\star]^{2^{1-\bar k(s,a)}}\KL_{P_\phi,P_\psi}(s,a).
    \end{align*}
    Finally, to bound the last term, using $(a+b)^2\leq 2(a^2+b^2)$, \cref{lemma:alternative_lemma4} and the Pinsker inequality, we have
    \begin{align*}
        \Big(|\Delta r(s',\pi(s'))|+&\gamma |\Delta P(s',\pi(s'))^\top V_\phi^\pi|\Big)^2 \leq
        2\left(|\Delta r(s',\pi(s'))|^2+\gamma^2 |\Delta P(s',\pi(s'))^\top V_\phi^\pi|^2\right),\\
        &\leq 2\Big(2\KL_{q_\phi,q_\psi}(s',\pi(s')) +8\gamma^2\varphi^2 M_{s'\pi(s')}^{(\bar k(s',\pi(s')))}[V_\phi^\star]^{2^{1-\bar k(s',\pi(s'))}}\KL_{P_\phi,P_\psi}(s',\pi(s'))\Big),\\
        &\leq 4C^\pi(s')(\KL_{q_\phi,q_\psi}(s',\pi(s')) +\KL_{P_\phi,P_\psi}(s',\pi(s'))),
    \end{align*}
    with $C^\pi(s')=\max\left(1,4\gamma^2\varphi^2 M_{s'\pi(s')}^{(\bar k(s',\pi(s')))}[V_\phi^\star]^{2^{1-\bar k(s',\pi(s'))}}\right)$.
    Therefore
    \begin{align*}
    \alpha_3^2\frac{(1-\gamma)^2}{(1+\gamma)^2}&\Delta_\pi(s,a)^2 \leq 4\max_{s'} C^\pi(s')(\KL_{q_\phi,q_\psi}(s',\pi(s')) +\KL_{P_\phi,P_\psi}(s',\pi(s'))),\\
    &=  4\max_{s'}\frac{\omega(s',\pi(s'))}{\omega(s',\pi(s'))} C^\pi(s')(\KL_{q_\phi,q_\psi}(s',\pi(s')) +\KL_{P_\phi,P_\psi}(s',\pi(s'))),\\
    &\leq  4\max_{\tilde s}\frac{C^\pi(\tilde s)}{\omega(\tilde s,\pi(\tilde s))} \max_{s'}\omega(s',\pi(s'))(\KL_{q_\phi,q_\psi}(s',\pi(s')) +\KL_{P_\phi,P_\psi}(s',\pi(s'))).
    \end{align*}
    In conclusion, we have the following set of inequalities:
    \begin{align*}
    \frac{\omega(s,a)(\alpha_1 \Delta_{\pi}(s,a))^2}{2} &\leq \omega(s,a)\KL_{q_\phi,q_\psi}(s,a),\\
    \frac{\omega(s,a)(\alpha_2 \Delta_{\pi}(s,a))^2}{8\gamma^2\varphi^2 M_{P_\phi(s,a)}^{(\bar k(s,a))}[V_\phi^\star]^{2^{1-\bar k(s,a)}}}&\leq  \omega(s,a)\KL_{P_\phi,P_\psi}(s,a),\\
    \min_{s'}\frac{\omega(s',\pi(s'))(\alpha_3 (1-\gamma)\Delta_{\pi}(s,a))^2}{4 C^\pi(s')(1+\gamma)^2} &\leq \max_{s'}\omega(s',\pi(s'))(\KL_{q_\phi,q_\psi}(s',\pi(s'))\\
    &\qquad+\KL_{P_\phi,P_\psi}(s',\pi(s'))).
    \end{align*}
    As in \cite{al2021adaptive} we observe that we can replace $\alpha_i$ by $\alpha_i/\sum_{j}\alpha_j$ (since $\sum_i\alpha_i >1$).
    Consequently, denoting by  $\Delta_n$ the $n$-dimensional simplex, we have
    \begin{align*}
    T_\varepsilon(\omega)^{-1}&\geq  \min_{\pi\in\Pi_0^\star(\phi)}\min_{s,a\neq \pi(s)}\inf_{\psi \in \bar \alt_{\pi,sa,\varepsilon}(\phi)}\omega(s,a)\KL_{\phi|\psi}(s,a) + \sum_{s'}\omega(s',\pi(s'))\KL_{\phi|\psi}(s',\pi(s')).
    \\
    &\geq \min_{\pi\in\Pi_0^\star(\phi)}\min_{s,a\neq \pi(s)} \inf_{\alpha\in \Delta_3}\sum_{i=1}^3 B_i(s,a) \alpha_i^2.
    \end{align*}
    where
    \begin{align*}
        B_1(s,a)&=\omega(s,a)\Delta_{\pi}(s,a)^2/2,\\
        B_2(s,a)&= \omega(s,a) \frac{\Delta_{\pi}(s,a)^2}{8\gamma^2\varphi^2 M_{sa}^{(\bar k(s,a))}[V_\phi^\star]^{2^{1-\bar k(s,a)}}},\\
        B_3(s,a) &= \min_{s'} \omega(s',\pi(s'))\frac{\left(\Delta_{\pi}(s,a)(1-\gamma)\right)^2}{4C^\pi(s')(1+\gamma)^2}.
    \end{align*}
    Therefore $T_\varepsilon(\omega)^{-1}\geq \min_{\pi\in\Pi_0^\star(\phi)}\min_{s,a\neq \pi(s)} \left(\sum_{i=1}^3 B_i(s,a)^{-1} \right)^{-1}$, from which  we conclude that:
    \begin{equation}
    \resizebox{\hsize}{!}{%
    	$
    T_\varepsilon(\omega)
     \leq\max_{\pi\in\Pi_0^\star(\phi)} \max_{s,a\neq \pi(s)}\Bigg( \frac{2+8\gamma^2\varphi^2 M_{sa}^{(\bar k(s,a))}[V_\phi^\star]^{2^{1-\bar k(s,a)}} }{\Delta_{\pi}(s,a)^2\omega(s,a)}
     +\max_{s'} \frac{4C^\pi(s')(1+\gamma)^2}{\omega(s',\pi(s'))\Delta_{\pi}(s,a)^2(1-\gamma)^2}\Bigg).$%
 }
    \end{equation}
\end{proof}

\subsubsection{Closed form solution under the generative model}
Under the generative model, we are able to find a closed-form solution of the sample complexity upper bound by slightly relaxing our upper bound of $T_\varepsilon(\omega)$. The procedure is similar to that used in \cite{al2021adaptive}.
\begin{theorem}
    Let $\varepsilon\geq 0$, and a communicating MDP $\phi$, with a unique optimal policy $\pi^star$. Then, for all $\omega\in \Delta(S\times A)$, we have:
     \begin{equation}
         T_\varepsilon(\omega)\leq  U(\omega) \leq \tilde U(\omega),
     \end{equation}
 where $ U(\omega)$ is defined in the previous theorem, and
     \begin{equation}
    \tilde U(\omega)=\max_{s,a\neq \pi^\star(s)} \frac{2+8\varphi^2M_{sa}^{(\bar k(s,a))}[V_\phi^\star]^{2^{1-\bar k(s,a)}} }{\Delta(s,a)^2\omega(s,a)}
     +\frac{\max_{s'}4C^{\pi^\star}(s')(1+\gamma)^2}{\min_{\tilde s}\omega(\tilde s,\pi^\star(\tilde s)) \deltamin^2(1-\gamma)^2}.
    \end{equation}
    where $\Delta(s,a) \coloneqq V_\phi^{\pi^\star}(s)-Q_\phi^{\pi^\star}(s,a)$.
\end{theorem}
\begin{proof}
The proof follows from the previous theorem. Since there is a unique optimal policy we have $\Delta_\pi(s,a) \geq \deltamin$, and thus
\begin{align*}
\tilde U(\omega) &\leq  \max_{s,a\neq \pi^\star(s)} \frac{2+8\varphi^2M_{sa}^{(\bar k(s,a))}[V_\phi^\star]^{2^{1-\bar k(s,a)}} }{\Delta(s,a)^2\omega(s,a)}
     +\max_{s'} \frac{4C^{\pi^\star}(s')(1+\gamma)^2}{\omega(s',\pi^\star(s')) \deltamin^2(1-\gamma)^2},\\
     &\leq \max_{s,a\neq \pi^\star(s)} \frac{2+8\varphi^2M_{sa}^{(\bar k(s,a))}[V_\phi^\star]^{2^{1-\bar k(s,a)}} }{\Delta(s,a)^2\omega(s,a)}
     +\frac{\max_{s'}4C^{\pi^\star}(s')(1+\gamma)^2}{\min_{\tilde s}\omega(\tilde s,\pi^\star(\tilde s)) \deltamin^2(1-\gamma)^2}.
\end{align*}
\end{proof}

For this particular bound, as in \cite{al2021adaptive}, we are able to find a closed form expression of the optimal generative allocation $\omega^\star\in \arg\min_{\omega\in\Delta(S\times A)} \tilde U(\omega)$ leading to an upper bound of the sample complexity lower bound. The following corollary is obtained by simply solving the optimization problem $\inf_{\omega\in\Delta(S\times A)} \tilde U(\omega)$. 

\medskip

\begin{corollary}\label{corollary:generative_solution}
    Consider a communicating MDP with unique optimal policy. Consider the bound defined in the previous theorem by  $\tilde U(\omega)$. Then, the generative solution $\omega^\star = \arginf_{\omega\in\Delta(S\times A)} \tilde U(\omega)$ is given by
    \begin{equation}
        \omega(s,a) = \begin{cases}
            H(s,a)/\Gamma &s,a\neq\pi^\star(s),\\
             \sqrt{H \sum_{s,a\neq \pi^\star(s)} H(s,a)/|S|}/\Gamma &\hbox{otherwise}.
        \end{cases}
    \end{equation}
    where
    \begin{align}
    & H(s,a)=\dfrac{2+8\varphi^2M_{sa}^{(\bar k(s,a))}[V_\phi^\star]^{2^{1-\bar k(s,a)}} }{\Delta(s,a)^2\omega(s,a)}, \quad H =\max_{s'}\dfrac{4C^{\pi^\star}(s')(1+\gamma)^2}{ \deltamin^2(1-\gamma)^2}, \\
    & \Gamma =\sum_{s,a\neq \pi^\star(s)}H(s,a) + \sqrt{|S|H \sum_{s,a\neq \pi^\star(s)} H(s,a)}.
    \end{align}
     Furthermore, the value of the problem is:
     \begin{equation}
     	\inf_{\omega\in\Delta(S\times A)} \tilde U(\omega)= \left(\sqrt{ \sum_{s,a\neq \pi^\star(s)}H(s,a)}+ \sqrt{|S|H} \right)^2 \leq  2\left(\sum_{s,a\neq \pi^\star(s)}H(s,a) + |S|H\right).
     \end{equation}
\end{corollary}

\subsubsection{Technical lemmas}

We finally state and prove the lemmas used in the derivation of our upper bounds of the sample complexity lower bound. These lemmas can be seen as an alternative to Lemma 4 used by the authors of \cite{al2021adaptive} to derive their upper bounds.

In what follows, we consider a finite set $\Omega=\{\omega_1,\dots,\omega_N\}$. For each $\omega\in \Omega$, let $f(\omega)$ be a real number, and we define the vector ${\bf f}(\Omega)=\begin{bmatrix}f(\omega_1) & \dots & f(\omega_N)\end{bmatrix}^\top$.

We start by a result, that can be deducted from the proof of Lemma 4 in \cite{al2021adaptive}.

\medskip

\begin{lemma}\label{lemma:bound_inner_product_1}
Let $P,Q$ be pmfs over some finite space $\Omega=\{\omega_1,\dots,\omega_N\}$. Let $f:\Omega \to \mathbb{R}$ and ${\bf f}(\Omega)\coloneqq\begin{bmatrix}f(\omega_1)&\dots& f(\omega_N)\end{bmatrix}^\top$. \\
Finally, we introduce the elementwise power\footnote{also known as as Hadamard power.} ${\bf f}^{\circ k}(\Omega)=\begin{bmatrix}f(\omega_1)^k& \dots& f(\omega_N)^k\end{bmatrix}^\top$. Then
\begin{align}\label{eq:bound_inner_product_1}
    |(P-Q)^\top {\bf f}(\Omega)|^2&\leq 4d_H(P,Q)^2\Big(2\mathbb{E}_{\omega \sim Q}[f(\omega)^2]+ (P-Q)^\top ({\bf f}^{\circ 2}(\Omega))   \Big),
\end{align}
where $d_H$ is the Hellinger distance.
\end{lemma}
\begin{proof}
    The proof can be easily deduced from Lemma 4 in \cite{al2021adaptive}. We present the proof for completeness.
    Let $\sqrt{P}$ be  the square root of the elements in $P$ (sim. $\sqrt{Q}$). We have:
    \begin{align*}
    (P-Q)^\top {\bf f}(\Omega)&= \sum_\omega (P(\omega)-Q(\omega)) f(\omega),\\
    &=\sum_\omega (\sqrt{P(\omega)} - \sqrt{Q(\omega)})(\sqrt{P(\omega)}+ \sqrt{Q(\omega)})f(\omega),\\
    &= (\sqrt{P}-\sqrt{Q})^\top [(\sqrt{P}+\sqrt{Q})\circ {\bf f}(\Omega)],
    \end{align*}
    where $\circ$ is the Hadamard product.
    We apply the Cauchy-Schwartz  inequality to the last term to get:
    \[
    |(P-Q)^\top {\bf f}(\Omega)|^2 \leq  \|\sqrt{P}-\sqrt{Q}\|_2^2 \|\|(\sqrt{P}+\sqrt{Q})\circ {\bf f}(\Omega)\|_2^2.
    \]
    Note that $\|\sqrt{P}-\sqrt{Q}\|_2= \sqrt{2}d_H(P,Q)$. Regarding  $\|(\sqrt{P}+\sqrt{Q})\circ {\bf f}(\Omega)\|_2$, using the inequality $(a+b)^2\leq 2(a^2+b^2)$, we have:
    \begin{align*}
    \|(\sqrt{P}+\sqrt{Q})\circ {\bf f}(\Omega)\|_2^2 &\leq2 \sum_{\omega}(P(\omega)+Q(\omega))f(\omega)^2,\\
    &= 2 \sum_{\omega}(2Q(\omega)+P(\omega)-Q(\omega))f(\omega)^2,\\
    &= 2\left(2\mathbb{E}_{\omega\sim Q}[f(\omega)^2] + (P-Q)^\top {\bf f}^{\circ 2}(\Omega)\right),
    \end{align*}
    which concludes the proof.
\end{proof}

Applying the above lemma recursively, we obtain the following result. 

\medskip

\begin{lemma}\label{lemma:bound_inner_product_2}
    Consider $f:\Omega\to\mathbb{R}$ as before.    Assume that $\max_{\omega\in\Omega} |f(\omega)|\leq F<\infty$. 
    Then,
    \begin{equation}
        |(P-Q)^\top {\bf f}(\Omega)|\leq  \sqrt{8}  \varphi   d_H(P,Q) \sup_{k\geq 1}\mathbb{E}_{\omega \sim Q}[f(\omega)^{2^k}]^{2^{-k}},
    \end{equation}
    where $\varphi$ is the golden ratio.
\end{lemma}
\begin{proof}
    The idea is to observe that we can use \cref{lemma:bound_inner_product_1} to bound $(P-Q)^\top {\bf f}^{\circ 2}(\Omega)$ in \cref{eq:bound_inner_product_1}. Then
    \begin{align*}
        |(P-Q)^\top  {\bf f}^{\circ k}(\Omega)|^2&\leq 4d_H(P,Q)^2\Big(2\mathbb{E}_{\omega \sim Q}[f(\omega)^{2k}] \nonumber + (P-Q)^\top {\bf f}^{\circ 2k}(\Omega)  \Big).
    \end{align*}
    For brevity, let  $M_k = \mathbb{E}_{\omega\sim Q}[f(\omega)^k]$, then
    \begin{align*}
        &|(P-Q)^\top {\bf f}(\Omega)| \leq 2 d_H(P,Q)\sqrt{2M_2+(P-Q)^\top {\bf f}^{\circ 2}(\Omega)},\\
        &\leq  2 d_H(P,Q)\sqrt{2M_2+  2 d_H(P,Q)\sqrt{2M_4 +(P-Q)^\top {\bf f}^{\circ 4}(\Omega)}},\\
        &\leq \alpha\sqrt{2M_2+  \alpha\sqrt{2M_4+\alpha\sqrt{2M_8+\cdots}}},
    \end{align*}
    where $\alpha= 2d_H(P,Q)$. A further rewriting yields
    \begin{align*}
        & \alpha\sqrt{2M_2+  \alpha\sqrt{2M_{4}+\alpha\sqrt{2M_8+\cdots}}},\\
        &= \sqrt{2\alpha^2M_2+  \alpha^3\sqrt{2M_{4}+\alpha\sqrt{2M_8+\cdots}}},\\
        &= \sqrt{2\alpha^2M_2+  \sqrt{2\alpha^6M_{4}+\alpha^7\sqrt{2M_8+\cdots}}},\\
        &= \sqrt{2\alpha^2 M_2+  \sqrt{2\alpha^6M_{4}+\sqrt{2\alpha^{14}M_8+\cdots}}},
    \end{align*}
    and note that the $k$-th term is given by $a_k=2\alpha^{2(2^k-1)} M_{2^k}$.
    Consider now the sequence $b_k= (a_k)^{2^{-k}}$, and note that
    \[
    \sup_{k\geq 1} b_k \leq \sup_{k\geq 1} \underbrace{\left(2 \alpha^{2(2^k-1)} \right)^{2^{-k}}}_{(\bullet)} \cdot \sup_{k\geq 1}M_{2^k}^{2^{-k}}.
    \]
    Observe that $(\bullet) = 2^{2^{-k}} \alpha^{2-2^{-k+1}}$ is a positive decreasing sequence, therefore we have that $\sup_{k\geq 1} b_k \leq \alpha \sqrt{2} \cdot \sup_{k\geq 1}M_{2^k}^{2^{-k}}$.

    Now, we notice  that $M_{2^k}^{2^{-k}}$ is bounded for all $k\geq 1$ from  the boundedness of $f$ over $\Omega$
    \[M_{2^k}^{2^{-k}}= \mathbb{E}_{\omega}[f(\omega)^{2^k}]^{2^{-k}}\leq F<\infty.\]
    
    Hence, by letting $M=\alpha \sqrt{2} \cdot \sup_{k\geq 1}M_{2^k}^{2^{-k}}$, and using Herschfeld's convergence theorem \cite{herschfeld1935infinite}, we find the desired result:
    \begin{align*}
         & \sqrt{2\alpha^2M_{2}+  \sqrt{2\alpha^6M_{4}+\sqrt{2\alpha^{14}M_8+\cdots}}}\\
         &\leq \sqrt{M^2+  \sqrt{M^{2^{2}}+\sqrt{M^{2^{3}}+\cdots}}},\\
         &= M\sqrt{1 +  \sqrt{1+\sqrt{1+\cdots}}} = M\varphi.
    \end{align*}
\end{proof}

\medskip

We are now ready to state the result that serves as an alternative to Lemma 4 in \cite{al2021adaptive}. Let $(\Delta P(s,a))_{s'}= P_\psi(s'|s,a)-P_\phi(s'|s,a)$.

\medskip

\begin{lemma}\label{lemma:alternative_lemma4}
Consider a fixed state-action pair $(s,a)$ and define $\bar V_\phi^\pi(s,a)\coloneqq \mathbb{E}_{s'\sim P_\phi(\cdot|s,a)}[V_\phi^\pi(s')]$. Let $f_\phi^\pi(s,a,s')= V_\phi^\pi(s') - \bar V_\phi^\pi(s,a)$ and $M_{k}(s,a)=\mathbb{E}_{s'\sim P_\phi(\cdot|s,a)}[f_\phi^\pi(s,a,s')^{2^k}]$. Then,  there exists $\bar k \in \mathbb{N}$ such that 
\begin{equation}
        |\Delta P(s,a)^\top {\bf f}_\phi^\pi(s,a)|^2 \leq 8 \varphi^2 \KL_{P_\phi,P_\psi}(s,a)  M_{\bar k}(s,a)^{2^{1-\bar k}},
    \end{equation}
    where  ${\bf f}_\phi^\pi(s,a)=\begin{bmatrix}
    	f_\phi^\pi(s,a,s_1) &f_\phi^\pi(s,a,s_2) &\cdots &f_\phi^\pi(s,a,s_{|S|})
    \end{bmatrix}^\top $ and $ \KL_{P_\phi,P_\psi}(s,a)=\KL(P_\phi(s,a),P_\psi(s,a))$. 
\end{lemma}
\begin{proof}
    Consider a fixed $(s,a)$.
    For any $s'\in S$ we have that $|f_\phi^\pi(s,a,s')| \leq \md_{sa}[V_\phi^\pi]$, therefore $\|{\bf f}_\phi^\pi(s,a)\|_\infty <\infty$. Using  \cref{lemma:bound_inner_product_2} with ${\bf f}_\phi^\pi(s,a)$ we find the result by taking the square on both sides, and using that $d_H^2(P,Q)\leq \KL(P,Q)$.
    

\end{proof}

\subsubsection{Decomposition of the set of confusing MDPs}

Decomposing the set $\alt_\varepsilon(\phi)$ directly presents several challenges. It does even seem possible to obtain a decomposition easy to work with.
Instead, we relax the problem and work on $\alt(\phi)=\{\psi: \psi\ll \phi, \Pi_0^\star(\phi) \cap \Pi_0^\star(\psi)=\emptyset\}$, a set containing $\alt_\varepsilon(\phi)$.  


\begin{lemma}\label{lemma:baralt_contains_alt}
    Let $\varepsilon\geq 0$. Then, in general $\alt_\varepsilon(\phi) \subseteq  \alt(\phi)$.
\end{lemma}
\begin{proof}
The  statement can be derived by contradiction: assume that $\psi\in \alt_\varepsilon(\phi)$ does not belong to $\alt(\phi)$. However, that implies that there is $\pi\in \Pi_0^\star(\phi)$ s.t. $\pi \in \Pi_0^\star(\psi)$, which is not true since by assumption $\Pi_\varepsilon^\star(\phi) \cap \Pi_\varepsilon^\star(\psi)=\emptyset$.
\end{proof}

\begin{lemma}\label{lemma:relaxation_confusing_set}Let 
    $\alt(\phi)=\{\psi: \psi\ll \phi, \Pi_0^\star(\phi) \cap \Pi_0^\star(\psi)=\emptyset\}$. Then $ \alt(\phi) \subseteq\cup_{\pi\in \Pi_0^\star(\phi)}\cup_{s}\cup_{a\neq\pi(s)}  \alt_{\pi,sa}(\phi)$, where  \[  \alt_{\pi,sa}(\phi)=\{\psi:\psi\ll \phi, Q_\psi^{\pi}(s,a) > V_\psi^{\pi}(s) \}.\]
\end{lemma}
\begin{proof}
The proof follows the same steps as that of  the decomposition lemma in \cite{al2021adaptive}, and we give it for completeness.

By contradiction, consider $\psi \in  \alt(\phi)$ s.t. for all $\pi \in \Pi_0^\star(\phi)$ and $s,a\neq\pi(s)$ we have $Q_\psi^{\pi}(s,a) \leq V_\psi^{\pi}(s)$. Since $Q_\psi^\pi(s,\pi(s)) = V_\psi^\pi(s)$, the following inequality holds for all $\pi\in \Pi_0^\star(\phi)$ and for all $(s,a)$
\[
       Q_\psi^{\pi}(s, a) \leq  V_\psi^\pi( s).
 \]
Define the Bellman operator for a generic policy $\pi'$ under $\psi$ as $(T_\psi^{\pi'}V)(s) = r_\psi(r, \pi'(s)) + \mathbb{E}_{s'\sim P(s,\pi'(s))}[V(s')]$. Then, from the above inequality that holds for all $(s,a)$ we get the following result
\[
T_\psi^{\pi_\psi^\star}V_\psi^\pi \leq V_\psi^\pi.
\]
By monotonicity of the Bellman operator, we get $T_\psi^{\pi_\psi^\star}T_\psi^{\pi_\psi^\star}V \leq T_\psi^{\pi_\psi^\star}V_\psi^\pi \leq V_\psi^\pi$. Iterating, we find
\[
V_\psi^{\pi_\psi^\star} = \lim_{n\to\infty} \left(T_\psi^{\pi_\psi^\star}\right)^n V \leq V_\psi^\pi,
\]
which is a contradiction since $\pi$ is not optimal under $\psi$.


\end{proof}

\else\fi
\end{document}